\documentclass{article}

     \PassOptionsToPackage{numbers, compress}{natbib}
     \usepackage[final]{neurips_2024}

\usepackage[utf8]{inputenc} %
\usepackage[T1]{fontenc}    %
\usepackage[colorlinks,
            linkcolor=red,
            anchorcolor=blue,
            citecolor=blue, 
            pagebackref=true
           ]{hyperref}       %
\usepackage{url}            %
\usepackage{booktabs}       %
\usepackage{amsfonts}       %
\usepackage{nicefrac}       %
\usepackage{microtype}      %
\usepackage{mmll}
\usepackage{mathrsfs}
\usepackage{tikz}
\usetikzlibrary{arrows,automata,positioning}
\usepackage[colorinlistoftodos, textsize=tiny, textwidth=2.2cm, backgroundcolor=blue!10, linecolor=magenta, bordercolor=magenta]{todonotes}
\usepackage{setspace}
\usepackage{enumitem}
\usepackage{titletoc}

\crefname{assumption}{assumption}{assumptions}

\DeclareMathOperator*{\argsup}{arg\,sup} 
\DeclareMathOperator*{\arginf}{arg\,inf}

\DeclareMathOperator*{\dist}{dist} 
\newcommand{\algname}{\text{DRPVI}}
\newcommand{\algnameVA}{\text{VA-DRPVI}}
\makeatletter
\renewcommand{\todo}[2][]{%
    \@todo[caption={#2}, #1]{\begin{spacing}{0.5}#2\end{spacing}}%
} 
\makeatother 

\renewcommand*{\backref}[1]{}
\renewcommand*{\backrefalt}[4]{%
\ifcase #1 %
    No citations.%
\or
    (p. #2.)%
\else
    (pp. #2.)%
\fi}%
\allowdisplaybreaks

\title{Minimax Optimal and Computationally Efficient Algorithms for Distributionally Robust Offline Reinforcement Learning}

\author{
   Zhishuai Liu%
   \\
   Duke University\\
   \texttt{zhishuai.liu@duke.edu} \\
   \And
   Pan Xu \\
   Duke University \\
   \texttt{pan.xu@duke.edu} \\
}

\begin{document}
\maketitle

\begin{abstract}
\label{sec:abstract}
Distributionally robust offline reinforcement learning (RL), which seeks robust policy training against environment perturbation by modeling dynamics uncertainty, calls for function approximations when facing large state-action spaces. However, the consideration of dynamics uncertainty introduces essential nonlinearity and computational burden, posing unique challenges for analyzing and practically employing function approximation. Focusing on a basic setting where the nominal model and perturbed models are linearly parameterized, we propose minimax optimal and computationally efficient algorithms realizing function approximation and initiate the study on instance-dependent suboptimality analysis in the context of robust offline RL. Our results uncover that function approximation in robust offline RL is essentially distinct from and probably harder than that in standard offline RL. Our algorithms and theoretical results crucially depend on a novel function approximation mechanism incorporating variance information, a new procedure of suboptimality and estimation uncertainty decomposition, a quantification of the robust value function shrinkage, and a meticulously designed family of hard instances, which might be of independent interest.
\end{abstract}

\section{Introduction}
\label{sec:introduction}
Offline reinforcement learning (RL) \citep{lange2012batch, levine2020offline}, which aims to learn an optimal policy achieving maximum expected cumulative reward from a pre-collected dataset, plays an important role in critical domains where online exploration is infeasible due to high cost or ethical issues, such as precision medicine \citep{wang2018supervised, gottesman2019guidelines, Liu2023dtr, liu2023estimation} and autonomous driving \citep{pan2018agile, sun2020scalability}.
The foundational assumption of offline RL \citep{levine2020offline, jin2021pessimism, xie2021bellman}  is that the offline dataset is collected from the same environment where learned policies are intended to be deployed.
However, this assumption can be violated in practice due to temporal changes in dynamics. In such cases, standard offline RL could face catastrophic failures \citep{farebrother2018generalization, packer2018assessing, zhao2020sim}. 
To address this issue, the robust offline RL \citep{morimoto2005robust, nilim2005robust} focuses on robust policy training against the environment perturbation, which serves as a promising solution. 
Existing empirical successes of robust offline RL rely heavily on expressive function approximations \citep{pinto2017robust, pattanaik2018robust, mandlekar2017adversarially, tessler2019action, zhang2020robust, kuang2022learning}, as the omnipresence of applications featuring large state and action spaces necessitates powerful function representations to enhance generalization capability of decision-making in RL. 

To theoretically understand robust offline RL with function approximation, the distributionally robust Markov decision process (DRMDP) \citep{satia1973markovian, nilim2005robust, iyengar2005robust} provides an established framework. In stark contrast to the standard MDP, DRMDP specifically tackles the {\it model uncertainty} by forming an uncertainty set around the nominal model, 
and takes a max-min formulation aiming to maximize the value function corresponding to a policy, uniformly across all perturbed models in the uncertainty set \citep{xu2006robustness, wiesemann2013robust, yang2022toward, panaganti2022robust, shi2024distributionally,yang2023distributionally,shen2024wasserstein}.
The core of DRMDPs lies in achieving an amenable combination of uncertainty set design and  corresponding techniques to solve the inner optimization over the uncertainty set. 
However, this consideration of  model uncertainty introduces fundamental challenges to function approximation in terms of computational and statistical efficiency, particularly given the need to maximally exploit essential information in the offline dataset. 
For instance, in cases where the state and action spaces are large, the commonly used $(s,a)$-rectangular uncertainty set can make the inner optimization computationally intractable for function approximation \citep{zhou2024natural}. 
Additionally, the distribution shifts, arising from the mismatch between the behavior policy and the target policy, as well as the mismatch between the nominal model and perturbed models, complicate the statistical analysis \citep{shi2024distributionally,blanchet2023double}. Several recent works attempt to conquer these challenges. \citet{panaganti2022robust} studied the $(s,a)$-rectangularity, and their algorithm may suffer from the above mentioned computational issue. Additionally, the $(s,a)$-rectangular uncertainty set may contain transitions that would never happen in reality, and thus leads to conservative policies; %
\citet{blanchet2023double} proposed a novel double pessimism principle, while their algorithm requires strong oracles, which is not practically implementable. 
Meanwhile, a line of works study function approximation in the online setting \citep{tamar2014scaling, roy2017reinforcement, wang2021online, badrinath2021robust, liu2024distributionally} or with a simulator \citep{zhou2024natural},
which are not applicable to offline RL. Thus, the following question arises:
\begin{center}
    {\it Is it possible to design a computationally efficient and minimax optimal algorithm for \\robust offline RL with function approximation?}
\end{center}
To answer the above question, we focus on a basic setting of {\it $d$-rectangular linear DRMDP}, 
where the nominal model is a standard linear MDP, and all perturbed models are parameterized in a linearly structured uncertainty set.
We provide the first {\it instance-dependent} suboptimality analysis in the DRMDP literature with function approximation, which offers insights into the problem’s intrinsic characteristics and challenges. Concretely, our contributions are summarized as follows.
\begin{itemize}[nosep, leftmargin=*]
    \item We propose a computationally efficient  algorithm, Distributionally Robust Pessimistic Value Iteration (\algname), based on 
    the pessimism principle \citep{jin2021pessimism, xie2021bellman,shi2024distributionally} with a new {\it function approximation} mechanism explicitly devised for $d$-rectangular linear DRMDPs. We show that \algname\ achieves the following instance-dependent upper bound on the suboptimality gap:
    \[
        \textstyle \beta_1\cdot\sup_{P\in\cU^{\rho}(P^0)}
        \sum_{h=1}^H\EE^{\pi^\star, P}\big[\sum_{i=1}^d\Vert \phi_i(s_h,a_h)\mathbf{1}_i\Vert_{\bLambda_h^{-1}} |s_1=s\big]\footnote{Here, $d$ is the feature dimension, $H$ is the horizon length, $\beta_1=\tilde{O}(\sqrt{d}H)$ is a tunning parameter in \algname, $\cU^{\rho}(P^0)$ is the uncertainty set with radius $\rho$, $\pi^{\star}$ is the optimal robust policy, $\bphi(\cdot,\cdot):\cS\times\cA\rightarrow\RR^d$ is the instance-dependent feature vector, and $\bLambda_h$ is the covariance matrix defined in \eqref{eq:covariance matrix}.},
    \]
    This bound resembles those established in offline RL within standard linear MDPs \citep{jin2021pessimism, zanette2021provable,xiong2023nearly}. 
    However, there are two significant differences in our results. First, our bound depends on the {\it supremum} over the uncertainty set of transition kernels instead of one single transition kernel. Second, our result relies on a {\it diagonal-based normalization}, instead of the Mahalanobis norm of the feature vector, $\Vert \bphi(s_h,a_h)\Vert_{\bLambda_h^{-1}}$. See \Cref{table:comparison} for a clearer comparison. These two distinctions are unique to DRMDPs with function approximation, which we discuss in more details in \Cref{sec: Robust Pessimistic Value Iteration}.
    Moreover, our analysis provides a novel pipeline for studying instance-dependent upper bounds of computationally efficient algorithms under $d$-rectangular linear DRMDPs.  

     \item We improve \algname\ by incorporating {\it variance information} into the new function approximation mechanism, resulting in the \algnameVA\ algorithm, which achieves a smaller upper bound:
    \[
        \textstyle \beta_2 \cdot \sup_{P\in\cU^{\rho}(P^0)}\sum_{h=1}^H\EE^{\pi^{\star}, P}\big[\sum_{i=1}^d\Vert\phi_i(s_h,a_h)\mathbf{1}_i \Vert_{\bSigma_h^{\star-1}}|s_1=s\big]\footnote{$\beta_2=\tilde{O}(\sqrt{d})$ is a  hyperparameter in \algnameVA;  $\bSigma_h^{\star}$ is the variance-weighted covariance matrix, see \eqref{eq:variance weighted covariance matrix_population_version}.}.
    \]
    This improves the result of \algname\ due to the fact that $\bSigma_h^{\star-1} \preceq H^2\bLambda_h^{-1}$ by definition \citep{yin2022near, xiong2023nearly}. Furthermore, when the uncertainty level $\rho=O(1)$, we show that the robust value function attains a {\it Range Shrinkage} property, leading to an improvement in the upper bound by an order of $H$. This explicit improvement is new in variance-aware algorithms, and is unique to DRMDPs.

    \item %
    We further establish an {\it information-theoretic lower bound}. We prove that the upper bound of \algnameVA\ matches the information-theoretic lower bound up to $\beta_2$, which implies that \algnameVA\ is near-optimal in the sense of information theory. %
    Importantly, both \algname\ and \algnameVA\ are computationally efficient and do not suffer from the high computational burden, as discussed above in settings with the $(s,a)$-rectangular uncertainty set, due to a decoupling property of the $d$-rectangular uncertainty set (see \Cref{remark:DRPVI-design} for more details).
    Thus, we confirm that, for robust offline RL with function approximation, both the {\it computational efficiency} %
    and {\it minimax optimality} are achievable under the setting of $d$-rectangular linear DRMDPs.

\end{itemize}
\begin{table}[t]
\small
\centering
\caption{Summary of instance-dependent results in offline RL with linear function approximation. $\bLambda_h$ and $\bSigma_h^{\star}$ are the empirical covariance matrix defined in \eqref{eq:covariance matrix} and \eqref{eq:variance weighted covariance matrix_population_version} respectively. Note that $\pi^{\star}$ means the optimal policy in standard MDPs and the optimal robust policy in DRMDPs. The definition of $\bSigma_h^{\star}$ also depends on the corresponding definition of $\pi^{\star}$. %
\label{table:comparison}}
\resizebox{\textwidth}{!}{
\begin{tabular}{ccc}
\toprule
Algorithm & Setting & Instance-dependent upper bound on the suboptimality gap \\
\midrule
\begin{tabular}{@{}c@{}}
PEVI \citep{jin2021pessimism} 
\end{tabular} & MDP &$dH\cdot \sum_{h=1}^H\EE^{\pi^\star,P}\big[\Vert \bphi(s_h,a_h)\Vert_{\bLambda_h^{-1}} |s_1=s\big]$ \\
\begin{tabular}{@{}c@{}}
LinPEVI-ADV \\
\citep{xiong2023nearly}
\end{tabular}
 & MDP & $\sqrt{d}H\cdot \sum_{h=1}^H\EE^{\pi^\star,P}\big[\Vert \bphi(s_h,a_h)\Vert_{\bLambda_h^{-1}} |s_1=s\big]$ \\
\begin{tabular}{@{}c@{}}
LinPEVI-ADV+ \\
\citep{xiong2023nearly}
\end{tabular}
 &  MDP & $\sqrt{d}\cdot \sum_{h=1}^H\EE^{\pi^\star,P}\big[\Vert \bphi(s_h,a_h)\Vert_{\bSigma_h^{\star-1}} |s_1=s\big]$ \\
\midrule
\begin{tabular}{@{}c@{}}
\algname\
(ours)
\end{tabular}
 & DRMDP & $\sqrt{d}H\cdot\sup_{P\in\cU^{\rho}(P^0)} \sum_{h=1}^H\EE^{\pi^\star,P}\big[\sum_{i=1}^d\Vert \phi_i(s_h,a_h)\mathbf{1}_i\Vert_{\bLambda_h^{-1}} |s_1=s\big]$\\
\begin{tabular}{@{}c@{}}
\algnameVA\ (ours)
\end{tabular}
 & DRMDP & $\sqrt{d}\cdot\sup_{P\in\cU^{\rho}(P^0)} \sum_{h=1}^H\EE^{\pi^\star,P}\big[\sum_{i=1}^d\Vert \phi_i(s_h,a_h)\mathbf{1}_i\Vert_{\bSigma_h^{\star-1}} |s_1=s\big]$\\
\bottomrule
\end{tabular}}
\end{table}

Our algorithm design and theoretical analysis draw inspiration from two crucial ideas proposed in standard linear MDPs: the reference-advantage decomposition \citep{xiong2023nearly} and the variance-weighted ridge regression \citep{zhou2021nearly}. However, the unique challenges in DRMDPs necessitate novel treatments that go far beyond a combination of existing techniques. Specifically, existing analysis of standard linear MDPs highly relies on the linear dependency of the Bellman equation on the (nominal) transition kernel. This linear dependency is disrupted by the consideration of model uncertainty, which induces essential {\it nonlinearity} that significantly complicates the statistical analysis of estimation error. To obtain our instance-dependent upper bounds, we establish a new theoretical analysis {\it pipeline}. This pipeline starts with a nontrivial  decomposition of the suboptimality, and employs a new uncertainty decomposition that transforms the estimation uncertainty over all perturbed models to estimation uncertainty under the nominal model. 

The information-theoretic lower bound in our paper is the first of its kind in the linear DRMDP setting, which could be of independent interest to the community. Previous lower bounds, which are based on the commonly used {\it Assouad's method} and established under the standard linear MDP, do not consider model uncertainty. In particular, one prerequisite for applying Assouad's method is switching the initial minimax objective to a minimax risk in terms of Hamming distance. The intertwining of this prerequisite with the nonlinearity induced by the model uncertainty makes the analysis significantly more challenging. To this end, we construct a novel family of {\it hard instances}, carefully designed to (1) mitigate the nonlinearity caused by the model uncertainty, (2) fulfil the prerequisite for Assouad's method, and (3) be concise enough to admit matrix analysis.

\paragraph{Notations:} We denote $\Delta(\cS)$ as the set of probability measures over some set $\cS$. For any number $H\in\ZZ_{+}$, we denote $[H]$ as the set of $\{1,2,\cdots, H\}$. For any function $V:\cS\rightarrow \RR$, we denote $[\PP_hV](s,a) = \EE_{s'\sim P_h(\cdot|s,a)}[V(s')]$ as the expectation of $V$ with respect to the transition kernel $P_h$, $[\text{Var}_hV](s,a) = [\PP_hV^2](s,a) - ([\PP_hV](s,a))^2$ as the variance of $V$,  $[\VV_h V](s,a)=\max\{1, [\Var_h V](s,a)\}$ as the truncated variance of $V$, and $[V(s)]_{\alpha}=\min\{V(s),\alpha\}$, given a scalar $\alpha>0$, as the truncated value of $V$. For a vector $\bx$, we denote $x_j$ as its $j$-th entry. And we denote $[x_i]_{i\in [d]}$ as a vector with the $i$-th entry being $x_i$. For a matrix $A$, denote $\lambda_i(A)$ as the $i$-th eigenvalue of $A$. For two matrices $A$ and $B$, we denote $A\preceq B$ as the fact that $B-A$ is a positive semidefinite matrix.
For any function $f:\cS\rightarrow\RR$, we denote $\|f\|_{\infty}=\sup_{s\in\cS}f(s)$. 
Given %
$P,Q\in\Delta(\cS)$, %
the total variation divergence of $P$ and $Q$ is defined as $D(P||Q)=1/2\int_{\cS}|P(s)-Q(s)|ds$.

\section{Most Related Work}

\paragraph{DRMDPs.}
The DRMDP framework has been extensively studied under different settings. The works of \cite{xu2006robustness, wiesemann2013robust, yu2015distributionally, mannor2016robust, goyal2023robust} assumed precise knowledge of the environment and formulated the DRMDP as classic planning problems. 
The works of \cite{zhou2021finite, yang2022toward, panaganti2022sample, xu2023improved, shi2024curious, yang2023robust} assumed access to a generative model and studied the sample complexities of DRMDPs.
The works of 
\cite{panaganti2022robust, shi2024distributionally, blanchet2023double} studied the offline setting assuming access to only an offline dataset, and established sample complexities under data coverage or concentrability assumptions. The works of 
\cite{wang2021online,badrinath2021robust,  dong2022online, liangsingle, liu2024distributionally}  studied the online setting where the agent can actively interact with the nominal environment to learn the robust policy.

\paragraph{DRMDPs with linear function approximation.}
\citet{tamar2014scaling, badrinath2021robust} proposed to use linear function approximation to solve DRMDPs with large state and action spaces and established asymptotic convergence guarantees. 
\citet{zhou2024natural} studied the natural Actor-Critic with function approximation, assuming access to a simulator. Their function approximation mechanisms depend on two novel uncertainty sets, one based on double sampling and the other on an integral probability metric. 
\citet{ma2022distributionally} first combined the linear MDP with the $d$-rectangular uncertainty set \citep{goyal2023robust}, and proposed the setting dubbed as the $d$-rectangular linear DRMDP, which naturally admits linear representations of the robust Q-functions\footnote{\citet{ma2022distributionally} study the offline $d$-rectangular linear DRMDPs with Kullback-Leibler (KL) uncertainty sets. We remark that 1) the proofs of their main lemmas (Lemma D.1 and Lemma D.2) related to suboptimality decomposition and the proof of theorems have technique flaws; 2) The formulation of their assumption 4.4 on the dual variable of the dual formulation of the KL-divergence is ambiguous and may be too strong to be realistic. Thus, the fundamental challenges of $d$-rectangular linear DRMDPs remain unresolved.}. 
\citet{panaganti2024bridging} leverages the $d$-rectangular linear DRMDP framework to address the distribution shift problem in offline linear MDPs.
\citet{blanchet2023double} studied the offline $d$-rectangular linear DRMDP setting, for which the provable efficiency is established under a double pessimism principle.
\citet{liu2024distributionally} then studied the online $d$-rectangular linear DRMDP setting and pointed out that the intrinsic nonlinearity of DRMDPs might pose additional challenges for linear function approximation. %
After the release of our work, a concurrent study \citep{wang2024sample} emerged, which independently investigated offline DRMDPs with linear function approximation. Their algorithms attained the same instance-dependent suboptimalities as our proposed algorithms \algname\ and \algnameVA. Their algorithm DROP also achieved the same order of worst-case suboptimality, $\tilde{O}(dH^2/\sqrt{K})$, as our \algname. However, we further demonstrated that our algorithm \algnameVA\ can strictly improve this result to $\tilde{O}(dH \min \{1/\rho, H\}/\sqrt{K})$. Moreover, we introduced a novel hard instance and established the first information-theoretic lower bound for offline DRMDPs with linear function approximation. We also note that there is a line of works \citep{blanchet2023double, panaganti2022robust} studied general function approximation under DRMDPs with the commonly studied $(s,a)$-rectangularity uncertainty sets, where no further structure is applied except the rectangularity.

\section{Problem Formulation}\label{sec:DRMDP and Linear MDP}
In this section, we provide the preliminary of $d$-rectangular linear DRMDPs, and describe the dataset as well as the learning goal in offline reinforcement learning. 

\paragraph{Standard MDPs.} We start with the standard MDP, which constitutes the basic of DRMDPs. A finite horizon Markov decision process is denoted by $\text{MDP}(\mathcal{S}, \mathcal{A}, H,  P, r)$, where $\mathcal{S}$ and $\mathcal{A}$ are the state and action spaces, $H\in \ZZ_+$ is the horizon length, $P=\{P_h\}_{h=1}^{H}$ denotes the set of probability transition kernels, 
$r=\{r_h\}_{h=1}^H$ denotes the reward functions. More specifically, for any $(h,s,a)\in[H]\times\cS\times\cA$, the transition kernel $P_h(\cdot|s,a)$ is a probability function over the state space $\cS$, and the reward function $r_h: \cS\times\cA \rightarrow [0,1]$ is assumed to be deterministic for simplicity. A sequence of deterministic policies is denoted as $\pi=\{\pi_h\}_{h=1}^H$, where $\pi_h:\cS \rightarrow \cA$ is the policy for step $h\in[H]$. Given any policy $\pi$ and transition $P$, for all $(s,a,h)\in\cS\times\cA\times[H]$, the corresponding value function $V_h^{\pi, P}(s):=\mathbb{E}^{\pi,P}\big[\sum_{t=h}^H r_t(s_t,a_t)\big|s_h=s\big]$ and Q-function $ Q_h^{\pi, P}(s,a):=\mathbb{E}^{\pi,P}\big[\sum_{t=h}^H r_t(s_t,a_t)\big|s_h=s,a_h=a \big]$ characterize the expected cumulative rewards starting from step $h$, and both of them are bounded in $[0,H]$. 

\paragraph{Distributionally robust MDPs.}
A finite horizon distributionally robust Markov decision process is denoted by $\text{DRMDP}(\mathcal{S}, \mathcal{A}, H, \mathcal{U}^{\rho}(P^0), r)$, where $P^0=\{P^0_h\}_{h=1}^H$ is the set of nominal transition kernels, and $\cU^{\rho}(P^0) =  \bigotimes_{h\in[H]}\cU^{\rho}_h(P_h^0)$ is the uncertainty set of transitions, where each $\cU^{\rho}_h(P_h^0)$ is usually defined as a ball centered at $P^0$ with radius/uncertainty level $\rho\geq 0$ based on some probability divergence measures \citep{iyengar2005robust, yang2022toward, xu2023improved}. To account for the model uncertainty, the 
robust value function $V_h^{\pi, \rho}(s) :=\inf_{P \in \mathcal{U}^{\rho}(P^0)}V_{h}^{\pi, P}(s)$, $\forall (h,s) \in [H] \times \cS$ 
is defined as the value function under the worst possible transition kernel within the uncertainty set $\cU^\rho(P^0)$. Similarly, the robust Q-function is defined as $ Q_h^{\pi, \rho}(s,a) =\inf_{P \in \mathcal{U}^{\rho}(P^0)}Q_{h}^{\pi, P}(s,a)$,
for any $(h,s,a) \in [H]\times \mathcal{S} \times \mathcal{A}$. Further, we define the optimal robust value function and the optimal robust Q-function as
\begin{align*}
  \textstyle  V_h^{\star, \rho}(s) = \sup_{\pi \in \Pi} V_h^{\pi, \rho}(s),\quad Q_h^{\star, \rho}(s,a) = \sup_{\pi \in \Pi} Q_h^{\pi, \rho}(s,a), \quad  \forall(h,s,a) \in [H]\times \mathcal{S} \times \mathcal{A}.
\end{align*}
where $\Pi$ is the set of all policies. The optimal robust policy $\pi^{\star}=\{\pi^{\star}_h\}_{h=1}^H$ is defined as the policy that achieves the optimal robust value function: $\textstyle \pi_h^{\star}(s) = \argsup_{\pi \in \Pi}V_h^{\pi, \rho}(s)$, $ \forall(h,s) \in [H]\times \mathcal{S}$.

\paragraph{$d$-rectangular linear DRMDPs.} 
A $d$-rectangular linear DRMDP is a DRMDP where the nominal environment is a special case of linear MDP with a simplex feature space \citep[Example 2.2]{jin2020provably} and the uncertainty set $\cU_h^\rho(P_h^0)$ is defined based on the linear structure of the nominal transition kernel $P_h^0$. In particular, we make the following assumption about the nominal environment.
\begin{assumption}
\label{assumption:linear MDP}
Let $\bphi: \cS \times \cA \rightarrow \RR^d$ be a state-action feature mapping such that $ \sum_{i=1}^d\phi_i(s,a)=1$, $\phi_i(s,a) \geq 0$,
for any $(i, s,a)\in [d] \times \cS \times \cA$. For any $(h,s,a) \in [H]\times \mathcal{S} \times \mathcal{A}$, the reward function and the nominal transition kernels have a linear representation:
$r_h(s,a)=\langle \bphi(s,a), \btheta_h\rangle, \text{ and }  P_h^0(\cdot|s,a)=\langle \bphi(s,a), \bmu_h^0(\cdot)\rangle$, where $\Vert \btheta_h \Vert_2 \leq \sqrt{d}$, and $\bmu_h^0=(\mu_{h,1}^0,\ldots,\mu_{h,d}^0)^\top$ are unknown probability measures over $\cS$. 
\end{assumption}

With notations in \Cref{assumption:linear MDP}, we define the factor uncertainty sets as $\cU_{h, i}^{\rho}(\mu^0_{h,i}) = \big\{\mu: \mu\in \Delta(\cS), D(\mu||\mu_{h,i}^0)\leq \rho\big\},\forall (h,i)\in[H]\times[d]$, where $D(\cdot||\cdot)$ is specified as the total variation (TV) divergence in this work. The  uncertainty set is defined as $\cU^{\rho}_h(P_h^0) = \bigotimes_{(s,a)\in \cS\times\cA}\cU_h^{\rho}(s,a; \bmu^0_h)$, where $\cU_h^{\rho}(s,a; \bmu^0_h)=\{\sum_{i=1}^d \phi_i(s,a)\mu_{h,i}(\cdot): \mu_{h,i}(\cdot) \in \cU_{h, i}^{\rho}(\mu^0_{h,i}), \forall i \in [d]\}$. A notable feature of this design is that the factor uncertainty sets $\{\cU_{h, i}^{\rho}(\mu^0_{h,i})\}_{h,i=1}^{H,d}$ are decoupled from the state-action pair $(s,a)$ and also independent with each other. As demonstrated later, this decoupling property results in a computationally efficient regime for function approximation.

\paragraph{Robust Bellman equation.}
Under the setting of $d$-rectangular linear DRMDPs, it is proved that the robust value function and the robust Q-function satisfy the robust Bellman equations \citep{liu2024distributionally}:
\begin{subequations}
\label{eq:robust bellman equation}
\begin{align}
 Q_h^{\pi, \rho}(s,a)&=\textstyle r_h(s,a)+\inf_{P_h(\cdot|s,a)\in\cU_h^{\rho}(s,a;\bmu_h^0)}[\PP_h V_{h+1}^{\pi,\rho}](s,a),\\
 V_h^{\pi, \rho}(s) &=\textstyle \EE_{a\sim\pi_h(\cdot|s)}\big[Q_h^{\pi,\rho}(s,a)\big],
\end{align}
\end{subequations}
and the optimal robust policy $\pi^{\star}$ is deterministic. Thus, we can restrict the policy class $\Pi$ to the deterministic one. This leads to the robust Bellman optimality equations:
\begin{subequations}
\label{eq:optimal robust bellman equation}
\begin{align}
     Q_h^{\star, \rho}(s,a) &= \textstyle r_h(s,a) +\inf_{P_h(\cdot|s,a) \in \cU_h^{\rho}(s,a;\bmu_h^0)}[\PP_h V_{h+1}^{\star, \rho}](s,a),\\
    V_h^{\star, \rho}(s)&= \textstyle \max_{a\in\cA}Q_h^{\star}(s,a).
\end{align}
\end{subequations}

\paragraph{Offline Dataset and the Learning Goal.}
Let $\cD$ denote an offline dataset consisting of $K$ i.i.d trajectories generated from the nominal environment MDP$(\cS, \cA,H,P^0, r)$ by a behavior policy $\pi^b=\{\pi_h^b\}_{h=1}^H$. In concrete, for each $\tau\in[K]$, the trajectory $\{(s_h^{\tau}, 
 a_h^{\tau}, r_h^{\tau})\}_{h=1}^H$ satisfies that $a_h^{\tau}\sim \pi_h^b(\cdot|s_h^{\tau})$, $r_h^{\tau}=r_h(s_h^{\tau},a_h^{\tau})$, and $s_{h+1}^{\tau}\sim P^0_h(\cdot|s_h^{\tau},a_h^{\tau})$ for any $h\in[H]$.
The goal of the robust offline RL is to learn the optimal robust policy $\pi^{\star}$ using the offline dataset $\cD$. %
We define the suboptimality gap between any policy $\hat{\pi}$ and the optimal robust policy $\pi^{\star}$ as %
\begin{align}
\label{eq:definition of suboptimality}
    \text{SubOpt}(\hat{\pi}, s_1, \rho) := V_1^{\star, \rho}(s_1) - V_1^{\hat{\pi}, \rho}(s_1).
\end{align}
Then the goal of an algorithm in distributionally robust offline reinforcement learning is to learn a robust policy $\hat{\pi}$ that minimizes the suboptimality gap $\text{SubOpt}(\hat{\pi}, s, \rho)$, for any $s\in\cS$.

\section{Warmup: Robust Pessimistic Value Iteration}\label{sec: Robust Pessimistic Value Iteration}
In this section, we first propose a simple algorithm in \Cref{alg:DR-PVI} as a warm start, and provide an instance-dependent upper bound on its suboptimality gap in \Cref{th:DRPVI}. %

The optimal robust Bellman equation \eqref{eq:optimal robust bellman equation} implies that the optimal robust policy $\pi^\star$ is greedy with respect to the optimal robust Q-function. Therefore, it suffices to estimate $Q_h^{\star,\rho}$ to approximate $\pi^\star$. 
To this end, we estimate the optimal robust Q-function  by iteratively performing an empirical version of the optimal robust Bellman equation similar to \eqref{eq:optimal robust bellman equation}. In concrete, given the estimators at step $h+1$, denoted by $\widehat{Q}_{h+1}(s, a)$ and $\widehat V_{h+1}(s)=\max_{a\in \mathcal{A}} \widehat{Q}_{h+1}(s, a)$, \citet{liu2024distributionally} show that applying one step backward induction similar to \eqref{eq:optimal robust bellman equation} leads to %
\begin{align}
\label{eq:linear form}
   Q_h(s,a)&=\textstyle r_h(s,a) + \inf_{P_h(\cdot|s,a) \in \cU_h^{\rho}(s,a;\bmu_h^0)}\big[\PP_h\widehat{V}_{h+1} \big](s,a) = \big\langle \bphi(s,a), \btheta_h+\bnu_h^{\rho}\big\rangle,
\end{align}
where $\nu_{h,i}^{\rho} := \max_{\alpha \in [0,H]}\{z_{h,i}(\alpha)-\rho(\alpha-\min_{s'}[\widehat{V}_{h+1}(s')]_{\alpha})\}$, $z_{h,i}(\alpha) := \EE^{\mu_{h,i}^0}[\widehat{V}_{h+1}(s') ]_{\alpha}$, $\forall i\in[d]$, $[\widehat{V}_{h+1}(s') ]_{\alpha} = \min\{\widehat{V}_{h+1}(s') ,\alpha\}$, and $\alpha$ is a dual variable stemming from the dual formulation (see \Cref{prop:strong duality for TV}). To estimate $Q_h(s,a)$, it suffices to estimate vectors $\bz_h(\alpha)=[z_{h,1}(\alpha),\ldots,z_{h,d}(\alpha)]$ and $\bnu_h^\rho$ as follows. 
\begin{itemize}[nosep, leftmargin=*]
    \item {\it Estimate $\bz_{h}(\alpha)$:} note that $[\PP_h^0[V_{h+1}]_{\alpha}](s,a) =\la \bphi(s,a),\bz_h(\alpha) \ra$ by \Cref{assumption:linear MDP}, where the expectation is taken with respect to the nominal kernel $P^0_h(\cdot|s,a)$. Given the estimator $\widehat V_{h+1}(s)$, it is natural to estimate $\bz_{h}(\alpha)$ by solving the following ridge regression on the offline dataset $\cD$. 
\begin{align}
\label{eq:ridge regression}
 \hat{\bz}_h(\alpha) &=\textstyle\argmin_{\bz \in \RR^d}\sum_{\tau=1}^K \big(\big[\widehat{V}_{h+1}(s_{h+1}^{\tau})\big]_{\alpha} - \bphi_h^{\tau \top}\bz\big)^2 + \lambda\Vert \bz \Vert_2^2 \notag\\
 &=\textstyle\bLambda_h^{-1} \big[\sum_{\tau=1}^{K}\bphi_h^\tau[\widehat{V}_{h+1}(s_{h+1}^{\tau})]_{\alpha}\big],
\end{align}
where $\lambda>0$, $\bphi_h^\tau$ is a shorthand notation for $\bphi(s_h^{\tau}, a_h^{\tau})$, and $\bLambda_h$ is the covariance matrix:
\begin{align}
\label{eq:covariance matrix}
    \textstyle\bLambda_h=\sum_{\tau=1}^{K}\bphi_h^{\tau}(\bphi_h^{\tau})^{\top}+ \lambda\Ib.
\end{align}

\item {\it Estimate $\hat{\bnu}_h^{\rho}$:}  based on $\hat{z}_{h,i} (\alpha)$, we can estimate $\hat{\nu}_{h,i}^{\rho}$ as follows.
\begin{align}
\label{eq:estimated nu}
    \textstyle \hat{\nu}_{h,i}^{\rho} = \max_{\alpha \in [0,H]}\{\hat{z}_{h,i} (\alpha)-\rho(\alpha - \min_{s'}[\widehat{V}_{h+1}^{\rho}(s')]_{\alpha}\}, \forall i\in[d].
\end{align}
 After these two steps, we immediately obtain the estimated robust Q-function at step $h$, 
 \begin{align}
 \label{eq:ridge estimator of Q-function}
    \textstyle \widehat{Q}_h(s,a) = \big\langle \bphi(s,a), \btheta_h+\hat{\bnu}_h^{\rho}\big\rangle.
 \end{align}
\end{itemize}

Note that these estimations are constructed based on an offline dataset $\cD$, which is known to cause distributional shift. 
We propose to incorporate a penalty term in the estimator \eqref{eq:ridge estimator of Q-function} following the pessimism principle in the face of uncertainty \citep{jin2021pessimism, xie2021bellman, shi2024distributionally}. %

\begin{algorithm}[ht]
   \caption{Distributionally Robust Pessimistic Value Iteration (\algname) }\label{alg:DR-PVI}
    \begin{algorithmic}[1]
        \REQUIRE{
            Input dataset $\cD$ and parameter $\beta_1$; $\widehat{V}^{\rho}_{H+1}(\cdot)=0$. 
                }
        \FOR {$h=H, \cdots, 1$}{
            \STATE $\bLambda_h \leftarrow \sum_{\tau=1}^K \bphi_h^{\tau}\bphi_h^{\tau \top}+\lambda \bI$

            \FOR {$i=1,\cdots,d$}{
                \STATE Update $\hat{\nu}_{h,i}^{\rho}$ according to \eqref{eq:estimated nu}
            }
            \ENDFOR
            \STATE $\Gamma_h(\cdot, \cdot) \leftarrow \beta_1\sum_{i=1}^d\big\Vert\phi_i(\cdot,\cdot)\mathbf{1}_i\big\Vert_{\bLambda_h^{-1}}$ \label{algline:DRPVI-penalty}
            \STATE $\widehat{Q}_h^{\rho}(\cdot, \cdot) \leftarrow \big\{\bphi(\cdot, \cdot)^{\top}(\btheta_h + \hat{\bnu}_h^{\rho}) - \Gamma_h(\cdot, \cdot)\big\}_{[0,H-h+1]}$
            \STATE $\hat{\pi}_h(\cdot|\cdot) \leftarrow \argmax_{\pi_h}\big \la \widehat{Q}_h^{\rho}(\cdot, \cdot), \pi_h(\cdot|\cdot)\big\ra_{\cA}$, and $\widehat{V}_h^{\rho}(\cdot) \leftarrow \la \widehat{Q}_h^{\rho}(\cdot, \cdot), \hat{\pi}_h(\cdot|\cdot)\ra_{\cA}$
        }
        \ENDFOR
    \end{algorithmic}
\end{algorithm}
\begin{remark}\label{remark:DRPVI-design}
    In \Cref{alg:DR-PVI}, the pessimism is achieved by subtracting a robust penalty term, $\sum_{i=1}^d\big\Vert\phi_i(\cdot,\cdot)\mathbf{1}_i\big\Vert_{\bLambda_h^{-1}}$, from the robust Q-function estimation, which is derived from bounding the robust estimation uncertainty arising from $d$ ridge regressions.
    In particular, at step $h\in[H]$, denoting $\alpha^\star_i = \argmax_{[0,H]}\big\{\hat{z}_{h,i}(\alpha) - \rho\big(\alpha - \min_{s'}\big[\widehat{V}_{h+1}^{\rho}(s')\big]_{\alpha}\big)\big\}, \forall i\in[d]$, we solve $d$ separate ridge regressions %
to obtain different coordinates of $\hat{\bnu}_{h}^{\rho}$. %
This design is tailored for the $d$-rectangular linear DRMDP, as we will see, leading to a distinct instance-dependent upper bound in \Cref{th:DRPVI}.
\end{remark}
\begin{remark}
\label{remark:computation tractability}
Notably, to solve the optimization problem with respect to $\alpha\in[0,H]$ in \eqref{eq:estimated nu}, one will repeatedly invoke the closed form solution \eqref{eq:ridge regression} for different values of $\alpha$. Moreover, the optimization %
is decoupled from the state-action pair, due to the decoupling property of $d$-rectangular uncertainty set.
Similar algorithm designs have also appeared in \cite{ma2022distributionally} for Kullback-Leibler divergence based linear DRMDPs and in \cite{liu2024distributionally}  for  online linear DRMDPs.
As for the computational tractability, 
we note that the minimization over $\alpha$ in \eqref{eq:estimated nu} has been implemented in \cite{liu2024distributionally} using the {\it minimize} function in the {\it Nelder-Mead} method \citep{nelder1965simplex} in the Python module {\it scipy.optimize}. The minimization over the state space is avoided under a `fail-state' assumption, common in  applications such as robotics and healthcare (see  Assumption 4.1 and  Remark 4.2 in their paper). Without this assumption, we can also use the Nelder-Mead method to solve it. Thus, \Cref{alg:DR-PVI} is in general computationally tractable. 
\end{remark}

Before presenting the theoretical guarantee of \algname, we make the following data coverage assumption, which is standard for offline linear MDPs \citep{wang2021statistical, duan2020minimax, yin2022near, xiong2023nearly}. 
\begin{assumption}
\label{assumption:feature coverage}
     We assume $\kappa := \min_{h\in[H]}\lambda_{\min}(\EE^{\pi^b, P^0}[\bphi(s_h,a_h)\bphi(s_h,a_h)^{\top}])>0$ for the behavior policy $\pi^b$ and the nominal transition kernel $P^0$. 
\end{assumption}  
\Cref{assumption:feature coverage} requires the behavior policy to sufficiently explore the state-action space under the nominal environment. 
Indeed, it implicitly assumes that the nominal and perturbed environments share the same state-action space, and that the full information of this space is accessible through the nominal environment and the behavior policy $\pi^b$. \Cref{assumption:feature coverage} rules out cases where new states emerge in perturbed environments that can never be queried under the nominal environment as a result of the distribution shift.
Now we present the theoretical guarantee for \Cref{alg:DR-PVI}.
\begin{theorem}
\label{th:DRPVI}
    Under \Cref{assumption:linear MDP,assumption:feature coverage}, $\forall K>\max\{512\log(2dH^2/\delta)/\kappa^2, 20449d^2H^2/\kappa\}$ and $\delta\in(0,1)$, if we set $\lambda=1$ and $\beta_1=\tilde{O}(\sqrt{d}H)$ in \Cref{alg:DR-PVI}, then with probability at least $1-\delta$, $\forall s\in\cS$, the suboptimality of \algname\ satisfies
    \begin{align}
    \label{eq:suboptimality upper bound of DRPVI}
         \text{SubOpt}(\hat{\pi}, s, \rho) \leq \beta_1\cdot\sup_{P\in\cU^{\rho}(P^0)}
        \sum_{h=1}^H\EE^{\pi^\star, P}\bigg[\sum_{i=1}^d\Vert \phi_i(s_h,a_h)\mathbf{1}_i\Vert_{\bLambda_h^{-1}} \big|s_1=s\bigg],
    \end{align}
    where $\bLambda_h$ is the empirical covariance matrix defined in \eqref{eq:covariance matrix}.
\end{theorem}

The result in \Cref{th:DRPVI} %
resembles existing instance-dependent bounds for standard linear MDPs \citep{jin2021pessimism,xiong2023nearly} (see \Cref{table:comparison} for a detailed comparison). However, there are two major distinctions between these results. First, our result depends on the weighted sum of diagonal elements $\sum_{i=1}^d\Vert \phi_i(s_h,a_h)\mathbf{1}_i\Vert_{\bLambda_h^{-1}}$, dubbed as the {\it $d$-rectangular robust estimation error}, instead of the Mahalanobis norm of the feature vector $\Vert \bphi(s_h,a_h)\Vert_{\bLambda_h^{-1}}$. As discussed in \Cref{remark:DRPVI-design}, this term primarily arises due to the necessity to solve $d$ distinct ridge regressions in each step, which presents a unique challenge in our analysis. Second,  we consider the {\it supremum expectation} of $d$-rectangular robust estimation error with respect to all transition kernels in the uncertainty set, which measures the {\it worst case coverage} of the covariance matrix $\bLambda_h$ under the optimal robust policy $\pi^{\star}$. 

{%
To connect with existing literature \citep{blanchet2023double}, we further show that under  \Cref{assumption:feature coverage}, the instance-dependent suboptimality bound can be simplified as follows.
\begin{corollary}
\label{corollary:DRPVI}
Under the same assumptions and settings as \Cref{th:DRPVI}, with probability at least $1-\delta$, for all $s\in\cS$, the suboptimality of \algname\ satisfies   $\text{SubOpt}(\hat{\pi}, s, \rho) = \tilde{O}( {\sqrt{d}H^2}/{(\sqrt{\kappa\cdot K})})$.
\end{corollary}

\begin{remark}
\label{remark: discussion on assumptions}
    Since $\|\bphi(\cdot,\cdot)\|_2\leq1$ by \Cref{assumption:linear MDP}, the coverage parameter $\kappa$ is trivially upper bounded by $1/d$. Assuming that $\kappa = c^{\dagger}/d$ for a constant $0<c^{\dagger}<1$, then we have $\text{SubOpt}(\hat{\pi}, s, \rho) = \tilde{O}( {dH^2}/{(c^{\dagger}\cdot \sqrt{K})})$. 
    This bound improves the state-of-the-art, \citep[Theorem 6.3]{blanchet2023double}, by $O(d)$.
\end{remark}
}

\section{Distributionally Robust Variance-Aware Pessimistic Value Iteration}
\label{sec:Variance-Aware Robust Pessimistic Value Iteration}
The instance-dependent bound in \Cref{th:DRPVI}
has an explicit dependency on $H$, which arises from the fact that $Q^{\rho}_h(s,a)\in[0,H]$ for any $(h,\rho)\in[H]\times(0,1]$ and the Hoeffding-type self-normalized concentration inequality used in our analysis. We will show in this section that the range of any robust value function could be much smaller under a refined analysis. Consequently, we can leverage variance information to improve \Cref{alg:DR-PVI} and achieve a strengthened upper bound.

\paragraph{Intuition} 
In the robust Bellman equation \eqref{eq:robust bellman equation}, the worst-case transition kernel would put as much mass as possible on the minimizer of $V_{h+1}^{\pi, \rho}(s)$, denoted by $s_{\min}$. 
Based on this observation, we conjecture that the robust Bellman equation \eqref{eq:robust bellman equation}  recursively reduces the maximal value of robust value functions, and thus shrinks its range. To see this, we define $\check{\mu}_{h,i}=(1-\rho)\mu^0_{h,i}+\rho\delta_{s_{\min}}$, where $\delta_{s_{\min}}$ is the Dirac measure at $s_{\min}$, and we assume $V_{h+1}^{\pi, \rho}(s_{\min})=0$ for any $(\pi, h)\in\Pi\times[H]$ just for illustration.
It is easy to verify that $\check{\mu}_{h,i}\in\cU_{h,i}^\rho(\mu_{h,i}^0)$ and is indeed the worst-case factor kernel. Then by \eqref{eq:robust bellman equation} we have $V_h^{\pi, \rho}(s) = \EE_{a\sim\pi}[r_h(s,a) + (1 -\rho)[\PP^0_hV_{h+1}^{\pi, \rho}](s,a)]$, which immediately implies $\max_{s\in\cS}V_h^{\pi, \rho}(s)\leq 1+ (1-\rho)\max_{s'\in\cS}V_{h+1}^{\pi, \rho}(s')$. This justifies our conjecture that the range of the robust value functions shrinks over stage. We dub this phenomenon as {\it Range Shrinkage} and summarize it in the following lemma, with a more formal proof postponed to \Cref{sec:proof of range shrinkage}.
\begin{lemma}[Range Shrinkage]
\label{lemma:Range Shrinkage-main paper}
    For any $(\rho,\pi, h) \in (0,1] \times \Pi \times [H]$, we have %
    \begin{align}
        \max_{s\in\cS}V_h^{\pi, \rho}(s) - \min_{s\in\cS}V_h^{\pi, \rho}(s) \leq  \frac{1-(1-\rho)^{H-h+1}}{\rho}.
    \end{align}
\end{lemma}
This phenomenon only appears in DRMDPs since the range of value function is generally $[0,H]$ in standard MDPs. A similar phenomenon is first observed in infinite horizon tabular DRMDPs \citep[Lemma 7]{shi2024curious}. One important implication of \Cref{lemma:Range Shrinkage-main paper} is that the conditional variance of any value function shrinks accordingly. In particular, when $\rho=O(1)$, the range of any robust value function would shrink to constant order, %
which leads to constant order %
conditional variances. This motivates us to leverage the variance information in both algorithm design and theoretical analysis. 
Inspired by the variance-weighted ridge regression in standard linear MDPs \citep{zhou2021nearly, min2021variance, yin2022near, xiong2023nearly}, we propose to improve the vanilla ridge regression in \eqref{eq:ridge regression} by incorporating variance weights. To this end,
we first propose an appropriate variance estimator, whose form is specifically motivated by our theoretical analysis framework, to quantify the variance information.

\begin{algorithm}[t]
    \caption{Distributionally Robust and Variance Aware Pessimistic Value Iteration (\algnameVA)\label{alg:VA-DRPVI}}
    \begin{algorithmic}[1]
        \REQUIRE{Input dataset $\cD$, $\cD'$ and $\beta_2$; $\widehat{V}_{H+1}^{\rho}(\cdot)=0$}
        \STATE Run \Cref{alg:DR-PVI} using dataset $\cD'$ to get $\{\widehat{V}_h^{'\rho}\}_{h\in[H]}$
        \FOR{$h=H, \cdots, 1$}{
            \STATE Construct variance estimator $\widehat{\sigma}^2_h(\cdot, \cdot;\alpha)$ using $\cD'$ by \eqref{eq: ridge_regression_for_estimation_of_variance} and \eqref{eq: variance estimator}
            \label{algline:VA-DRPVI-construct variance estimator}
            \STATE $\bSigma_h(\alpha) = \sum_{\tau=1}^K\bphi_h^{\tau}\bphi_h^{\tau\top}/\widehat{\sigma}_h^2(s_h^{\tau}, a_h^{\tau};\alpha) + \lambda\mathbf{I}$
            \STATE %
            $\hat{\bz}_h(\alpha) = \bSigma_h^{-1}(\alpha)\Big(\sum_{\tau=1}^K\bphi_h^{\tau}\big[\widehat{V}_{h+1}^{\rho}(s_{h+1}^{\tau}) \big]_{\alpha}/\widehat{\sigma}^2_h(s_h^{\tau}, a_h^{\tau};\alpha) \Big)$
            \STATE $\alpha_i=\argmax_{\alpha\in[0,H]}\{\hat{z}_{h,i}(\alpha)-\rho(\alpha-\min_{s'}[\widehat{V}_{h+1}^{\rho}(s')]_{\alpha})\},~\forall i \in [d]$
            \STATE $\hat{\nu}_{h,i}^{\rho} = \hat{z}_{h,i}(\alpha_i)-\rho(\alpha_i-\min_{s'}[\widehat{V}_{h+1}^{\rho}(s')]_{\alpha_i}),~\forall i \in [d]$
            \STATE $\Gamma_h(\cdot, \cdot) \leftarrow \beta_2\sum_{i=1}^d\Vert \phi_i(\cdot, \cdot)\mathbf{1}_i\Vert_{\bSigma_h^{-1}(\alpha_i)}$
            \STATE $\widehat{Q}_h^{\rho}(\cdot, \cdot) = \{\bphi(\cdot, \cdot)^{\top}(\btheta_h + \hat{\bnu}_h^{\rho})-\Gamma_h(\cdot, \cdot)\}_{[0, H-h+1]}$
            \STATE $\hat{\pi}_h(\cdot|\cdot) \leftarrow \argmax_{\pi_h}\la \widehat{Q}_h^{\rho}(\cdot, \cdot), \pi_h(\cdot|\cdot)\ra_{\cA}$, $\widehat{V}_{h}^{\rho}(\cdot) \leftarrow \la \widehat{Q}_h^{\rho}(\cdot, \cdot),  \hat{\pi}_h(\cdot|\cdot)\ra_{\cA}$
        }
        \ENDFOR
    \end{algorithmic}
\end{algorithm}

\paragraph{Variance estimation}
We first run \Cref{alg:DR-PVI} using an offline dataset $\cD'$ that is independent of $\cD$ to obtain estimators of the optimal robust value functions $\{\widehat{V}^{'\rho}_h\}_{h\in[H]}$. By \Cref{assumption:linear MDP}, the variance of $[\widehat{V}^{'\rho}_{h+1}]_{\alpha}$ under the nominal environment is $[\text{Var}_h[\widehat{V}_{h+1}^{'\rho}]_{\alpha}](s,a) = [\PP^0_h[\widehat{V}_{h+1}^{'\rho}]_{\alpha}^2](s,a) - ([\PP^0_h[\widehat{V}_{h+1}^{'\rho}]_{\alpha}](s,a))^2 = \la\bphi(s,a), \bz_{h,2} \ra - (\la\bphi(s,a), \bz_{h,1} \ra)^2$. We  estimate $\bz_{h,1}$ and $\bz_{h,2}$ via ridge regression similarly as in \eqref{eq:ridge regression}:
\begin{subequations}
\label{eq: ridge_regression_for_estimation_of_variance}
\begin{align}
    \tilde{\bz}_{h,2}(\alpha) &= \textstyle \argmin_{\bz \in \RR^d}\sum_{\tau=1}^K \big(\big[\widehat{V}^{'\rho}_{h+1}(s_{h+1}^{\tau})\big]^2_{\alpha} - \bphi_h^{\tau \top}\bz\big)^2 + \lambda\Vert \bz \Vert_2^2, \label{eq: ridge_regression_for_estimation_of_variance_2}\\
    \tilde{\bz}_{h,1}(\alpha) & = \textstyle \argmin_{\bz \in \RR^d}\sum_{\tau=1}^K \big(\big[\widehat{V}^{'\rho}_{h+1}(s_{h+1}^{\tau})\big]_{\alpha} - \bphi_h^{\tau \top}\bz\big)^2 + \lambda\Vert \bz\Vert_2^2\label{eq: ridge_regression_for_estimation_of_variance_1}.
\end{align}
\end{subequations}
We then construct the following truncated variance estimator 
{\small
\begin{align}
\label{eq: variance estimator}
    \widehat{\sigma}_h^2(s,a; \alpha) := \max\bigg\{1, \big[\bphi(s,a)^{\top}\tilde{\bz}_{h,2}(\alpha) \big]_{[0,H^2]} - \big[\bphi(s,a)^{\top}\tilde{\bz}_{h,1}(\alpha) \big]_{[0,H]}^2 - \tilde{O}\bigg(\frac{dH^3}{\sqrt{K\kappa}}\bigg)\bigg\},
\end{align}}%
where the last term is a penalty to achieve pessimistic estimations of conditional variances.

\paragraph{Variance-Aware Function Approximation Mechanism}
Similar to the two-step estimation procedure of \Cref{alg:DR-PVI}, we first estimate $\bz_{h}(\alpha)$ by  the following variance-weighted ridge regression under the nominal environment:
{\small
\begin{align}
    \label{eq: weighted ridge regression VA}
 \hat{\bz}_h(\alpha) &=\argmin_{\bz \in \RR^d}\sum_{\tau=1}^K \frac{\big(\big[\widehat{V}_{h+1}^{\rho}(s_{h+1}^{\tau})\big]_{\alpha} - \bphi_h^{\tau \top}\bz\big)^2}{\widehat{\sigma}_h^2(s_h^{\tau}, a_h^{\tau}; \alpha)} + \lambda\Vert \bz \Vert_2^2 \notag\\
 &= \bSigma_h^{-1}(\alpha) \bigg[\sum_{\tau=1}^{K}\frac{\bphi_h^\tau[\widehat{V}^{\rho}_{h+1}(s_{h+1}^{\tau})]_{\alpha}}{\widehat{\sigma}_h^2(s_h^{\tau}, a_h^{\tau}; \alpha)}\bigg],
\end{align}}%
where $\bSigma_h(\alpha)=\sum_{\tau=1}^K\bphi_h^{\tau}\bphi_h^{\tau\top}/\widehat{\sigma}^2_h(s_h^{\tau}, a_h^{\tau};\alpha)+\lambda \Ib$ is the empirical variance-weighted covariance matrix, which can be deemed as an estimator of the following variance-weighted covariance matrix
\begin{align}\label{eq:variance weighted covariance matrix_population_version}
\textstyle \bSigma_h^{\star}=\sum_{\tau=1}^K\bphi_h^{\tau}\bphi_h^{\tau\top}/[\VV_h V_{h+1}^{\star,\rho}](s_h^{\tau}, a_h^{\tau})+\lambda\mathbf{I}.
\end{align}
In the second step, we estimate $\nu_{h,i}^{\rho}, \forall i\in [d]$ in the same way as \eqref{eq:estimated nu}.
We then add a pessimism penalty based on $\bSigma_h(\alpha)$. 
We present the full algorithm details in \Cref{alg:VA-DRPVI}.

\begin{theorem}
\label{th:VA-DRPVI}
    Under \Cref{assumption:linear MDP,assumption:feature coverage}, for $K>\max\{\tilde{O}(d^2H^6/\kappa), \tilde{O}(H^4/\kappa^2)\}$ and $\delta\in (0,1)$, if we set $\lambda=1/H^2$ and $\beta_2=\tilde{O}(\sqrt{d})$ in \Cref{alg:VA-DRPVI}, then with probability at least $1-\delta$, the suboptimality of \algnameVA\ satisfies
    \begin{align}
    \label{eq:suboptimality bound of VA-DRPVI}
         \text{SubOpt}(\hat{\pi}, s, \rho) \leq \beta_2 \cdot \sup_{P\in\cU^{\rho}(P^0)}\sum_{h=1}^H\EE^{\pi^{\star}, P}\bigg[\sum_{i=1}^d\Vert\phi_i(s_h,a_h)\mathbf{1}_i \Vert_{\bSigma_h^{\star-1}}\big|s_1=s\bigg],
    \end{align}
    where $\bSigma_h^\star$ is the population variance-weighted covariance matrix defined as in \eqref{eq:variance weighted covariance matrix_population_version}. %
\end{theorem}

Note that the bound in \Cref{th:VA-DRPVI} %
does not explicitly depend on $H$ anymore compared with that in \Cref{th:DRPVI}. %
A naive observation is that $[\VV_h V_{h+1}^{\star,\rho}](s, a)\in [1,H^2]$. By comparing the definitions in \eqref{eq:covariance matrix} and \eqref{eq:variance weighted covariance matrix_population_version}, we have $\bSigma_h^{\star-1} \preceq H^2\bLambda_h^{-1}$. Thus the upper bound of \Cref{alg:VA-DRPVI} is never worse than that of \Cref{alg:DR-PVI}.
This improvement brought by variance information is similar to that in standard linear MDPs \citep[Theorem 2]{xiong2023nearly}.
However, thanks to the range shrinkage phenomenon, we can further show that \algnameVA\ is strictly better than \algname\ when the uncertainty level is of constant order. %
\begin{corollary}
\label[corollary]{corollary:VA-DRPVI}
    Under the same assumptions and settings as \Cref{th:VA-DRPVI}, given the uncertainty level $\rho$, we have with probability at least $1-\delta$, for all $s\in \cS$, the suboptimality of \algnameVA\ satisfies
    \begin{align*}
       \text{SubOpt}(\hat{\pi}, s, \rho) \leq \beta_2\cdot\frac{(1-(1-\rho)^H)}{\rho}\cdot\sup_{P\in\cU^{\rho}(P^0)}\sum_{h=1}^H\EE^{\pi^{\star}, P}\bigg[\sum_{i=1}^d\Vert\phi_i(s_h,a_h)\mathbf{1}_i \Vert_{\bLambda_h^{-1}}|s_1=s\bigg].
    \end{align*}
\end{corollary}
\begin{remark}
Note that $(1-(1-\rho)^H)/\rho = \bTheta(\min\{1/\rho, H\})$.
When $\rho=O(1)$, the suboptimality of \Cref{alg:VA-DRPVI} is strictly smaller than that of \Cref{alg:DR-PVI} by $H$. With a similar argument as in \Cref{remark: discussion on assumptions}, if we assume there exist a constant $0<c^{\dagger}<1$, such that $\kappa = c^{\dagger}/d$ in \Cref{assumption:feature coverage}, then the instance-dependent upper bound can be simplified to $\tilde{O}( {dH\min\{1/\rho, H\}}/{(c^{\dagger}\cdot\sqrt{K}}))$,
which improves the state-of-the-art \citep[Theorem 6.3]{blanchet2023double} by $O(dH)$ when $\rho=O(1)$.
\end{remark}

\section{Information-Theoretic Lower Bound}
\label{sec:lower bound}
For a matrix $\Ab\in\RR^{d\times d}$ and a state $s\in\cS$, we define  function $\Phi(\cdot, \cdot):\RR^{d\times d}\times\cS\rightarrow\RR$ as follows.%
\begin{align}\label{eq:def-instance-dependent-uncertainty-function}
\Phi(\Ab,s) = \sup_{P\in\cU^{\rho}(P^0)}
\sum_{h=1}^H\EE^{\pi^\star, P}\bigg[\sum_{i=1}^d\Vert \phi_i(s_h,a_h)\mathbf{1}_i\Vert_{\Ab} \big|s_1=s\bigg].
\end{align}
It can be seen our upper bounds in previous sections primarily depend on quantities such as $\Phi(\bLambda_h^{-1},s)$ and $\Phi(\bSigma_h^{\star -1},s)$. Roughly speaking, these quantities characterize the discrepancy between the (weighted) covariance matrix of the offline dataset and the state action pairs generated from the transition probability in the uncertainty set. Hence we call $\Phi(\cdot, \cdot)$ the uncertainty function.

We now establish an information-theoretic lower bound to show that the uncertainty function is unavoidable for $d$-rectangular linear DRMDPs. Let $\cM$ be a class of DRMDPs and we define $\text{SubOpt}(M,\hat{\pi},s,\rho)$ as the suboptimality gap specific to one DRMDP instance $M\in\cM$.
\begin{theorem}
    \label{th:lower bound}
    Given uncertainty level $\rho\in(0,3/4]$, dimension $d$, horizon length $H$ and sample size $K>\tilde{O}(d^6)$, there exists a class of $d$-rectangular linear DRMDPs $\cM$ and an offline dataset $\cD$ of size $K$ such that for all $s\in\cS$, with probability at least $1-\delta$, 
    $\inf_{\hat{\pi}}\sup_{M\in\cM} \text{SubOpt}(M, \hat{\pi},s,\rho) \geq c \cdot \Phi(\bSigma_h^{\star-1}, s)$, where $c$ is a universal constant.
\end{theorem}
\Cref{th:lower bound} shows that the uncertainty function $\Phi(\bSigma_h^{\star-1}, s)$ is intrinsic to the information-theoretic lower bound, and thus is inevitable. It is noteworthy that the lower bound in \Cref{th:lower bound} aligns with the upper bound in \Cref{th:VA-DRPVI} up to a factor of $\beta_2$, which implies that \algnameVA\ is minimax optimal in the sense of information theory, but with a small gap of $\tilde{O}(\sqrt{d})$. Consequently, we affirm that, in the context of robust offline reinforcement learning with function approximation, both the computational efficiency and minimax optimality are achievable under the setting of $d$-rectangular linear DRMDPs with TV uncertainty sets. Moreover, \Cref{th:lower bound} also suggests that achieving a good robust policy necessitates the worst case coverage of the offline dataset over the entire uncertainty set of transition models, which is significantly different from standard linear MDPs where a good coverage under the nominal model is enough \citep{jin2021pessimism,yin2022near,xiong2023nearly}. Such a distinction indicates that learning in linear DRMDPs may be more challenging in comparison to standard linear MDPs.

{%
Further, we highlight that the hard instances we constructed also satisfy \Cref{assumption:feature coverage}. It remains an interesting direction to explore what would happen if the nominal and perturbed environments don't share exactly the same state-action space. We conjecture that since there could be absolutely new states emerging in perturbed environments that can never be explored in the nominal environment, the policy learned merely using data collected from the nominal environment could be arbitrarily bad.
}

\paragraph{Challenges and novelties in construction of hard instances}
Existing tight lower bound analysis in standard linear MDPs \citep{zanette2021provable, yin2022near, xiong2023nearly} generally depends on the Assouad's method and a family of hard instances indexed by $\bxi\in\{-1,1\}^{dH}$. 
However, they do not consider model uncertainty, which largely hinders the derivation of explicit formulas for the robust value functions. Further, one prerequisite of the Assouad's method is switching the initial minimax suboptimality $\inf_{\hat{\pi}}\max_{M\in\cM}\text{SubOpt}(\hat{\pi}, s, \rho)$ to a risk of the form $\inf_{\bxi'}\max_{\bxi}D_H(\bxi,\bxi')$, where $D_H(\cdot,\cdot)$ is the Hamming distance.
The model uncertainty significantly complicates this procedure, as the nonlinearity involved disrupts the linear dependency between the value function and the index $\bxi$. At the core of \Cref{th:lower bound} is a novel class of hard instances $\cM$. At a high-level, the hard instances should (1) fulfill the $d$-rectangular linear DRMDP conditions, (2) mitigate the nonlinearity caused by model uncertainty, (3) achieve the prerequisite for Assouad's method, and (4) be concise enough to admit matrix analysis.
We postpone details on the construction of hard instances and the proof of \Cref{th:lower bound} to \Cref{sec:proof of lower bound}.

As a side product of \Cref{th:lower bound}, we show in the following corollary an information-theoretic lower bound in terms of the instance-dependent uncertainty function $\Phi(\bLambda_h^{-1}, s)$ in \Cref{th:DRPVI}. 
\begin{corollary}
    \label{corollary:lower bound} 
    Under the same setting in \Cref{th:lower bound}, the class of hard instances $\cM$ and offline dataset $\cD$ in \Cref{th:lower bound} also suggests that, with probability at least $1-\delta$, $\inf_{\hat{\pi}}\sup_{M\in\cM} \text{SubOpt}(\hat{\pi},s,\rho) \geq c \cdot \Phi(\bLambda_h^{-1}, s)$, where $c$ is a universal constant.
\end{corollary}
 
This implies that the uncertainty function $\Phi(\bLambda_h^{-1}, s)$ in \Cref{th:DRPVI} also arises from the information-theoretic lower bound. 
We note the lower bound in \Cref{corollary:lower bound} matches the upper bound in \Cref{th:DRPVI} up to $\beta_1$, thus \algname\ is also minimax optimal in the sense of information theory, but with a larger gap of $\tilde{O}(\sqrt{d}H)$. Moreover, the only difference between \Cref{th:lower bound} and \Cref{corollary:lower bound} is the covariance matrix. Due to the fact that $\bLambda_h^{-1} \preceq \bSigma_h^{\star,-1}$, the information-theoretic lower bound in \Cref{th:lower bound} is indeed tighter than that in \Cref{corollary:lower bound}.

\section{Conclusions}
We studied robust offline RL with function approximation under the setting of $d$-rectangular linear DRMDPs with TV uncertainty sets. We first proposed the \algname\ algorithm and built up a theoretical analysis pipeline to establish the first instance-dependent upper bound on the suboptimality gap in the context of robust offline RL. 
We then showed an interesting range shrinkage phenomenon specific to DRMDPs, and we proposed the \algnameVA\ algorithm, which leverages the conditional variance information of the optimal robust value function. Based on the analysis pipeline built above, we show that the upper bound of \algnameVA\ achieves sharp dependence on the horizon length $H$. 
In addition, we found that an uncertainty function consisting of two crucial quantities--a supremum over uncertainty set and a diagonal-based normalization--appears in all upper bounds.
We further established an information-theoretic lower bound to prove that the uncertainty function is unavoidable for robust offline RL under the setting of $d$-rectangular linear DRMDPs.

It remains an interesting future research question whether the computational and provable efficiency can be achieved in other settings for robust offline RL with function approximation. 
Another interesting future direction is to explore the unique challenges of applying general function approximation techniques in standard offline RL \citep{chen2019information} to DRMDPs. %

\section*{Acknowledgments}
We would like to thank the anonymous reviewers for their helpful comments. ZL and PX was supported in part by the National Science Foundation (DMS-2323112) and the Whitehead Scholars Program at the Duke University School of Medicine. The views and conclusions in this paper are those of the authors and should not be interpreted as representing any funding agency.

\bibliographystyle{plainnat}
\bibliography{reference}

\clearpage
\appendix

\newpage
\section{Additional Related Work}
\paragraph{Offline Linear MDPs.} Our work focuses on the offline linear MDP setting where the nominal transition kernel, from which the offline dataset is collected, admits the linear MDP structure. Numerous works have studied the provable efficiency and statistical limits of algorithms under this setting \citep{jin2021pessimism, zanette2021provable, xie2021bellman, yin2022near, xiong2023nearly}.  The most relevant study to ours is the recent work of \cite{xiong2023nearly}, which established the minimax optimality of offline linear MDPs. At the core of their analysis is an advantage-reference technique designed for offline RL under linear function approximation, together with a variance aware pessimism-based algorithm.
However, the offline linear MDP setting still remains understudied in the context of DRMDPs.

\paragraph{Transfer-Learning in Low Rank MDPs.} Besides the distributionally robust perspective to solve the planning problem in a nearly unknown target environment, another line of work focuses on transfer learning in low-rank MDPs \citep{cheng2022provable, lu2022provable, agarwal2023provable, bose2024offline}. Specifically, the problem setup assumes that the agent has access to information of several source tasks. The agent learns a common representation from the source domains and then leverages the learned representation to learn a policy performing well in the target tasks with limited information. This setting is in stark contrast to DRMDPs, where the agent only has access to the information of a single source domain, without any available information of the target domain, assuming the same task is being performed. This motivates the pessimistic principle of the distributionally robust perspective.
Among the aforementioned works, \citet{bose2024offline} studied the offline multi-task RL, which is the most closely related to our setting. In particular, they investigate the representation transfer error in their Theorem 1, stating that the learned representation can lead to a transition kernel that is close to the target kernel in terms of the TV divergence. Note that the uncertainty is induced by the representation estimation error, which is different from our setting assuming that the uncertainty comes from perturbations on underlying factor distributions. Nevertheless, this work provides evidence that TV divergence is a reasonable measure to quantify the uncertainty in transition kernels and motivates a future research direction in learning robust policies that are robust to the uncertainty induced by the representation estimation error.

\section{A More Computationally Efficient Variant of \algnameVA}
\label{sec:The Modified VA-DRPVI Algorithm}
In this section, we propose a modified version of \Cref{alg:VA-DRPVI}, which reduces the computation cost in the ridge regressions for variance estimation and achieves the same theoretical guarantees.

\paragraph{Variance Estimator.} In  \Cref{sec:Variance-Aware Robust Pessimistic Value Iteration}, we estimate the variance of the truncated robust value function $[\widehat{V}^{'\rho}_{h+1}]_{\alpha}$. Thus, for different $\alpha$, we need to establish different variance estimators, which significantly increases the computational burden. The theoretical analysis of \Cref{alg:VA-DRPVI} suggests that it suffices to estimate the the variance of $\widehat{V}^{'\rho}_{h+1}$, instead of the truncated one. In particular, we know $[\Var_h\widehat{V}_{h+1}^{'\rho}](s,a) = [\PP^0_h(\widehat{V}_{h+1}^{'\rho})^2](s,a) - ([\PP^0_h\widehat{V}_{h+1}^{'\rho}](s,a))^2 = \la\bphi(s,a), \bz_{h,2} \ra - (\la\bphi(s,a), \bz_{h,1} \ra)^2$. 
Then we estimate $\bz_{h,1}$ and $\bz_{h,2}$ via ridge regression:
\begin{subequations}
\label{eq: modified_ridge_regression_for_estimation_of_variance}
\begin{align}
    \tilde{\bz}_{h,2} &= \argmin_{\bz \in \RR^d}\sum_{\tau=1}^K \big(\big(\widehat{V}^{'\rho}_{h+1}(s_{h+1}^{\tau})\big)^2 - \bphi_h^{\tau \top}\bz\big)^2 + \lambda\Vert \bz \Vert_2^2, \label{eq: modified_ridge_regression_for_estimation_of_variance_2}
    \\
    \tilde{\bz}_{h,1} & = \argmin_{\bz \in \RR^d}\sum_{\tau=1}^K \big(\widehat{V}^{'\rho}_{h+1}(s_{h+1}^{\tau}) - \bphi_h^{\tau \top}\bz\big)^2 + \lambda\Vert \bz\Vert_2^2.\label{eq: mpdified_ridge_regression_for_estimation_of_variance_1}
\end{align}
\end{subequations}
We construct the following truncated variance estimator:
\begin{align}
\label{eq: modified variance estimator}
    \widehat{\sigma}_h^2(s,a) := \max\Big\{1, \big[\bphi(s,a)^{\top}\tilde{\bz}_{h,2} \big]_{[0,H^2]} - \big[\bphi(s,a)^{\top}\tilde{\bz}_{h,1} \big]_{[0,H]}^2 - \tilde{O}\Big(\frac{dH^3}{\sqrt{K\kappa}}\Big)\Big\}.
\end{align}
The modified variance-aware algorithm is presented in \Cref{alg:modified VA-DRPVI} and the theoretical guarantee is presented in \Cref{th:modified VA-DRPVI}.

\begin{algorithm}
    \caption{Modified \algnameVA}\label{alg:modified VA-DRPVI}
    \begin{algorithmic}[1]
        \REQUIRE{Input dataset $\cD$, $\cD'$ and $\beta_2$; $\widehat{V}_{H+1}^{\rho}(\cdot)=0$}
        \STATE Run \Cref{alg:DR-PVI} using dataset $\cD'$ to get $\{\widehat{V}_h^{'\rho}\}_{h\in[H]}$
        \FOR{$h=H, \cdots, 1$}{
            \STATE Construct variance estimator $\widehat{\sigma}^2_h(\cdot, \cdot)$ using $\cD'$ by \eqref{eq: modified_ridge_regression_for_estimation_of_variance} and \eqref{eq: modified variance estimator}
            \STATE $\bSigma_h = \sum_{\tau=1}^K\bphi_h^{\tau}\bphi_h^{\tau\top}/\widehat{\sigma}_h^2(s_h^{\tau}, a_h^{\tau}) + \lambda\mathbf{I}$
            \STATE $\hat{\bz}_h(\alpha) = \bSigma_h^{-1}\Big(\sum_{\tau=1}^K\bphi_h^{\tau}\big[\widehat{V}_{h+1}^{\rho}(s_{h+1}^{\tau}) \big]_{\alpha}/\widehat{\sigma}^2_h(s_h^{\tau}, a_h^{\tau}) \Big)$
            \STATE $\alpha_i=\argmax_{\alpha\in[0,H]}\{\hat{z}_{h,i}(\alpha)-\rho(\alpha-\min_{s'}[\widehat{V}_{h+1}^{\rho}(s')]_{\alpha})\},~\forall i \in [d]$
            \STATE $\hat{\nu}_{h,i}^{\rho} = \hat{z}_{h,i}(\alpha_i)-\rho(\alpha_i-\min_{s'}[\widehat{V}_{h+1}^{\rho}(s')]_{\alpha_i}),~\forall i \in [d]$
            \STATE $\Gamma_h(\cdot, \cdot) \leftarrow \beta_2\sum_{i=1}^d\Vert \phi_i(\cdot, \cdot)\mathbf{1}_i\Vert_{\bSigma_h^{-1}}$
            \STATE $\widehat{Q}_h^{\rho}(\cdot, \cdot) = \{\bphi(\cdot, \cdot)^{\top}(\btheta_h + \hat{\bnu}_h^{\rho})-\Gamma_h(\cdot, \cdot)\}_{[0, H-h+1]}$
            \STATE $\hat{\pi}_h(\cdot|\cdot) \leftarrow \argmax_{\pi_h}\la \widehat{Q}_h^{\rho}(\cdot, \cdot), \pi_h(\cdot|\cdot)\ra_{\cA}$, $\widehat{V}_{h}^{\rho}(\cdot) \leftarrow \la \widehat{Q}_h^{\rho}(\cdot, \cdot),  \hat{\pi}_h(\cdot|\cdot)\ra_{\cA}$
        }
        \ENDFOR
    \end{algorithmic}
\end{algorithm}

\begin{theorem}
\label{th:modified VA-DRPVI}
    Under \Cref{assumption:linear MDP,assumption:feature coverage}, for $K>\max\{\tilde{O}(d^2H^6/\kappa), \tilde{O}(H^4/\kappa^2)\}$ and $\delta\in (0,1)$, if we set $\lambda=1/H^2$ and $\beta_2=\tilde{O}(\sqrt{d})$ in \Cref{alg:modified VA-DRPVI}, then with probability at least $1-\delta$, for all $s\in\cS$, the suboptimality of \algnameVA\ satisfies
    \begin{align}
        \text{SubOpt}(\hat{\pi}, s, \rho) \leq \beta_2 \cdot \sup_{P\in\cU^{\rho}(P^0)}\sum_{h=1}^H\EE^{\pi^{\star}, P}\Big[\sum_{i=1}^d\Vert\phi_i(s_h,a_h)\mathbf{1}_i \Vert_{\bSigma_h^{\star-1}}|s_1=s\Big],
    \end{align}
    where $\bSigma_h^{\star}=\sum_{\tau=1}^K\bphi_h^{\tau}\bphi_h^{\tau\top}/[\VV_h V_{h+1}^{\star}](s_h^{\tau}, a_h^{\tau})+\lambda\mathbf{I}$.
\end{theorem}

\begin{remark}
    The computation cost of \Cref{alg:modified VA-DRPVI} is much smaller than \Cref{alg:VA-DRPVI}, as the variance estimators are not related to $\alpha$ anymore. Notably, \Cref{alg:modified VA-DRPVI} shares the same upper bound as \Cref{alg:VA-DRPVI}. According to \Cref{th:lower bound}, we know the modified algorithm is also minimax optimal.
    
\end{remark}

\section{Experiments}

We conduct numerical experiments to illustrate the performances of our proposed algorithms, \algname\ and \algnameVA, and compare it with the  their non-robust counterpart, PEVI \citep{jin2021pessimism}. All numerical experiments were conducted on a MacBook Pro with a 2.6 GHz 6-Core Intel CPU. The implementation of our \algname\ algorithm is available at \url{https://github.com/panxulab/Offline-Linear-DRMDP}.

\paragraph{Construction of the simulated linear MDP} We leverage the simulated linear MDP setting proposed by \citet{liu2024distributionally} and modify it as an offline RL problem. 
In particular, the source and target linear MDP environment are shown in \Cref{fig:mdp_5states} and \Cref{fig:perturbed_mdp}. The state space is set to be $\cS = \{x_1,\cdots,x_5\}$ and the action space is to be $\cA=\{-1,1\}^4\subset\RR^4$. At each episode, the state always starts with $x_1$, and then transits to $x_2, x_4, x_5$ with probability defined in the figures. $x_2$ is an intermediate state, and it can transit to $x_3, x_4, x_5$ with probability defined on the lines. Moreover, Both $x_4$ and $x_5$ are absorbing states. $x_4$ ($x_5$) is the fail state (goal state), and the reward starting from which is always 0 (1).  
The reward functions and transition probabilities are designed to depend on the hyperparameter $\bxi\in\RR^4$ as shown in the figure.
 The target environment is constructed by only perturbing the transition probability at $x_1$ of the source environment, and the extend of perturbation is controlled by the hyperparameter $q\in(0,1)$. We refer more details on the construction of the simulated linear DRMDP to the Supplementary A.1 of \cite{liu2024distributionally}. 

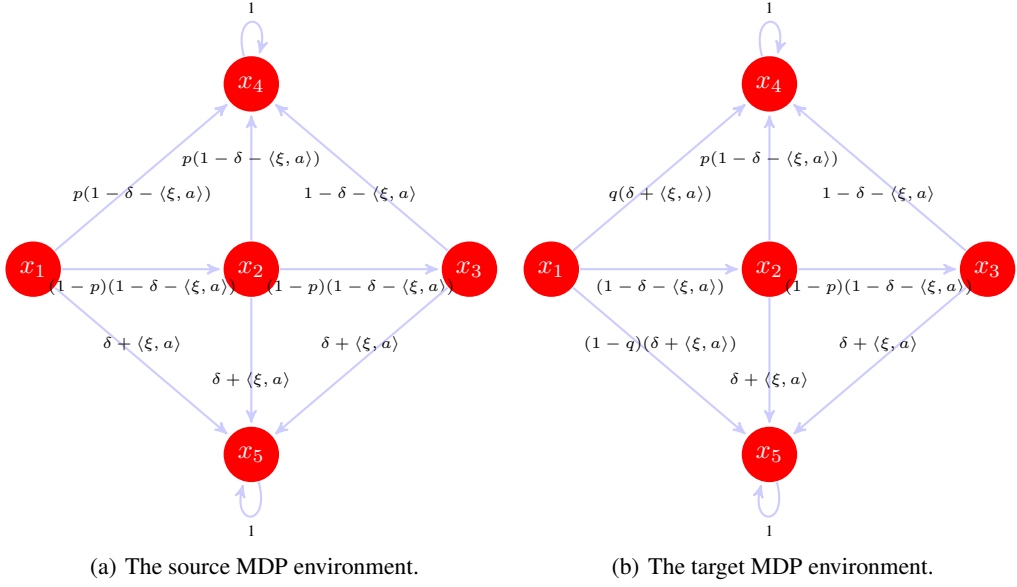
\begin{figure*}[ht]
    \centering
    \subfigure[The source MDP environment.]{
        \begin{tikzpicture}[->,>=stealth',shorten >=1pt,auto,node distance=2.9cm,thick]
            \tikzstyle{every state}=[fill=red,draw=none,text=white,minimum size=0.5cm]
            \node[state] (S1) {$x_1$};
            \node[state] (S2) [right of=S1] {$x_2$};
            \node[state] (S3) [right of=S2] {$x_3$};
            \node[state] (S4) [above=1.7cm of S2] {$x_4$};
            \node[state] (S5) [below=1.7cm of S2] {$x_5$};            
            \path   (S1) edge[draw=blue!20] node[below] {\tiny $(1-p)(1-\delta-\la\xi,a\ra)$} (S2)
                         edge[draw=blue!20] node[below] {\tiny$p(1-\delta-\la\xi,a\ra)$} (S4)
                         edge[draw=blue!20] node[above] {\tiny$\delta+\la\xi,a\ra$} (S5)
                    (S2) edge[draw=blue!20] node[below] {\tiny$(1-p)(1-\delta-\la\xi,a\ra)$} (S3)
                         edge[draw=blue!20] node[above] {\tiny$p(1-\delta-\la\xi,a\ra)$} (S4)
                         edge[draw=blue!20] node[below] {\tiny$\delta+\la\xi,a\ra$} (S5)
                    (S3) edge[draw=blue!20] node[below] {\tiny$1-\delta-\la\xi,a\ra$} (S4)
                         edge[draw=blue!20] node[above] {\tiny$\delta+\la\xi,a\ra$} (S5)
                    (S4) edge[draw=blue!20] [loop above] node {\tiny 1} (S4)
                    (S5) edge[draw=blue!20] [loop below] node {\tiny 1} (S5);
        \end{tikzpicture}
        \label{fig:mdp_5states}
    }
    \subfigure[The target MDP environment.]{
        \begin{tikzpicture}[->,>=stealth',shorten >=1pt,auto,node distance=2.9cm,thick]
            \tikzstyle{every state}=[fill=red,draw=none,text=white,minimum size=0.5cm]
            \node[state] (S1) {$x_1$};
            \node[state] (S2) [right of=S1] {$x_2$};
            \node[state] (S3) [right of=S2] {$x_3$};
            \node[state] (S4) [above=1.7cm of S2] {$x_4$};
            \node[state] (S5) [below=1.7cm of S2] {$x_5$};            
            \path   (S1) edge[draw=blue!20] node[below] {\tiny$(1-\delta-\la\xi,a\ra)$} (S2)
                         edge[draw=blue!20] node[below] {\tiny$q(\delta+\la\xi,a\ra)$} (S4)
                         edge[draw=blue!20] node[above] {\tiny$(1-q)(\delta+\la\xi,a\ra)$} (S5)
                    (S2) edge[draw=blue!20] node[below] {\tiny$(1-p)(1-\delta-\la\xi,a\ra)$} (S3)
                         edge[draw=blue!20] node[above] {\tiny$p(1-\delta-\la\xi,a\ra)$} (S4)
                         edge[draw=blue!20] node[below] {\tiny$\delta+\la\xi,a\ra$} (S5)
                    (S3) edge[draw=blue!20] node[below] {\tiny$1-\delta-\la\xi,a\ra$} (S4)
                         edge[draw=blue!20] node[above] {\tiny$\delta+\la\xi,a\ra$} (S5)
                    (S4) edge[draw=blue!20] [loop above] node {\tiny 1} (S4)
                    (S5) edge[draw=blue!20] [loop below] node {\tiny 1} (S5);
        \end{tikzpicture}
        \label{fig:perturbed_mdp}
    }
    \caption{The source and the target linear MDP environments. The value on each arrow represents the transition probability. For the source MDP, there are five states and three steps, with the initial state being  $x_1$, the fail state being $x_4$, and $x_5$ being an absorbing state with reward 1. The target MDP on the right is obtained by perturbing the transition probability at the first step of the source MDP, with others remaining the same. 
    }
\end{figure*}

\paragraph{Implementation} We simply use the random policy that chooses actions uniformly at random at any $(s, a, h) \in \cS\times\cA\times[H]$ to collect offline dataset. The offline dataset containing 100 trajectories collected by the behavior policy from the source environment. We conduct ablation study by setting the hyperpameter $\bxi = (1/\|\bxi\|_1, 1/\|\bxi\|_1, 1/\|\bxi\|_1, 1/\|\bxi\|_1)^{\top}$ and consider different choices of $\|\bxi\|_1\in\{0.1, 0.2, 0.3\}$.
 Following \cite{liu2024distributionally}, we use heterogeneous uncertainty level for our two algorithms. Specifically, we set $\rho_{1,4}=0.5$ and $\rho_{h,i}=0$ for all other cases. The experiment results are shown in \Cref{fig:simulation-results-app}.

\begin{figure*}[!thbp]
    \centering
      \subfigure[$\Vert\xi\Vert_1 = 0.1$, $\rho_{1,4}=0.3$]{\includegraphics[scale=0.33]{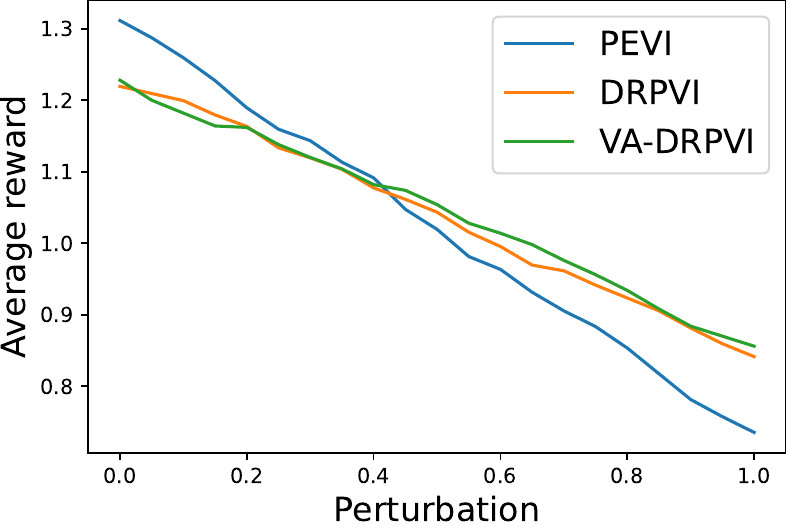}
      \label{fig:simulated_MDP_xi01_rho05}}
      \subfigure[$\Vert\xi\Vert_1 = 0.1$, $\rho_{1,4}=0.4$]{\includegraphics[scale=0.33]{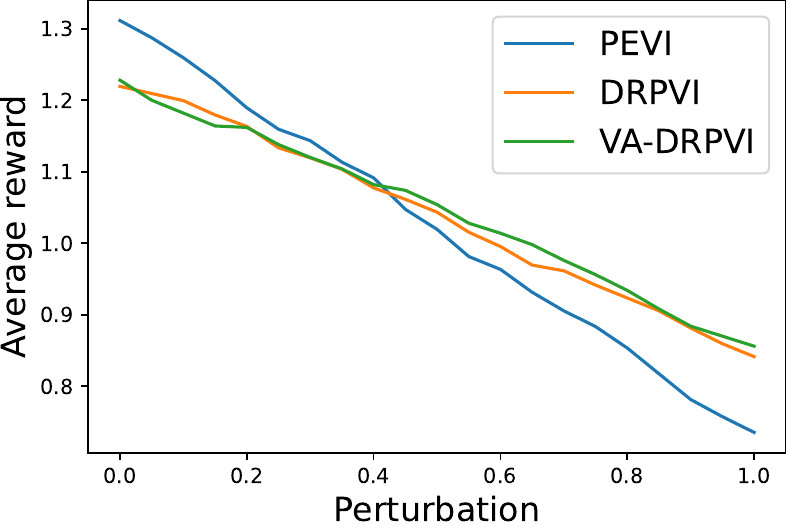}}
      \subfigure[$\Vert\xi\Vert_1 = 0.1$, $\rho_{1,4}=0.5$]{\includegraphics[scale=0.33]{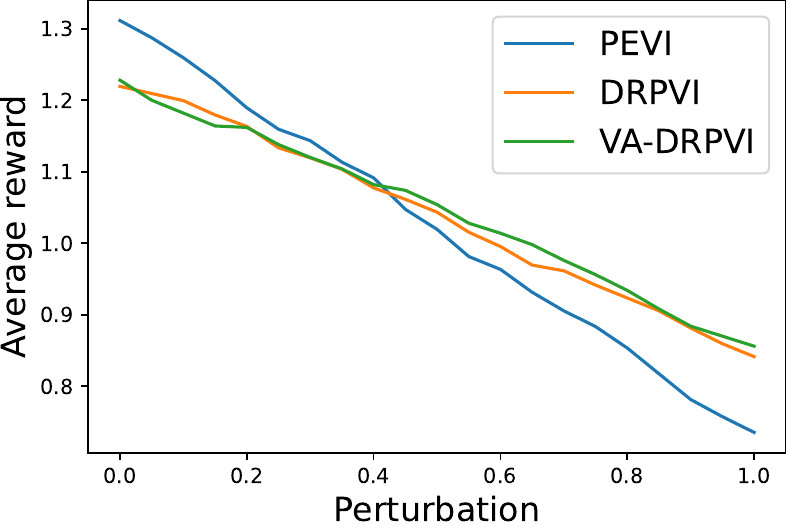}}\\
      \subfigure[$\Vert\xi\Vert_1 = 0.2$, $\rho_{1,4}=0.3$]{\includegraphics[scale=0.33]{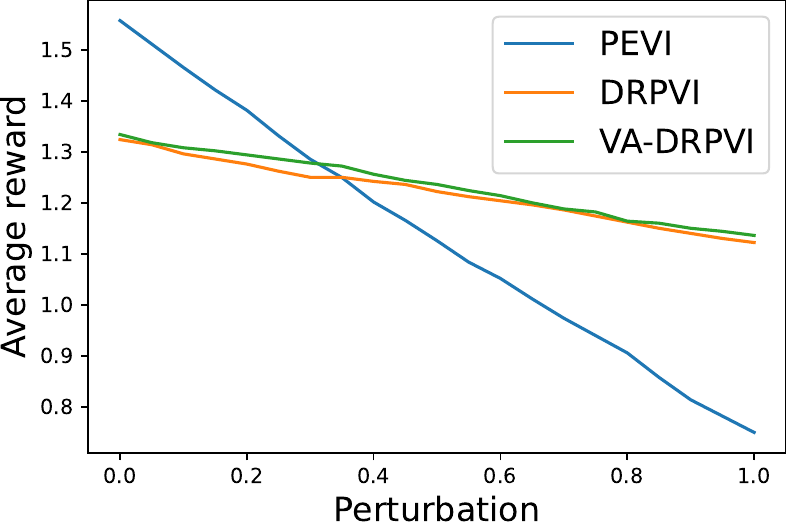}
      \label{fig:simulated_MDP_xi01_rho04}}
      \subfigure[$\Vert\xi\Vert_1 = 0.2$, $\rho_{1,4}=0.4$]{\includegraphics[scale=0.33]{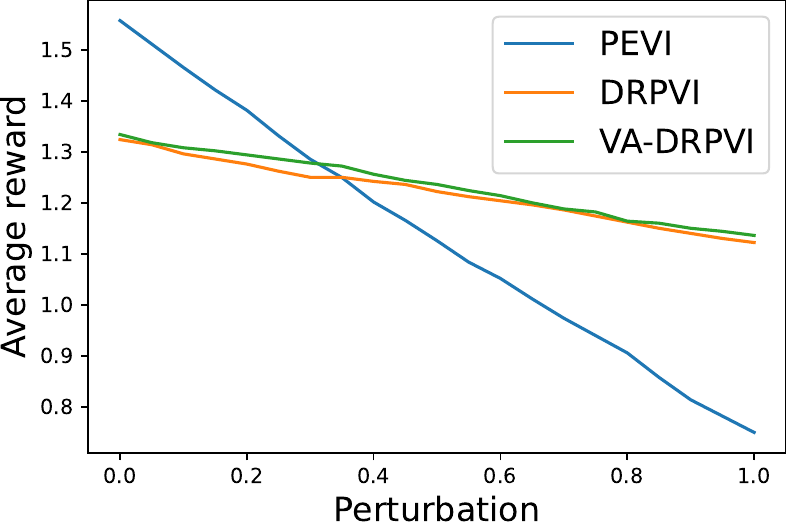}}
      \subfigure[$\Vert\xi\Vert_1 = 0.2$, $\rho_{1,4}=0.5$]{\includegraphics[scale=0.33]{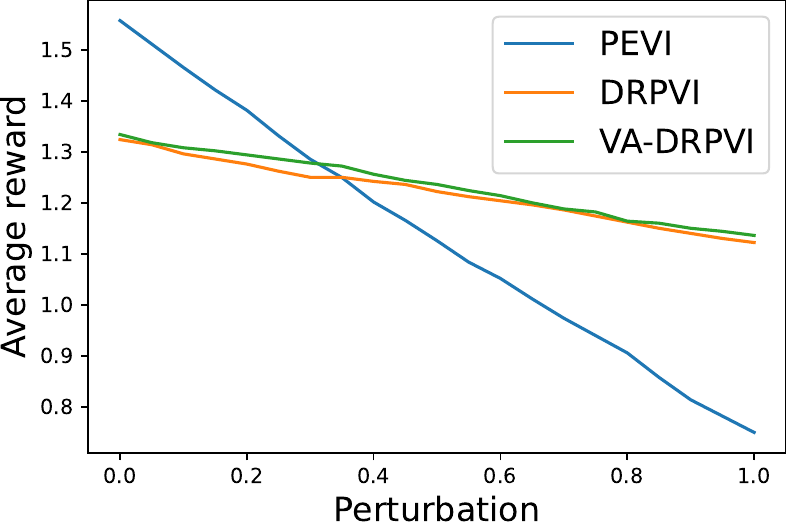}
      \label{fig:simulated_MDP_xi03_rho04}}\\
      \subfigure[$\Vert\xi\Vert_1 = 0.3$, $\rho_{1,4}=0.3$]{\includegraphics[scale=0.33]{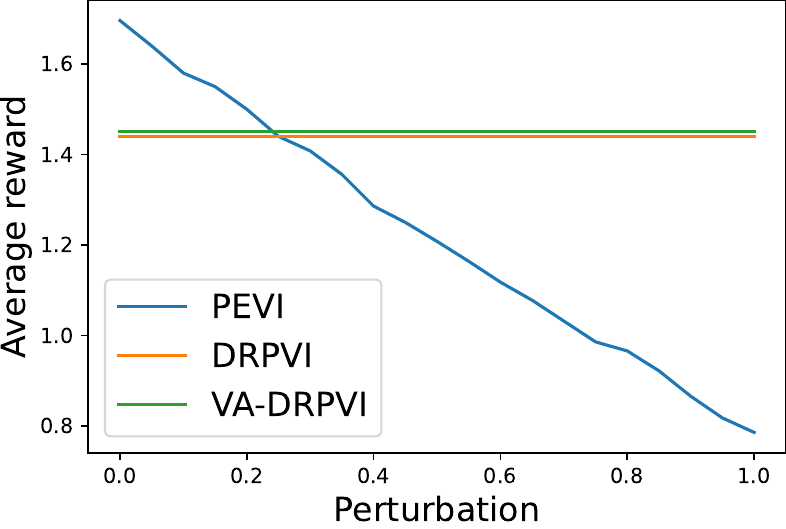}
      \label{fig:simulated_MDP_xi01_rho03}}
      \subfigure[$\Vert\xi\Vert_1 = 0.3$, $\rho_{1,4}=0.4$]{\includegraphics[scale=0.34]{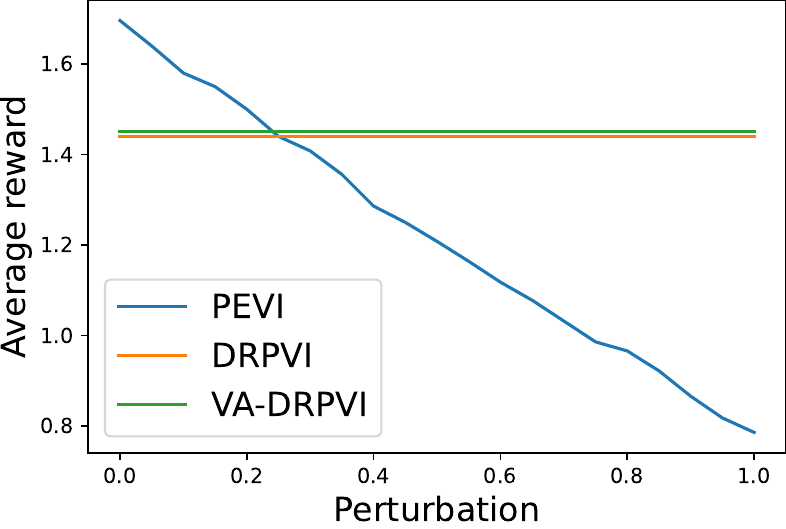}
      \label{fig:simulated_MDP_xi02_rho03}}
      \subfigure[$\Vert\xi\Vert_1 = 0.3$, $\rho_{1,4}=0.5$]{\includegraphics[scale=0.34]{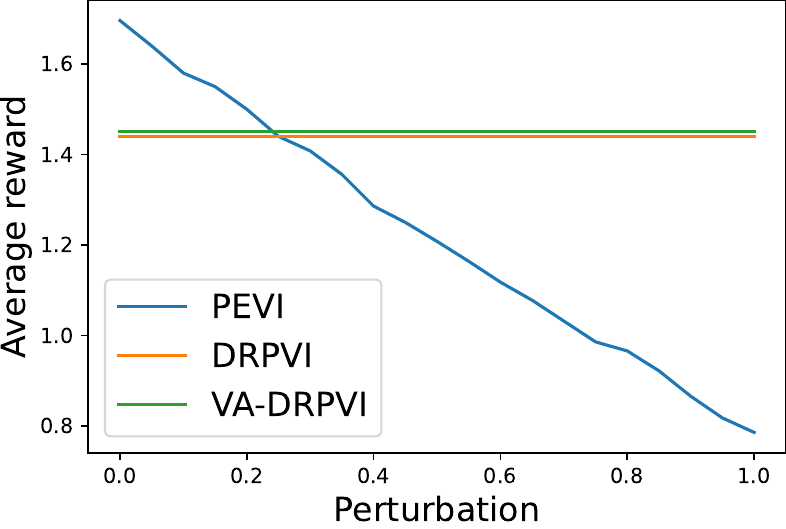}
      \label{fig:simulated_MDP_xi03_rho03}}
    \caption{Simulation results under different source domains. The $x$-axis represents the perturbation level corresponding to different target environments. $\rho_{1,4}$ is the input uncertainty level for our \algnameVA\ algorithm. $\Vert\xi\Vert_1$ is the hyperparameter of the linear DRMDP environment.}
    \label{fig:simulation-results-app}
\end{figure*}

\Cref{fig:simulation-results-app} shows the performances of the learned policies of three algorithms. We conclude that both of our proposed algorithms are robust to environmental perturbation compared to the non-robust PEVI. Furthermore, \algnameVA slightly outperforms \algname\ in most settings. These numerical results are consistent with our theoretical findings.

\section{Proof of \Cref{th:DRPVI}}

Our analysis mainly deals with the challenges induced by the model uncertainty, $\inf_{P\in\cU^{\rho}(P^0)}$, and the need to maximally exploit the information in the offline dataset. More specifically, the proof of \Cref{th:DRPVI} mainly constitutes of two steps. 

\paragraph{Step 1: suboptimality decomposition.} We first decompose the suboptimality gap in the following lemma to connect it with the estimation error, the full proof of which can be found in \Cref{sec:Regret Decomposition for DRMDP}.
\begin{lemma}[Suboptimality Decomposition for DRMDP]
    \label{lemma:Regret Decomposition for DRMDP}
    If the following holds  %
    \begin{align}
    \label{eq:Regret Decomposition-upper bound of estimation error}
        \Big|\inf_{P_h(\cdot|s,a)\in\cU_h^{\rho}(s,a;\bmu_{h,i}^0)}[\PP_h\widehat{V}_{h+1}^{\rho}](s,a) - \bphi(s,a)\hat{\bnu}_h^{\rho}\Big|\leq \Gamma_h(s,a), \forall (s,a,h)\in\cS\times\cA\times[H],
    \end{align}
    then we have $\textstyle\text{SubOpt}(\hat{\pi}, s, \rho)\leq 2\sup_{P\in\cU^{\rho}(P^0)}\sum_{h=1}^H\EE^{\pi^{\star}, P}\big[\Gamma_h(s_h,a_h)|s_1=s\big].$
    \end{lemma}
The main challenge in deriving \Cref{lemma:Regret Decomposition for DRMDP} lies in the dependency of the robust Bellman equation \eqref{eq:robust bellman equation} on the nominal kernel $P^0$, which is not linear and does not even have an explicit form. 
It should be noted that the term $\big|\inf_{P_h(\cdot|s,a)\in\cU_h^{\rho}(s,a;\bmu_{h,i}^0)}[\PP_h\widehat{V}_{h+1}^{\rho}](s,a) - \bphi(s,a)^\top\hat{\bnu}_h^{\rho}\big|$ in condition \eqref{eq:Regret Decomposition-upper bound of estimation error} stands for the estimation error of the estimated robust Q-function in \eqref{eq:ridge estimator of Q-function}, which we refer to as the robust estimation uncertainty. 
\Cref{lemma:Regret Decomposition for DRMDP} shows that under the condition that the robust estimation uncertainty is bounded by $\Gamma_h(s,a)$, the suboptimality gap can be upper bounded in terms of $\Gamma_h(s,a)$. To conclude the proof, it remains to derive $\Gamma_h(s,a)$ and then substitute it back into the result in \Cref{lemma:Regret Decomposition for DRMDP}.

    \paragraph{Step 2: bounding the robust estimation uncertainty.} We now bound the robust estimation uncertainty in \Cref{lemma:Regret Decomposition for DRMDP} by the following result, the full proof of which can be found in \Cref{sec: proof of Estimation Error Bound}.
    \begin{lemma}[Robust Estimation Uncertainty Bound]
    \label{lemma:Estimation Error Bound}
    For any sufficiently large sample size $K$ satisfying $K>\max\{512\log(2dH^2/\delta)/\kappa^2, 20449d^2H^2/\kappa\}$,
    and any fixed $\delta\in (0,1)$, if we set $\lambda=1$ in \Cref{alg:DR-PVI}, then with probability at least $1-\delta$, for all $(s,a,h)\in\cS\times\cA\times[H]$, we have
    \begin{align}
    \label{eq:Estimation Error Bound}
        \Big|\inf_{P_h(\cdot|s,a)\in\cU_h^{\rho}(s,a;\bmu_{h,i}^0)}[\PP_h\widehat{V}_{h+1}^{\rho}](s,a)-\bphi(s,a)^\top\hat{\bnu}_h^{\rho}\Big|\leq \Gamma_h(s,a),
    \end{align}
    where $\Gamma_h(s,a) = 4\sqrt{d}H\sqrt{\iota}\sum_{i=1}^d\|\phi_i(s,a)\mathbf{1}_i\|_{\bLambda_h^{-1}}$ and $\iota = \log(2dH^2K/\delta)$.
    \end{lemma}
    $\Gamma_h(s,a)$ provides an explicit bound for the robust estimation uncertainty, which also serves as the penalty term in Line \ref{algline:DRPVI-penalty} of \Cref{alg:DR-PVI}. The main challenge of deriving \Cref{lemma:Estimation Error Bound} lies in inferring the worst-case behavior using information merely from the nominal environment. Our idea is to first transform the robust estimation uncertainty to the estimation uncertainty of ridge regressions \eqref{eq:ridge regression} on the nominal model $P^0$, where the samples are collected and statistical control is available. We then adopt a reference-advantage decomposition technique, which is new in the linear DRMDP literature, to further decompose the estimation uncertainty on the nominal model into the reference uncertainty and the advantage uncertainty. The remaining proof is to bound the reference uncertainty and advantage uncertainty respectively using concentration and union bound arguments under an induction framework to address the temporal dependency. We highlight that all these arguments are specifically designed for the unique problem of DRMDP, which is novel and nontrivial. %

\section{Proof of the Suboptimality Upper Bounds}
\label{sec:Proof of the Main results}
In this section, we prove the main results in \Cref{corollary:DRPVI}, \Cref{remark: discussion on assumptions}, \Cref{th:VA-DRPVI}, and \Cref{corollary:VA-DRPVI}, which give out the instance-dependent upper bounds of the proposed algorithms. Before the proof, we introduce some useful notations. For any function $f:\cS\rightarrow [0, H-1]$, define 
\begin{align}\label{def:widehat_Inf_Q}
\widehat{\inf_{P_h(\cdot|s,a)\in\cU_h^{\rho}(s,a;\bmu_{h,i}^0)}}[\PP_hf](s,a):= \bphi(s,a)^{\top}\hat{\bnu}_h^{\rho}(f),
\end{align}
where for each $i\in[d]$, we have 
\begin{align*}
 \hat{\nu}^{\rho}_{h,i}(f)&=\max_{\alpha\in[0,H]}\Big\{\hat{\EE}^{\mu^{0}_{h,i}}[f(s)]_{\alpha}-\rho(\alpha-\min_{s'\in\cS}[f(s')]_{\alpha})\Big\},\\
 \hat{\EE}^{\mu^{0}_{h,i}}[f(s)]_{\alpha}&=\Big[\bLambda_h^{-1}\sum_{\tau=1}^K\bphi_h^{\tau}[f(s_{h+1}^{\tau})]_{\alpha}\Big]_i.
\end{align*}

\subsection{Proof of \Cref{corollary:DRPVI}}
\label{sec:proof of corollary:DRPVI}
The proof of \Cref{corollary:DRPVI} is straightforward given our result in \Cref{th:DRPVI}. 

\begin{proof}
Define  $\tilde{\bLambda}_h=\EE^{\pi^b,P^0}[\bphi(s_h,a_h)\bphi(s_h,a_h)^{\top}], \forall h\in[H]$.
By \Cref{assumption:feature coverage}, we have $\tilde{\bLambda}_h\succeq\kappa\cdot \mathbf{I}$. 
We further bound \eqref{eq:def-instance-dependent-uncertainty-function} as follows,
\begin{align}
    &\sup_{P\in\cU^{\rho}(P^0)}\sum_{h=1}^H\EE^{\pi^\star, P}\bigg[\sum_{i=1}^d\Vert \phi_i(s_h,a_h)\mathbf{1}_i\Vert_{\bLambda_h^{-1}} \Big|s_1=s\bigg]\notag\\
    & \leq \sup_{P\in\cU^{\rho}(P^0)}\frac{2}{\sqrt{K}}\EE^{\pi^\star, P}\bigg[\sum_{h=1}^H\sum_{i=1}^d\Vert \phi_i(s_h,a_h)\mathbf{1}_i\Vert_{\tilde{\bLambda}_h^{-1}} \Big|s_1=s\bigg] \label{eq:proof of cor-matix_concentration} \\
    & = \sup_{P\in\cU^{\rho}(P^0)}\frac{2}{\sqrt{K}}\EE^{\pi^\star, P}\bigg[\sum_{h=1}^H\sum_{i=1}^d\phi_i(s,a)\sqrt{\mathbf{1}_i^\top\tilde{\bLambda}^{-1}_h\mathbf{1}_i}\Big|s_1=s \bigg]\notag\\
    & \leq \sup_{P\in\cU^{\rho}(P^0)}\frac{2}{\sqrt{K}}\EE^{\pi^\star, P}\bigg[\sum_{h=1}^H\sum_{i=1}^d\phi_i(s,a)\sqrt{\lambda_{\max}(\tilde{\bLambda}^{-1}_h)}\Big|s_1=s \bigg]\label{eq:diagnal value is less than the largest eigenvalue}\\
    & = \sup_{P\in\cU^{\rho}(P^0)}\frac{2}{\sqrt{K}}\EE^{\pi^\star, P}\bigg[\sum_{h=1}^H\sum_{i=1}^d\phi_i(s,a)\sqrt{\frac{1}{\lambda_{\min}(\tilde{\bLambda}_h)}}\Big|s_1=s \bigg]\notag\\
    &\leq \sup_{P\in\cU^{\rho}(P^0)}\frac{2}{\sqrt{K}}\EE^{\pi^\star, P}\bigg[\sum_{h=1}^H\sqrt{\frac{1}{\kappa}} \bigg]\label{eq:lower bound on the minimal value}\\
    & = \frac{2H}{\sqrt{K\cdot\kappa}},\notag
\end{align}
where \eqref{eq:proof of cor-matix_concentration} is due to \Cref{lemma:matrix-normalized concentration}, 
\eqref{eq:diagnal value is less than the largest eigenvalue} is due to the fact that for any matrix $\bA$, $\lambda_{\min}\leq \bA_{ii} \leq \lambda_{\max}$, where $\bA_{ii}$ is the $i$-th diagonal element of $\bA$. \eqref{eq:diagnal value is less than the largest eigenvalue} holds due to \Cref{assumption:feature coverage} and the fact that $\sum_{i=1}^d\phi_i(s,a)=1$.
We conclude the proof by invoking  
\Cref{th:DRPVI}.
\end{proof}

\subsection{Proof of \Cref{th:VA-DRPVI}}
The proof idea is similar to that of \Cref{th:DRPVI}, except that we additionally analyze the variance estimation and apply the Bernstein-type self-normalized concentration inequality to bound the reference uncertainty, which is the dominant term. We start from analyzing the estimation error of conditional variances in the following lemma.
\begin{lemma}
\label{lemma:variance estimation}
    Under \Cref{assumption:linear MDP,assumption:feature coverage}, when $K\geq \tilde{O}(H^4/\kappa^2)$,
    then with probability at least $1-\delta$, for all $(s,a,h)\in\cS\times\cA\times[H]$ and any fixed $\alpha$, we have
    \begin{align*}
        \big[\VV_h[V_{h+1}^{\star, \rho}]_{\alpha}\big](s,a)-\tilde{O}\Big(\frac{dH^3}{\sqrt{K\kappa}}\Big)\leq \widehat{\sigma}^2_h(s,a;\alpha)\leq \big[\VV_h[V_{h+1}^{\star, \rho}]_{\alpha}\big](s,a).
    \end{align*}
\end{lemma}

The following lemma bounds the estimation error by reference-advantage decomposition.
\begin{lemma}[Variance-Aware Reference-Advantage Decomposition]
\label{lemma:Variance-Aware Reference-Advantage Decomposition}
    There exist $\{\alpha_i\}_{i\in[d]}$, where $\alpha_i\in[0,H], \forall i\in [d]$, such that 
    \begin{align*}
        &\Big|\inf_{P_h(\cdot|s,a)\in\cU_h^{\rho}(s,a;\bmu_{h,i}^0)}[\PP_h\widehat{V}_{h+1}^{\rho}](s,a)-\widehat{\inf_{P_h(\cdot|s,a)\in\cU_h^{\rho}(s,a;\bmu_{h,i}^0)}}[\PP_h\widehat{V}_{h+1}^{\rho}](s,a)\Big|   \\
        &\leq \underbrace{\lambda\sum_{i=1}^d\|\phi_i(s,a)\mathbf{1}_i\|_{\bSigma_h^{-1}(\alpha_i)}\|\EE^{\bmu_{h}^0}[V^{\star,\rho}_{h+1}(s)]_{\alpha_i}\|_{\bSigma_h^{-1}(\alpha_i)}}_\text{i}\\
        &\quad +\underbrace{\sum_{i=1}^d\|\phi_i(s,a)\mathbf{1}_i\|_{\bSigma_h^{-1}(\alpha_i)}\Big\|\sum_{\tau=1}^K\frac{\bphi_h^{\tau}\eta_h^{\tau}([V_{h+1}^{\star, \rho}]_{\alpha_i})}{\widehat{\sigma}^2_h(s_h^{\tau}, a_h^{\tau};\alpha_i)}\Big\|_{\bSigma_h^{-1}(\alpha_i)}}_\text{ii} \\
        &\quad + \underbrace{\lambda\sum_{i=1}^d\|\phi_i(s,a)\mathbf{1}_i\|_{\bSigma_h^{-1}(\alpha_i)}\Big\|\EE^{\bmu_{h}^0}\big[[\widehat{V}^{\rho}_{h+1}(s)]_{\alpha_i} - [V^{\star,\rho}_{h+1}(s)]_{\alpha_i} \big]\Big\|_{\bSigma_h^{-1}(\alpha_i)}}_\text{iii}\\
&\quad + \underbrace{\sum_{i=1}^d\|\phi_i(s,a)\mathbf{1}_i\|_{\bSigma_h^{-1}(\alpha_i)}\Big\|\sum_{\tau=1}^K\frac{\bphi_h^{\tau}\eta_h^{\tau}([\widehat{V}^{\rho}_{h+1}(s)]_{\alpha_i} - [V^{\star,\rho}_{h+1}(s)]_{\alpha_i})}{\widehat{\sigma}^2_h(s_h^{\tau}, a_h^{\tau};\alpha_i)}
\Big\|_{\bSigma_h^{-1}(\alpha_i)}}_\text{iv},
    \end{align*}
    where $\eta_h^{\tau}([f]_{\alpha_i}) = \big(\big[\PP_h^0[f]_{\alpha_i}\big](s_h^{\tau},a_h^{\tau})-[f(s_{h+1}^{\tau})]_{\alpha_i} \big)$, for any function $f:\cS\rightarrow [0,H-1]$.
\end{lemma}

Now we are ready to prove \Cref{th:VA-DRPVI}
\begin{proof}[Proof of \Cref{th:VA-DRPVI}]
To prove this theorem, we bound the estimation error by $\Gamma_h(s,a)$, then invoke \Cref{lemma:Regret Decomposition for DRMDP}
to get the result. First, we bound terms i-iv in \Cref{lemma:Variance-Aware Reference-Advantage Decomposition} to deduce $\Gamma_h(s,a)$ at each step $h\in[H]$, respectively.

\paragraph{Bound i and iii:} We set $\lambda = 1/H^2$ to ensure that for all $(s,a,h)\in\cS\times\cA\times[H]$, we have
\begin{align}
\label{eq:Variance-Aware Reference-Advantage Decomposition-bound i&iii}
    \text{i} + \text{iii} \leq \sqrt{\lambda}\sqrt{d}H\sum_{i=1}^d\|\phi_i(s,a)\mathbf{1}_i\|_{\bSigma_h^{-1}(\alpha_i)}= \sqrt{d}\sum_{i=1}^d\|\phi_i(s,a)\mathbf{1}_i\|_{\bSigma_h^{-1}(\alpha_i)}.
\end{align}
\paragraph{Bound ii:}
For all $(s,a,\alpha)\in\cS\times\cA\times[0,H]$, by definition we have $\widehat{\sigma}_h(s,a;\alpha)\geq 1$. Thus, for all $(h,\tau,i)\in[H]\times[K]\times[d]$, we have 
$|\eta_h^{\tau}([V_{h+1}^{\star, \rho}]_{\alpha_i})/\widehat{\sigma}_h(s_h^{\tau},a_h^{\tau},\alpha_i)|\leq H$. 
Note that $V_{H+1}^{\star,\rho}$ is independent of $\cD$, we can directly apply Bernstein-type self-normalized concentration inequality \Cref{lemma:Bernstein Concentration of Self-Normalized Processes} and a union bound to obtain the upper bound. In concrete, we define the filtration $\cF_{\tau-1, h} = \sigma(\{(s_h^{j},a_h^{j})\}_{j=1}^{
\tau}\cup\{s_{h+1}^{j}\}_{j=1}^{\tau-1})$. 
Since $V_{h+1}^{\star, \rho}$ and $\widehat{\sigma}_h(s,a;\alpha)$ are independent of $\cD$, thus $\eta_h^{\tau}([V_{h+1}^{\star, \rho}]_{\alpha_i})/\widehat{\sigma}_h(s_h^{\tau},a_h^{\tau},\alpha_i)$ is mean-zero conditioned on the filtration $\cF_{\tau-1, h}$. Further, we have
\begin{align}
    \EE\Big[\Big(\frac{\eta_h^{\tau}([V_{h+1}^{\star, \rho}]_{\alpha_i})}{\widehat{\sigma}_h(s_h^{\tau},a_h^{\tau};\alpha_i)}\Big)^2 \Big|\cF_{\tau-1, h}\Big]&=\frac{[\Var[V_{h+1}^{\star,\rho}]_{\alpha_i}](s_h^{\tau}, a_h^{\tau})}{\widehat{\sigma}^2_h(s_h^{\tau}, a_h^{\tau};\alpha_i)}\label{eq:Variance-Aware Reference-Advantage Decomposition-sigma_hat is independent of D}\\
    & \leq \frac{[\VV[V_{h+1}^{\star,\rho}]_{\alpha_i}](s_h^{\tau}, a_h^{\tau})}{\widehat{\sigma}_h^2(s_h^{\tau}, a_h^{\tau};\alpha_i)}\notag\\
    & =\frac{[\VV[V_{h+1}^{\star,\rho}]_{\alpha_i}](s_h^{\tau}, a_h^{\tau}) - \tilde{O}(dH^3/\sqrt{K\kappa}) }{\widehat{\sigma}_h^2(s_h^{\tau}, a_h^{\tau};\alpha_i)} + \frac{\tilde{O}(dH^3/\sqrt{K\kappa})}{\widehat{\sigma}_h^2(s_h^{\tau}, a_h^{\tau};\alpha_i)}\notag\\
    &\leq 1+\frac{\tilde{O}(dH^3/\sqrt{K\kappa})}{\widehat{\sigma}_h^2(s_h^{\tau}, a_h^{\tau};\alpha_i) - \tilde{O}(dH^3/\sqrt{K\kappa})}\label{eq:Variance-Aware Reference-Advantage Decomposition-invoke variance estimation lemma}\\
    &\leq 1+2\tilde{O}\Big(\frac{dH^3}{\sqrt{K\kappa}} \Big)\label{eq:Variance-Aware Reference-Advantage Decomposition-set K large enough},
\end{align}
where \eqref{eq:Variance-Aware Reference-Advantage Decomposition-sigma_hat is independent of D} holds by the fact that $\widehat{\sigma}_h^2(\cdot,\cdot;\cdot)$ is independent of $\cD$ and $(s_h^\tau,a_h^\tau)$ is $\cF_{\tau-1, h}$ measurable. \eqref{eq:Variance-Aware Reference-Advantage Decomposition-invoke variance estimation lemma} holds by \Cref{lemma:variance estimation}, and \eqref{eq:Variance-Aware Reference-Advantage Decomposition-set K large enough} holds by setting $K\geq \tilde{\Omega}(d^2H^6/\kappa)$ such that $\widehat{\sigma}_h^2(s_h^{\tau}, a_h^{\tau};\alpha_i) - \tilde{O}(dH^3/\sqrt{K\kappa})\geq 1-\tilde{O}(dH^3/\sqrt{K\kappa})\geq 1/2$.
Further, by \eqref{eq:Variance-Aware Reference-Advantage Decomposition-set K large enough}, our choice of $K$ also ensures that $\EE\big[\big(\eta_h^{\tau}([V_{h+1}^{\star, \rho}]_{\alpha_i})\big)^2 |\cF_{\tau-1, h}\big]=O(1)$. Then by \Cref{lemma:Bernstein Concentration of Self-Normalized Processes}, we have
\begin{align*}
    \Bigg\|\sum_{\tau=1}^K\frac{\bphi_h^{\tau}\eta_h^{\tau}([V_{h+1}^{\star, \rho}]_{\alpha_i})}{\widehat{\sigma}^2_h(s_h^{\tau}, a_h^{\tau};\alpha_i)}\Bigg\|_{\bSigma_h^{-1}(\alpha_i)}\leq \tilde{O}(\sqrt{d}).
\end{align*}
This implies 
\begin{align}
\label{eq:Variance-Aware Reference-Advantage Decomposition-bound ii}
    \text{ii}\leq \tilde{O}(\sqrt{d})\sum_{i=1}^d\|\phi_i(s,a)\mathbf{1}_i\|_{\bSigma_h^{-1}(\alpha_i)}.
\end{align}

\paragraph{Bound iv:} Following the same induction analysis procedure, we have $\|[\widehat{V}_{h+1}^{\rho}]_{\alpha_i}-[V_{h+1}^{\star,\rho}]_{\alpha_i}\|\leq \tilde{O}(\sqrt{d}H^2/\sqrt{K\kappa})$. Then, using standard $\epsilon$-covering number argument and \Cref{lemma:Hoeffding Concentration of Self-Normalized Processes}, we have
\begin{align}
\label{eq:Variance-Aware Reference-Advantage Decomposition-bound iv}
    \text{iv} \leq \tilde{O}\Big(\frac{d^{3/2}H^2}{\sqrt{K\kappa}}\Big)\sum_{i=1}^d\|\phi_i(s,a)\mathbf{1}_i\|_{\bSigma_h^{-1}(\alpha_i)}.
\end{align}
To make it non-dominant, we require $K\geq \tilde{\Omega}(d^2H^4/\kappa)$.
By \Cref{lemma:variance estimation}, for any $\alpha\in[0,H]$, we have
\begin{align*}
    \widehat{\sigma}_h^2(s_h^{\tau}, a_h^{\tau};\alpha)\leq [\VV_h[V_{h+1}^{\star,\rho}]_{\alpha}](s_h^{\tau},a_h^{\tau})\leq [\VV_hV_{h+1}^{\star,\rho}](s_h^{\tau},a_h^{\tau}),
\end{align*}
this implies that %
\begin{align*}   \bigg(\sum_{\tau=1}^K\frac{\bphi_h^{\tau}\bphi_h^{\tau\top}}{\widehat{\sigma}^2_h(s_h^{\tau},a_h^{\tau};\alpha_i)}+\lambda \mathbf{I}  \bigg)^{-1}\preceq \bigg(\sum_{\tau=1}^K\frac{\bphi_h^{\tau}\bphi_h^{\tau\top}}{[\VV_hV_{h+1}^{\star,\rho}](s_h^{\tau},a_h^{\tau})}+\lambda \mathbf{I}  \bigg)^{-1} := \bSigma_h^{\star-1}.
\end{align*}
Combining \eqref{eq:Variance-Aware Reference-Advantage Decomposition-bound i&iii}, \eqref{eq:Variance-Aware Reference-Advantage Decomposition-bound ii} and \eqref{eq:Variance-Aware Reference-Advantage Decomposition-bound iv}, we have 
\begin{align*}
    &\Big|\inf_{P_h(\cdot|s,a)\in\cU_h^{\rho}(s,a;\bmu_{h,i}^0)}[\PP_h\widehat{V}_{h+1}^{\rho}](s,a)-\widehat{\inf_{P_h(\cdot|s,a)\in\cU_h^{\rho}(s,a;\bmu_{h,i}^0)}}[\PP_h\widehat{V}_{h+1}^{\rho}](s,a)\Big|\\
    &\leq \tilde{O}(\sqrt{d})\sum_{i=1}^d\|\phi_i(s,a)\mathbf{1}_i\|_{\bSigma_h^{\star-1}}.
\end{align*}
Define $\Gamma_h(s,a)=\tilde{O}(\sqrt{d})\sum_{i=1}^d\|\phi_i(s,a)\mathbf{1}_i\|_{\bSigma_h^{\star-1}}$, we concludes the proof by invoking \Cref{lemma:Regret Decomposition for DRMDP}.
\end{proof}

\subsection{Proof of \Cref{corollary:VA-DRPVI}}
\label{sec:proof of corollary-VA}
In this section, we prove \Cref{corollary:VA-DRPVI}. We start with an interesting phenomenon, we call `range shrinkage', stated in the following lemma.
\begin{lemma}[Range Shrinkage]
\label{lemma:Range Shrinkage}
    For any $(\rho,\pi, h) \in (0,1] \times \Pi \times [H]$, we have 
    \begin{align}
    \label{eq:Range Shrinkage}
        \max_{s\in\cS}V_h^{\pi, \rho}(s) - \min_{s\in\cS}V_h^{\pi, \rho}(s) \leq  \frac{1-(1-\rho)^{H-h+1}}{\rho}.
    \end{align}
\end{lemma}

\begin{proof}[Proof of \Cref{corollary:VA-DRPVI}] By the fact that the variance of a random variable can be upper bounded by the square of its range and \Cref{lemma:Range Shrinkage}, for all $(s,a,h)\in\cS\times\cA\times[H]$, we have 
    \begin{align*}
        [\VV V_{h+1}^{\star}](s,a)\leq \Big(\frac{1-(1-\rho)^{H-h+1}}{\rho}\Big)^2\leq \Big(\frac{1-(1-\rho)^{H}}{\rho}\Big)^2.
    \end{align*}
    Then we have 
    \begin{align*}
        \sum_{\tau=1}^K\frac{\bphi_h^{\tau}\bphi_h^{\tau\top}}{[\VV_hV_{h+1}^{\star}](s_h^{\tau},a_h^{\tau})}+\frac{1}{H^2}\mathbf{I} \succeq \sum_{\tau=1}^K\frac{\bphi_h^{\tau}\bphi_h^{\tau\top}}{(\frac{1-(1-\rho)^H}{\rho})^2}+\frac{1}{H^2}\mathbf{I}.
    \end{align*}
    Thus we have
    \begin{align*}
        \bSigma^{\star-1}_h=\Big(\sum_{\tau=1}^K\frac{\bphi_h^{\tau}\bphi_h^{\tau\top}}{[\VV_hV_{h+1}^{\star}](s_h^{\tau},a_h^{\tau})}+\frac{1}{H^2}\mathbf{I}\Big)^{-1}\preceq \Big(\frac{1-(1-\rho)^{H}}{\rho}\Big)^2 \Big(\sum_{\tau=1}^K\bphi_h^{\tau}\bphi_h^{\tau\top}+\frac{1}{H^2} \mathbf{I} \Big)^{-1}.
    \end{align*}
    By \Cref{th:VA-DRPVI}, we have
    \begin{align*}
        \text{SubOpt}(\hat{\pi}, s, \rho) &\leq \tilde{O}(\sqrt{d}) \cdot \sup_{P\in\cU^{\rho}(P^0)}\sum_{h=1}^H\EE^{\pi^{\star}, P}\Big[\sum_{i=1}^d\Vert\phi_i(s_h,a_h)\mathbf{1}_i \Vert_{\bSigma_h^{\star-1}}\big|s_1=s\Big]\\
        &\leq \tilde{O}(\sqrt{d})\cdot \frac{1-(1-\rho)^H}{\rho}\sup_{P\in\cU^{\rho}(P^0)}\sum_{h=1}^H\EE^{\pi^{\star}, P}\Big[\sum_{i=1}^d\Vert\phi_i(s_h,a_h)\mathbf{1}_i \Vert_{\bLambda_h^{-1}}\big|s_1=s\Big].
    \end{align*}
    This concludes the proof.
\end{proof}

\section{Proof of the Information-Theoretic Lower Bound}
\label{sec:proof of lower bound}
In this section, we prove the information-theoretic lower bound. We first introduce the construction of hard instances in \Cref{sec:Construction of Hard Instances}, then we prove \Cref{th:lower bound} in  \Cref{sec:proof of lower bound-subsection}, and prove \Cref{corollary:lower bound} in \Cref{sec:proof of lower bound-corollary}.
\subsection{Construction of Hard Instances}
\label{sec:Construction of Hard Instances}
We design a family of $d$-rectangular linear DRMDPs parameterized by a Boolean vector $\bxi=\{\bxi_h\}_{h\in[H]}$, where $\bxi_h\in\{-1,1\}^d$.
For a given $\bxi$ and uncertainty level $\rho\in(0,3/4]$, the corresponding $d$-rectangular linear DRMDP $M_{\bxi}^\rho$ has the following structure. 
The state space $\cS=\{x_1, x_2\}$ and the action space $\cA=\{0,1\}^d$. The initial state distribution $\mu_0$ is defined as 
\begin{align*}
    \mu_0(x_1)=\frac{d+1}{d+2}\quad \text{and} \quad \mu_0(x_2)=\frac{1}{d+2}.
\end{align*}
The feature mapping $\bphi:\cS\times\cA\rightarrow\RR^{d+2}$ is defined as
\begin{align*}
    &\bphi(x_1,a)^{\top}=\Big(\frac{a_1}{d}, \frac{a_2}{d}, \cdots, \frac{a_d}{d}, 1-\sum_{i=1}^d\frac{a_i}{d}, 0 \Big)\\
    &\bphi(x_2,a)^{\top}=\big(0, 0, \cdots, 0, 0, 1\big),
\end{align*}
which satisfies $\phi_i(s,a)\geq 0$ and $\sum_{i=1}^d\phi_i(s,a)=1$.
The factor distributions $\{\bmu_h\}_{h\in[H]}$ are defined as 
\begin{align*}
    \bmu_h^\top=\big(\underbrace{\delta_{x_1}, \delta_{x_1}, \cdots, \delta_{x1}, \delta_{x_1}}_\text{$d+1$}, \delta_{x_2}\big), \forall h\in[H],
\end{align*}
so the transition is homogeneous and does not depend on action but only on state. The reward parameters $\{\btheta_h\}_{h\in[H]}$ are defined as 
\begin{align*}
    \btheta_h^\top=\delta\cdot\Big(\frac{\xi_{h1}+1}{2}, \frac{\xi_{h2}+1}{2}, \cdots, \frac{\xi_{hd}+1}{2}, \frac{1}{2}, 0 \Big), \forall h\in[H],
\end{align*}
where $\delta$ is a parameter to control the differences among instances, which is to be determined later.
The reward $r_h$ is generated from the normal distribution $r_h\sim\cN(r_h(s_h,a_h),1)$, where $r_h(s,a)=\bphi(s,a)^\top\btheta_h$. Note that
\begin{align*}
    r_h(x_1, a)=\bphi(x_1,a)^\top\btheta_h = \frac{\delta}{2d}\big(\la\bxi_h, a\ra+d\big)\geq 0\quad \text{and} \quad r_h(x_2, a)=\bphi(x_2,a)^\top\btheta_h=0,~ \forall a\in\cA,
\end{align*}
which means that $x_2$ is a worst state in terms of the mean reward. Thus, the worst case transition kernel should have the highest possible transition probability to $x_2$. This construction is pivotal in achieving a concise expression of robust value function.
Further, we only consider model uncertainty in the first step. By the fact that $x_2$ is the worse state, we know the worst case factor distribution for the first step is 
\begin{align*}
    \check{\bmu}_1^\top=\big((1-\rho)\delta_{x_1}+\rho\delta_{x_2}, (1-\rho)\delta_{x_1}+\rho\delta_{x_2}, \cdots, (1-\rho)\delta_{x_1}+\rho\delta_{x_2}, (1-\rho)\delta_{x_1}+\rho\delta_{x_2}, \delta_{x_2} \big).
\end{align*}
We illustrate the designed $d$-rectangular linear DRMDP $M_{\bxi}^{\rho}$ in \Cref{fig:illustration of the hard instance-source} and \Cref{fig:illustration of the hard instance-target}.

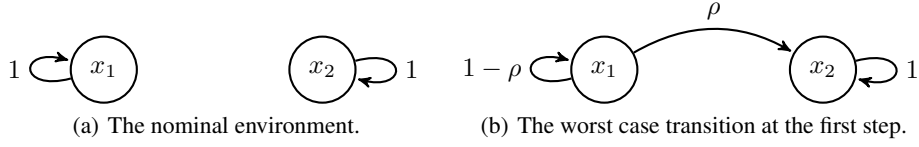
\begin{figure}
\centering
\subfigure[The nominal environment.]{
\begin{tikzpicture}[->,>=stealth',shorten >=1pt,auto,node distance=2cm,thick]
  \node[state] (x1) {$x_1$};
  \node[state, right=of x1] (x2) {$x_2$};
  \path[->] (x1) edge[loop left] node[left] {$1$} (x1);
  \path[->] (x2) edge[loop right] node[right] {$1$} (x2);
\end{tikzpicture}
\label{fig:illustration of the hard instance-source}
}
\subfigure[The worst case transition at the first step.]{
\begin{tikzpicture}[->,>=stealth',shorten >=1pt,auto,node distance=2cm,thick]
  \node[state] (x1) {$x_1$};
  \node[state, right=of x1] (x2) {$x_2$};
  \path[->] (x1) edge[loop left] node[left] {$1-\rho$} (x1);
  \path[->] (x1) edge[bend left] node[above] {$\rho$} (x2);
  \path[->] (x2) edge[loop right] node[right] {$1$} (x2);
\end{tikzpicture}
\label{fig:illustration of the hard instance-target}
}
\caption{The nominal environment and the worst case environment. The value on each arrow represents the transition probability.
The MDP has two states and 
$H$ steps. For the nominal environment, both  $x_1$ and $x_2$ are absorbing states, which means that the state will always stay at the initial state in the nominal environment.
The worst case environment on the right is obtained by perturbing the transition probability at the first step of the nominal environment, with others remain the same.}
\end{figure}
Finally, we design the procedure for collecting the offline dataset.
 We assume the $K$ trajectories are collected by a behavior policy $\pi^b=\{\pi_h^b\}_{h\in[H]}$ defined as 
 \[\pi_h^b\sim \text{Unif}\big(\{\be_1, \cdots, \be_d, \mathbf{0}\}\big),\forall h\in [H],
 \]
 where $\{\be_i\}_{i\in[d]}$ are the canonical basis vectors in $\RR^d$. The initial state is generated according to $\mu_0$. It is straightforward to check that the constructed hard instances satisfy \Cref{assumption:feature coverage}. We denote the offline dataset as $\cD$.

\subsection{Proof of \Cref{th:lower bound}}
\label{sec:proof of lower bound-subsection}
With this family of hard instances, we are ready to prove the information-theoretic lower bound. First, we define some notations. For any $\bxi \in \{-1,1\}^{dH}$, let $\QQ_{\bxi}$ denote the distribution of dataset $\cD$ collected from the MDP $M_{\bxi}$. Denote the family of parameters as $\Omega = \{-1,1\}^{dH}$ and the family of hard instances as $\cM=\{M_{\bxi}:\bxi\in\Omega\}$.
\begin{proof}[Proof of \Cref{th:lower bound}]
The proof constitutes three steps. In the first step, we lower bound the minimax suboptimality gap by testing error in the following \Cref{lemma:reduction to testing}, the full proof of which can be found in \Cref{sec:proof of reduction to testing}.
\begin{lemma}[Reduction to testing] 
\label{lemma:reduction to testing}
For the given family of $d$-rectangular linear DRMDPs, we have
\begin{align}
\label{eq:reduction to testing}
     \inf_{\hat{\pi}}\sup_{M\in\cM}\text{SubOpt}(\hat{\pi}, x_1, \rho)\geq (1-\rho)\cdot\frac{\delta dH}{8d}\cdot\min_{\substack{\bxi, \bxi'\in\Omega\\D_H(\bxi,\bxi')=1}}\inf_{\psi}\Big[\QQ_{\bxi}(\psi(\cD)\neq\bxi) + \QQ_{\bxi'}(\psi(\cD)\neq\bxi') \Big],
\end{align}
where for fixed indices $\bxi$ and  $\bxi'$, $\psi$ is any test function taking value in $\{\bxi, \bxi'\}$.
\end{lemma}
In the second step, we lower bound the testing error on the right hand side of \eqref{eq:reduction to testing} in the following \Cref{lemma:Lower bound on testing error}, the full proof of which can be found in \Cref{sec:proof of lower bound on testing error}. 
\begin{lemma}[Lower bound on testing error]
\label{lemma:Lower bound on testing error} For the given family of $d$-rectangular linear DRMDPs, let $\delta = d^{3/2}/\sqrt{2K}$, then we have
\begin{align*}
    \min_{\substack{\bxi, \bxi'\\D_H(\bxi,\bxi')=1}}\inf_{\psi}\Big[\QQ_{\bxi}(\psi(\cD)\neq\bxi) + \QQ_{\bxi'}(\psi(\cD)\neq\bxi') \Big] \geq \frac{1}{2}.
\end{align*}
\end{lemma}
By \Cref{lemma:reduction to testing} and \Cref{lemma:Lower bound on testing error}, we have
    \begin{align}
    \label{eq:minimax lower bound in dHK}
        \inf_{\hat{\pi}}\sup_{M\in\cM}\text{SubOpt}(\hat{\pi}, x_1, \rho)\geq \frac{d^{3/2}H}{128\sqrt{K}}.
    \end{align}
In the last step, we upper bound the uncertainty function $\Phi(\bSigma_h^{\star},s)$ in the following \Cref{lemma:upper bound on the summation}, the full proof of which can be found in \Cref{sec:proof of upper bound on the summation}.
\begin{lemma}
\label{lemma:upper bound on the summation}For all $M_{\bxi}\in\cM$,
when $K\geq \tilde{O}(d^4)$, then with probability at least $1-\delta$, we have
\begin{align*}
    \sup_{P\in\cU^{\rho}(P^0)}\sum_{h=1}^H\EE^{\pi^{\star}, P}\Big[\sum_{i=1}^d\Vert\phi_i(s_h,a_h)\mathbf{1}_i \Vert_{\bSigma_h^{\star-1}}\big|s_1=x_1\Big]\leq \frac{4d^{3/2}H}{\sqrt{K}}.
\end{align*}
\end{lemma}
By \Cref{lemma:upper bound on the summation} and \eqref{eq:minimax lower bound in dHK}, we know that with probability at least $1-\delta$, there exist a universal constant $c$, such that
    \begin{align*}
         \inf_{\hat{\pi}}\sup_{M\in\cM}\text{SubOpt}(\hat{\pi}, x_1, \rho)\geq c\cdot \sup_{P\in\cU^{\rho}(P^0)}\sum_{h=1}^H\EE^{\pi^{\star}, P}\Big[\sum_{i=1}^d\Vert\phi_i(s_h,a_h)\mathbf{1}_i \Vert_{\bSigma_h^{\star-1}}|s_1=x_1\Big].
    \end{align*}
    This concludes the proof.
\end{proof}

\subsection{Proof of \Cref{corollary:lower bound}}
\label{sec:proof of lower bound-corollary}
\begin{proof}
    The result in \Cref{corollary:lower bound} directly follows from the fact shown in \eqref{eq:upper bound on the summation-two matrix equal}: for the constructed hard instances, we have $\bSigma_h^{\star}=\bLambda_h$. Thus, we complete the proof by directly substituting $\bSigma_h^{\star}$ in the result of \Cref{th:lower bound} by $\bLambda_h$.
\end{proof}

\section{Proof of Technical Lemmas}
\subsection{Proof of \Cref{lemma:Regret Decomposition for DRMDP}}
\label{sec:Regret Decomposition for DRMDP}
\begin{proof}
    First, we decompose $\text{SubOpt}(\hat{\pi}, s, \rho)$ as follows
    \begin{align*}
        \text{SubOpt}(\hat{\pi}, s, \rho) = \underbrace{ V_1^{\pi^{\star},\rho}(s)-\widehat{V}_1^{\rho}(s)}_\text{I}
       +\underbrace{\widehat{V}_1^{\rho}(s)-V_1^{\hat{\pi},\rho}(s)}_\text{II},
    \end{align*}
    then we bound term I and term II, respectively.
    \paragraph{Bounding term I:} Note that 
    \begin{align}
        V_h^{\pi^{\star}, \rho}(s) - \widehat{V}_h^{\rho}(s) &= Q_h^{\pi^{\star}, \rho}(s, \pi_h^{\star}(s))-\widehat{Q}_h^{\rho}(s, \hat{\pi}_h(s))\notag\\
        &=Q_h^{\pi^{\star}, \rho}(s, \pi_h^{\star}(s))-\widehat{Q}_h^{\rho}(s, \pi_h^{\star}(s))+\widehat{Q}_h^{\rho}(s, \pi_h^{\star}(s))-\widehat{Q}_h^{\rho}(s, \hat{\pi}_h(s))\notag\\
        &\leq Q_h^{\pi^{\star}, \rho}(s, \pi_h^{\star}(s))-\widehat{Q}_h^{\rho}(s, \pi_h^{\star}(s))\label{eq:Regret Decomposition-I-greedy policy}.
    \end{align}
    Here \eqref{eq:Regret Decomposition-I-greedy policy} holds by the fact that $\hat{\pi}_h(s)$ is the greedy policy corresponding to $\widehat{Q}_h^{\rho}(s,a)$, which leads to $\widehat{Q}_h^{\rho}(s, \pi_h^{\star}(s))-\widehat{Q}_h^{\rho}(s, \hat{\pi}_h(s))\leq 0$. Further, by the robust Bellman equation \eqref{eq:robust bellman equation}, we have
    \begin{align*}
       &Q_h^{\pi^{\star}, \rho}(s, \pi_h^{\star}(s))-\widehat{Q}_h^{\rho}(s, \pi_h^{\star}(s)) \\
       &=r_h(s, \pi_h^{\star}(s)) + \inf_{P_h(\cdot|s,a)\in\cU_h^{\rho}(s,a;\bmu_{h,i}^0)}[\PP_hV_{h+1}^{\pi^{\star},\rho}](s,\pi^{\star}_h(s)) -\widehat{Q}_h^{\rho}(s, \pi_h^{\star}(s)) \\
       &=r_h(s, \pi_h^{\star}(s)) + \inf_{P_h(\cdot|s,a)\in\cU_h^{\rho}(s,a;\bmu_{h,i}^0)}[\PP_hV_{h+1}^{\pi^{\star},\rho}](s,\pi^{\star}_h(s)) -r_h(s, \pi_h^{\star}(s))\\
       &\quad -\inf_{P_h(\cdot|s,a)\in\cU_h^{\rho}(s,a;\bmu_{h,i}^0)}[\PP_h\widehat{V}_{h+1}^{\rho}](s,\pi^{\star}_h(s)) + r_h(s, \pi_h^{\star}(s))+\inf_{P_h(\cdot|s,a)\in\cU_h^{\rho}(s,a;\bmu_{h,i}^0)}[\PP_h\widehat{V}_{h+1}^{\rho}](s,\pi^{\star}_h(s)) \\
       &\quad -\widehat{Q}_h^{\rho}(s, \pi_h^{\star}(s)).
    \end{align*}
    To proceed, we define the robust Bellman update error as follows
    \begin{align*}
        \zeta_h^{\rho}(s,a)=r_h(s,a)+\inf_{P_h(\cdot|s,a)\in\cU_h^{\rho}(s,a;\bmu_{h,i}^0)}[\PP_h\widehat{V}_{h+1}^{\rho}](s,a)-\widehat{Q}_h^{\rho}(s,a),
    \end{align*}
    and denote the worst case transition kernel with respect to the estimated robust value function as $\widehat{P}=\{\widehat{P}_h\}_{h\in[H]}$, where $\widehat{P}_h(\cdot|s,a)=\arginf_{P_h(\cdot|s,a)\in\cU_h^{\rho}(s,a;\bmu_{h,i}^0)}[\PP_h\widehat{V}_{h+1}^{\rho}](s,a), \forall (s,a)\in\cS\times\cA$. Then we have
    \begin{align}
        &Q_h^{\pi^{\star}, \rho}(s, \pi_h^{\star}(s))-\widehat{Q}_h^{\rho}(s, \pi_h^{\star}(s)) \notag\\
        &= \inf_{P_h(\cdot|s,a)\in\cU_h^{\rho}(s,a;\bmu_{h,i}^0)}[\PP_hV_{h+1}^{\pi^{\star},\rho}](s,\pi^{\star}_h(s))- \inf_{P_h(\cdot|s,a)\in\cU_h^{\rho}(s,a;\bmu_{h,i}^0)}[\PP_h\widehat{V}_{h+1}^{\rho}](s,\pi^{\star}_h(s))+ \zeta_h^{\rho}(s,\pi_h^{\star}(s))\notag\\
        &\leq \big[\widehat{\PP}_h(V_{h+1}^{\pi^{\star},\rho}-\widehat{V}_{h+1}^{\rho})\big](s,\pi_h^{\star}(s))+\zeta_h^{\rho}(s,\pi_h^{\star}(s)).\label{eq:Regret Decomposition-I-upper bound of difference of Q}
    \end{align}
    Combining \eqref{eq:Regret Decomposition-I-greedy policy} and \eqref{eq:Regret Decomposition-I-upper bound of difference of Q}, we have for any $h\in[H]$, 
    \begin{align}
    \label{eq:Regret Decomposition-I-recursive formular}
        V_h^{\pi^{\star}, \rho}(s) - \widehat{V}_h^{\rho}(s) \leq \big[\widehat{\PP}_h(V_{h+1}^{\pi^{\star},\rho}-\widehat{V}_{h+1}^{\rho})\big](s,\pi_h^{\star}(s))+\zeta_h^{\rho}(s,\pi_h^{\star}(s)).
    \end{align}
    Recursively applying \eqref{eq:Regret Decomposition-I-recursive formular}, we have 
    \begin{align}
    \label{eq:Regret Decomposition-I-upper bound}
        V_1^{\pi^{\star}, \rho}(s) - \widehat{V}_1^{\rho}(s)\leq \sum_{h=1}^H\EE^{\pi^{\star}, \widehat{P}}\big[\zeta_h^\rho(s_h,a_h)|s_1=s\big].
    \end{align}
    \paragraph{Bounding term II:} Note that $\widehat{V}_h^{\rho}(s) - V_h^{\hat{\pi},\rho}(s) = \widehat{Q}_h^{\rho}(s, \hat{\pi}_h(s)) - Q_h^{\hat{\pi}, \rho}(s, \hat{\pi}_h(s))$, by the robust Bellman equation \eqref{eq:robust bellman equation}, we have
    \begin{align*}
        &\widehat{V}_h^{\rho}(s) - V_h^{\hat{\pi},\rho}(s)\\
        &=\widehat{Q}_h^{\rho}(s, \hat{\pi}_h(s))-r_h(s, \hat{\pi}_h(s))-\inf_{P_h(\cdot|s,a)\in\cU_h^{\rho}(s,a;\bmu_{h,i}^0)}[\PP_h\widehat{V}_{h+1}^{\rho}](s,\hat{\pi}_h(s)) \notag\\
        &\quad +\inf_{P_h(\cdot|s,a)\in\cU_h^{\rho}(s,a;\bmu_{h,i}^0)}[\PP_h\widehat{V}_{h+1}^{\rho}](s,\hat{\pi}_h(s)) - \inf_{P_h(\cdot|s,a)\in\cU_h^{\rho}(s,a;\bmu_{h,i}^0)}[\PP_hV_{h+1}^{\hat{\pi},\rho}](s,\hat{\pi}_h(s))\\
        &=-\zeta_h^{\rho}(s,\hat{\pi}_h(s)) + \inf_{P_h(\cdot|s,a)\in\cU_h^{\rho}(s,a;\bmu_{h,i}^0)}[\PP_h\widehat{V}_{h+1}^{\rho}](s,\hat{\pi}_h(s)) - \inf_{P_h(\cdot|s,a)\in\cU_h^{\rho}(s,a;\bmu_{h,i}^0)}[\PP_hV_{h+1}^{\hat{\pi},\rho}](s,\hat{\pi}_h(s)).
    \end{align*}
    To proceed, we denote the worst case transition kernel with respect to the robust value function of $\hat{\pi}$ as $P^{\hat{\pi}}=\{P^{\hat{\pi}}_h\}_{h\in[H]}$, where $P_h^{\hat{\pi}}(\cdot|s,a)=\arginf_{P_h(\cdot|s,a)\in\cU_h^{\rho}(s,a;\bmu_{h,i}^0)}[\PP_hV_{h+1}^{\hat{\pi},\rho}](s,a)$, then we have
    \begin{align}
    \label{eq:Regret Decomposition-II-recursive formular}
        \widehat{V}_h^{\rho}(s) - V_h^{\hat{\pi},\rho}(s)\leq -\zeta_h^{\rho}(s,\hat{\pi}_h(s)) + \big[\PP_h^{\hat{\pi}}(\widehat{V}_{h+1}^{\rho} - V_{h+1}^{\hat{\pi}, \rho}) \big](s,\hat{\pi}_h(s)).
    \end{align}
    Applying \eqref{eq:Regret Decomposition-II-recursive formular} recursively, we have
    \begin{align}
    \label{eq:Regret Decomposition-II-upper bound}
        \widehat{V}_1^{\rho}(s)-V_1^{\hat{\pi}, \rho}(s)\leq \sum_{h=1}^H\EE^{\hat{\pi}, P^{\hat{\pi}}}\big[-\zeta_h^{\rho}(s_h,a_h)|s_1=s\big].
    \end{align}
    Now it remains to bound the robust Bellman error $\zeta_h^{\rho}(\cdot,\cdot)$. In particular, we aim to show that for all $(s,a,h)\in\cS\times\cA\times[H]$, 
    \[
    0\leq \zeta_h^{\rho}(s,a)\leq 2 \Gamma_h(s,a).
    \]
    Note that $\zeta_h^{\rho}(s,a)=r_h(s,a)+\inf_{P_h(\cdot|s,a)\in\cU_h^{\rho}(s,a;\bmu_{h,i}^0)}[\PP_h\widehat{V}_{h+1}^{\rho}](s,a)-\widehat{Q}_h^{\rho}(s,a)$. Recall the definition of $\widehat{Q}_h^{\rho}(s,a)$ in \Cref{alg:DR-PVI} and the notation in \eqref{def:widehat_Inf_Q}, and we have %
    \begin{align*}
        \zeta_h^{\rho}(s,a)&=r_h(s,a)+\inf_{P_h(\cdot|s,a)\in\cU_h^{\rho}(s,a;\bmu_{h,i}^0)}[\PP_h\widehat{V}_{h+1}^{\rho}](s,a)\\
        &\qquad - \max\Big\{r_h(s,a)+ \widehat{\inf_{P_h(\cdot|s,a)\in\cU_h^{\rho}(s,a;\bmu_{h,i}^0)}}[\PP_h\widehat{V}_{h+1}^{\rho}](s,a)-\Gamma_h(s,a),0\Big\}.
    \end{align*}
If $r_h(s,a)+ \widehat{\inf}_{P_h(\cdot|s,a)\in\cU_h^{\rho}(s,a;\bmu_{h,i}^0)}[\PP_h\widehat{V}_{h+1}^{\rho}](s,a)-\Gamma_h(s,a)\leq 0$, then $\zeta_h^{\rho}(s,a)=r_h(s,a)+\inf_{P_h(\cdot|s,a)\in\cU_h^{\rho}(s,a;\bmu_{h,i}^0)}[\PP_h\widehat{V}_{h+1}^{\rho}](s,a)\geq 0$. If $r_h(s,a)+ \widehat{\inf}_{P_h(\cdot|s,a)\in\cU_h^{\rho}(s,a;\bmu_{h,i}^0)}[\PP_h\widehat{V}_{h+1}^{\rho}](s,a)-\Gamma_h(s,a)> 0$, then we have $\zeta_h^{\rho}(s,a)=r_h(s,a)+\inf_{P_h(\cdot|s,a)\in\cU_h^{\rho}(s,a;\bmu_{h,i}^0)}[\PP_h\widehat{V}_{h+1}^{\rho}](s,a) - r_h(s,a)- \widehat{\inf}_{P_h(\cdot|s,a)\in\cU_h^{\rho}(s,a;\bmu_{h,i}^0)}[\PP_h\widehat{V}_{h+1}^{\rho}](s,a)+\Gamma_h(s,a)\geq -\Gamma_h(s,a)+\Gamma_h(s,a)=0$, where we used the condition in \eqref{eq:Regret Decomposition-upper bound of estimation error}. In conclusion, we have $ \zeta_h^{\rho}(s,a)\geq 0$.

    On the other hand, we always have%
    \begin{align*}
        \zeta_h^{\rho}(s,a) &\leq \inf_{P_h(\cdot|s,a)\in\cU_h^{\rho}(s,a;\bmu_{h,i}^0)}[\PP_h\widehat{V}_{h+1}^{\rho}](s,a) -\widehat{\inf_{P_h(\cdot|s,a)\in\cU_h^{\rho}(s,a;\bmu_{h,i}^0)}}[\PP_h\widehat{V}_{h+1}^{\rho}](s,a) + \Gamma_h(s,a)\\
        &\leq 2\Gamma_h(s,a).
    \end{align*}
    Thus, for all $(s,a,h)\in\cS\times\cA\times[H]$, we have 
    \begin{align}
    \label{eq:Regret Decomposition-bound of Bellman update error}
        0\leq \zeta_h^{\rho}(s,a)\leq 2\Gamma_h(s,a).
    \end{align}
    Combining \eqref{eq:Regret Decomposition-I-upper bound}, \eqref{eq:Regret Decomposition-II-upper bound} and \eqref{eq:Regret Decomposition-bound of Bellman update error}, we have
    \begin{align}
        \text{SubOpt}(\hat{\pi},s, \rho) &\leq \sum_{h=1}^H\EE^{\pi^{\star}, \widehat{P}}\big[\zeta_h(s_h,a_h)|s_1=s\big] + \sum_{h=1}^H\EE^{\hat{\pi}, P^{\hat{\pi}}}\big[-\zeta_h^{\rho}(s_h,a_h)|s_1=s\big]\notag\\
        &\leq \sum_{h=1}^H\EE^{\pi^{\star}, \widehat{P}}\big[\zeta^{\rho}_h(s_h,a_h)|s_1=s\big] \notag\\
        & \leq \sup_{P\in\cU^{\rho}(P^0)}\sum_{h=1}^H\EE^{\pi^{\star}, P}\big[\zeta^{\rho}_h(s_h,a_h)|s_1=s\big]\notag\\
        &\leq 2\sup_{P\in\cU^{\rho}(P^0)}\sum_{h=1}^H\EE^{\pi^{\star}, P}\big[\Gamma_h(s_h,a_h)|s_1=s\big].\notag
    \end{align}
    This concludes the proof.
\end{proof}

\subsection{Proof of \Cref{lemma:Estimation Error Bound}}
\label{sec: proof of Estimation Error Bound}
In this section, we prove \Cref{lemma:Estimation Error Bound}. Before the proof, we first present several auxiliary lemmas.

\begin{lemma}[Reference-Advantage Decomposition]
\label{lemma:Reference-Advantage Decomposition}
    There exist real values $\{\alpha_i\}_{i\in[d]}$, where $\alpha_i\in[0,H], \forall i\in [d]$, such that 
    \begin{align*}
        &\Big|\inf_{P_h(\cdot|s,a)\in\cU_h^{\rho}(s,a;\bmu_{h,i}^0)}[\PP_h\widehat{V}_{h+1}^{\rho}](s,a)-\widehat{\inf_{P_h(\cdot|s,a)\in\cU_h^{\rho}(s,a;\bmu_{h,i}^0)}}[\PP_h\widehat{V}_{h+1}^{\rho}](s,a)\Big|\\
        &\leq \underbrace{\lambda\sum_{i=1}^d\|\phi_i(s,a)\mathbf{1}_i\|_{\bLambda_h^{-1}}\|\EE^{\bmu_{h}^0}[V^{\star,\rho}_{h+1}(s)]_{\alpha_i}\|_{\bLambda_h^{-1}}}_\text{i}+\underbrace{\sum_{i=1}^d\|\phi_i(s,a)\mathbf{1}_i\|_{\bLambda_h^{-1}}\Big\|\sum_{\tau=1}^K\bphi_h^{\tau}\eta_h^{\tau}([V_{h+1}^{\star, \rho}]_{\alpha_i})\Big\|_{\bLambda_h^{-1}}}_\text{ii}  \\
        &\quad+ \underbrace{\lambda\sum_{i=1}^d\|\phi_i(s,a)\mathbf{1}_i\|_{\bLambda_h^{-1}}\Big\|\EE^{\bmu_{h}^0}\big[[\widehat{V}^{\rho}_{h+1}(s)]_{\alpha_i} - [V^{\star,\rho}_{h+1}(s)]_{\alpha_i} \big]\Big\|_{\bLambda_h^{-1}}}_\text{iii} \\
        &\quad+\underbrace{\sum_{i=1}^d\|\phi_i(s,a)\mathbf{1}_i\|_{\bLambda_h^{-1}}\Big\|\sum_{\tau=1}^K\bphi_h^{\tau}\eta_h^{\tau}([\widehat{V}^{\rho}_{h+1}(s)]_{\alpha_i} - [V^{\star,\rho}_{h+1}(s)]_{\alpha_i})\Big\|_{\bLambda_h^{-1}}}_\text{iv},
    \end{align*}
    where $\eta_h^{\tau}([f]_{\alpha_i}) = \big(\big[\PP_h^0[f]_{\alpha_i}\big](s_h^{\tau},a_h^{\tau})-[f(s_{h+1}^{\tau})]_{\alpha_i} \big)$, for any function $f:\cS\rightarrow [0,H-1]$.
\end{lemma}

\begin{lemma}[Bound of Weights]
\label{lemma:weight bound}
For any $h \in [H]$, denote the weight $\bw_h^{\rho} = \btheta_h+\hat{\bnu}_h^{\rho}$ in \Cref{alg:DR-PVI}, then $\bw_h^{\rho}$ satisfies
    \begin{align*}
        \Vert \bw_h^{\rho} \Vert_2 \leq 2H\sqrt{dK/\lambda}.
    \end{align*}
\end{lemma}

\begin{lemma}\cite[Lemma B.2]{jin2021pessimism}
\label{lemma:self-normalize concentration for offline setting}
Let $f: \mathcal{S} \rightarrow[0, R-1]$ be any fixed function. For any $\delta \in(0,1)$, we have
\begin{align*}
    \mathbb{P}\Big(\Big\|\sum_{\tau =1}^K \phi_h^\tau \cdot \eta_h^\tau(f)\Big\|_{\bLambda_h^{-1}}^2 \geq R^2\Big(2 \log \Big(\frac{1}{\delta}\Big)+d \log \Big(1+\frac{K}{\lambda}\Big)\Big)\Big) \leq \delta.
\end{align*}
\end{lemma}
\begin{lemma}[Covering number of function class $\cV_h$]
\label{lemma:covering number of the function class V}
    For any $h\in[H]$, let $\cV_h$ denote a class of functions mapping from $\cS$ to $\RR$ with the following parametric form 
    \begin{align*}
        \cV_h(s)=\max_{a\in\cA}\Big\{\bphi(s,a)^{\top}\btheta-\beta\sum_{i=1}^d\sqrt{\phi_i(s,a)\mathbf{1}_i^{\top}\bSigma_h^{-1}\phi_i(s,a)\mathbf{1}_i}\Big\}_{[0, H-h+1]},
    \end{align*}
    where the parameters $(\btheta, \beta, \bSigma_h)$ satisfy $\Vert \btheta\Vert \leq L$, $\beta \in [0, B]$, $\lambda_{\min}(\bSigma_h)\geq \lambda$. Assume $\Vert \bphi(s,a)\Vert\leq 1$ for all (s,a) pairs, and let $\cN_h({\epsilon})$ be the $\epsilon$-covering number of $\cV$ with respect to the distance $\dist(V_1, V_2)=\sup_x|V_1(x)-V_2(x)|$. Then 
    \begin{align*}
        \log\cN_h({\epsilon}) \leq d\log(1+4L/\epsilon) + d^2\log\big[1+8d^{1/2}B^2/(\lambda\epsilon^2) \big].
    \end{align*}
\end{lemma}

\begin{lemma}\cite[Covering number of an interval]{vershynin2018high}
\label{lemma:Covering number of an interval}
     Denote the $\epsilon$-covering number of the closed interval $[a,b]$ for some real number $b>a$ with respect to the distance metric $d(\alpha_1, \alpha_2)=|\alpha_1-\alpha_2|$ as $\cN_{\epsilon}([a,b])$. Then we have $\cN_{\epsilon}([a,b])\leq 3(b-a)/\epsilon$.
\end{lemma}

\begin{proof}[Proof of \Cref{lemma:Estimation Error Bound}]
To prove this lemma, we bound terms i-iv in \Cref{lemma:Reference-Advantage Decomposition} at each step $h\in[H]$, respectively. 
To deal with the temporal dependency, we follow the induction procedure proposed in \cite{xiong2023nearly} and make essential adjustments to adapt to the robust setting.

\paragraph{The base case.}
We start from the last step $H$.  By the fact that any robust value function is upper bounded by $H$, then with $\lambda=1$, for all $(s,a)\in\cS\times\cA$, we have
\begin{align}
\label{eq:Estimation Error Bound-bound i&iii of step H}
    \text{i} + \text{iii} \leq 2H \sum_{i=1}^d\|\phi_i(s,a)\mathbf{1}_i\|_{\bLambda_h^{-1}}.
\end{align}
Next, we bound term ii. 
  Note that $V_{H+1}^{\star,\rho}$ is independent of $\cD$, we can directly apply Hoeffding-type self-normalized concentration inequality \Cref{lemma:Hoeffding Concentration of Self-Normalized Processes} and a union bound to obtain the upper bound. In concrete, we define the filtration $\cF_{\tau-1, h} = \sigma(\{(s_h^{j},a_h^{j})\}_{j=1}^{
\tau}\cup\{s_{h+1}^{j}\}_{j=1}^{\tau-1})$. Since $V_{H+1}^{\star,\rho}$ is independent of $\cD$ and is upper bounded by $H$, thus we have $\eta_H^{\tau}([V_{H+1}^{\star,\rho}]_{\alpha_i})|\cF_{\tau-1, H}$ is mean zero, i.e., $\EE[\eta_H^{\tau}([V_{H+1}^{\star,\rho}]_{\alpha_i})|\cF_{\tau-1, H}]=0$ and $H$-subGaussian. By \Cref{lemma:Hoeffding Concentration of Self-Normalized Processes}, for any fixed index $i\in[d]$, with probability at least $1-\delta/2dH^2$, we have
\begin{align*}
    \Big\|\sum_{\tau=1}^K\bphi_H^{\tau}\eta_H^{\tau}([V_{H+1}^{\star, \rho}]_{{\alpha}_i})\Big\|_{\bLambda_h^{-1}}^2 &\leq 2H^2\log\Big(\frac{2dH^2\det(\bLambda_h)^{1/2}}{\delta\det(\lambda \mathbf{I})^{1/2}}\Big).
\end{align*}
By the proof of Lemma B.2 in \cite{jin2021pessimism}, we know $\det(\bLambda_h)\leq(\lambda+K)^d$. Thus, we have 
\begin{align*}
    \Big\|\sum_{\tau=1}^K\bphi_H^{\tau}\eta_H^{\tau}([V_{H+1}^{\star, \rho}]_{{\alpha}_i})\Big\|_{\bLambda_h^{-1}}^2\leq 2H^2\Big(\frac{d}{2}\log\frac{\lambda + K}{\lambda} + \log\frac{2dH^2}{\delta} \Big)
    \leq dH^2\log\frac{2dH^2K}{\delta}.
\end{align*}
Then by a union bound over $i\in[d]$, with probability at least $1-\delta/2H^2$, we have
\begin{align}
\label{eq:Estimation Error Bound-bound ii of step H}
    \text{ii} \leq \sqrt{d}H\sqrt{\iota}  \sum_{i=1}^d\|\phi_i(s,a)\mathbf{1}_i\|_{\bLambda_h^{-1}},
\end{align}
where $\iota = \log(2dH^2K/\delta) \geq 1$.
As for the term iv,
by construction we have $V_{H+1}^{\star, \rho}=\widehat{V}_{H+1}^{\rho}=0$ with probability 1.
Thus, we trivially have 
\begin{align}
\label{eq:Estimation Error Bound-bound iv of step H}
    \text{iv} \leq \sqrt{d}H\sqrt{\iota}  \sum_{i=1}^d\|\phi_i(s,a)\mathbf{1}_i\|_{\bLambda_h^{-1}}.
\end{align}
Combining \eqref{eq:Estimation Error Bound-bound i&iii of step H}, \eqref{eq:Estimation Error Bound-bound ii of step H} and \eqref{eq:Estimation Error Bound-bound iv of step H}, for all $(s,a)\in\cS\times\cA$, with probability at least $1-\delta/2H^2$, we have
\begin{align}
    &\Big|\inf_{P_H(\cdot|s,a)\in\cU_H^{\rho}(s,a;\bmu_{H,i}^0)}[\PP_H\widehat{V}_{H+1}^{\rho}](s,a)-\widehat{\inf_{P_H(\cdot|s,a)\in\cU_H^{\rho}(s,a;\bmu_{H,i}^0)}} [\PP_H\widehat{V}_{H+1}^{\rho}](s,a)\Big|\notag\\
    &\leq 4\sqrt{d}H\sqrt{\iota}  \sum_{i=1}^d\|\phi_i(s,a)\mathbf{1}_i\|_{\bLambda_h^{-1}}.\label{eq:Estimation Error Bound-estimation error bound of step H}
\end{align}
Thus, we define $\Gamma_H(s,a) := 4\sqrt{d}H\sqrt{\iota}  \sum_{i=1}^d\|\phi_i(s,a)\mathbf{1}_i\|_{\bLambda_h^{-1}}$. By the definition of $\widehat{Q}_H^{\rho}(s,a)$ in \Cref{alg:DR-PVI}, we have
\begin{align*}
    \widehat{Q}_H^{\rho}(s,a) = \big\{r_H(s,a) - \Gamma_H(s,a)\big\}_{[0,1]} \leq r_H(s,a) = Q_H^{\star, \rho}(s,a),
\end{align*}
which implies that a pessimistic estimation is achieved at step $H$, i.e., $V_H^{\star, \rho}(s) \geq \widehat{V}^{\rho}_H(s), \forall s\in \cS$. 
Next, we study $V_H^{\star, \rho}(s)-\widehat{V}^{\rho}_H(s)$.
The intuition is that given the estimation error bound in \eqref{eq:Estimation Error Bound-estimation error bound of step H}, with sufficient data, the difference between $V_H^{\star, \rho}(s)$ and $\widehat{V}^{\rho}_H(s)$ should be small. 
Specifically, we have
\begin{align}
    V_H^{\star, \rho}(s) - \widehat{V}_H^{\rho}(s)&=Q_H^{\star, \rho}(s, \pi_H^{\star}(s)) - \widehat{Q}_H^{\rho}(s,\pi_H^{\star}(s)) + \widehat{Q}_H^{\rho}(s,\pi_H^{\star}(s)) - \widehat{Q}_H^{\rho}(s,\hat{\pi}(s))\notag\\
    &\leq r_H(s, \pi^{\star}_H(s)) + \inf_{P_H(\cdot|s,a)\in\cU_H^{\rho}(s,a;\bmu_{H,i}^0)}[\PP_H\widehat{V}_{H+1}^{\rho}](s,a)  \notag \\
    &\qquad -r_H(s, \pi^{\star}_H(s)) - \widehat{\inf_{P_H(\cdot|s,a)\in\cU_H^{\rho}(s,a;\bmu_{H,i}^0)}} [\PP_H\widehat{V}_{H+1}^{\rho}](s,a) + \Gamma_H(s,\pi_H^{\star}(s))\label{eq:Estimation Error Bound-expanse according to def}\\
    &\leq 2\Gamma_H(s,\pi_H^{\star}(s))\notag,
\end{align}
where \eqref{eq:Estimation Error Bound-expanse according to def} holds by the robust Bellman equation \eqref{eq:robust bellman equation} and the fact that $\widehat{Q}_H^{\rho}(s,\pi_H^{\star}(s)) - \widehat{Q}_H^{\rho}(s,\hat{\pi}(s))\leq 0$. Then we bound the pessimism term $\Gamma_H(s,a)$ in terms of the sample size $K$. By \Cref{lemma:matrix-normalized concentration}, when $K\geq \max\{512\log(2dH^2/\delta)/\kappa^2, 4/\kappa\}$, with probability at least $1-\delta/2H^2$, we have
\begin{align*}
    2\Gamma_H(s,a) =8\sqrt{d}H\sqrt{\iota}\sum_{i=1}^d\|\phi_i(s,a)\mathbf{1}_i\|_{\bLambda_h^{-1}}\leq \frac{16\sqrt{d}H\sqrt{\iota}}{\sqrt{K}}\sum_{i=1}^d\phi_i(s,a)\big(\tilde{\bLambda}_H^{-1}\big)^{1/2}_{ii},
\end{align*}
where $\tilde{\bLambda}_H = \EE^{\pi^b,P^0}[\bphi(s_H,a_H)\bphi(s_H,a_H)^{\top}]$. 
Note that for any positive definite matrix $A$, we know $\lambda_{\min}(A)\leq A_{ii} \leq \lambda_{\max}(A)$. Thus, by \Cref{assumption:feature coverage}, we have
\begin{align}
\label{eq:Estimation Error Bound-bound the pessimism term by K at step H}
    2\Gamma_H(s,a)  \leq \frac{16\sqrt{d}H\cdot1\sqrt{\iota}}{\sqrt{K\kappa}}:=R_H.
\end{align}
To summarize, we define the event
\begin{align*}
    \cE_H = \big\{0\leq V_H^{\star, \rho}(s)-\widehat{V}^{\rho}_H(s)\leq R_H, \forall s\in\cS \big\}.
\end{align*}
Then by a union bound over \eqref{eq:Estimation Error Bound-estimation error bound of step H} and \eqref{eq:Estimation Error Bound-bound the pessimism term by K at step H}, we know
$\cE_H$ holds with probability at least $1-\delta_H=1-\delta/H^2$. This concludes the proof of the base case. 

\paragraph{Inductive Hypothesis.} Suppose with probability at least $1-\delta_{h+1}$, we have
\begin{align}
    &\Big|\inf_{P_{h+1}(\cdot|s,a)\in\cU_{h+1}^{\rho}(s,a;\bmu_{h+1,i}^0)}[\PP_{h+1}\widehat{V}_{h+2}^{\rho}](s,a)-\widehat{\inf_{P_{h+1}(\cdot|s,a)\in\cU_{h+1}^{\rho}(s,a;\bmu_{h+1,i}^0)}} [\PP_{h+1}\widehat{V}_{h+2}^{\rho}](s,a)\Big|\notag\\
    &\leq 4\sqrt{d}H\sqrt{\iota}  \sum_{i=1}^d\|\phi_i(s,a)\mathbf{1}_i\|_{\Lambda_{h+1}^{-1}} :=\Gamma_{h+1}(s,a),\label{eq:Estimation Error Bound-estimation error bound of step h+1}
\end{align}
and
\begin{align}
\label{eq:Estimation Error Bound-event E_h+1}
    \cE_{h+1}=\big\{0\leq V_{h+1}^{\star}(s) - \widehat{V}_{h+1}(s)\leq R_{h+1}:= \frac{16\sqrt{d}H(H-h)\sqrt{\iota}}{\sqrt{K\kappa}}, \forall s\in\cS \big\}.
\end{align}

\paragraph{Inductive Step.} Next, we establish the result for step $h$. 
First, terms i, ii and iii at step $h$ can be similarly bounded as in the base case, i.e., we have
\begin{align}
\label{eq:Estimation Error Bound-bound i&ii&iii of step h+1}
    \text{i} + \text{ii} + \text{iii} \leq 3\sqrt{d}H\sqrt{\iota}\sum_{i=1}^d\|\phi_i(s,a)\mathbf{1}_i\|_{\bLambda_h^{-1}},
\end{align}
with probability at least $1-\delta/3H^2$. It remains to bound the term iv and ensure it is non-dominating. Here, we need to deal with the temporal dependency, as $[\widehat{V}_{h+1}^{\rho}(s)]_{\alpha} - [V_{h+1}^{\star,\rho}(s)]_{\alpha}$ is correlated to $\{(s_h^{\tau}, a_h^{\tau}, s_{h+1}^{\tau})\}_{\tau=1}^K$, thus we need a uniform concentration argument. Consider the function class
\begin{align*}
    \cV_h(D,B,\lambda) = \{V_h(s;\theta, \beta, \bSigma): \cS\rightarrow [0,H]~ \text{with}~ \|\theta\|\leq D, \beta\in[0,B], \bSigma  \succeq \lambda I \},
\end{align*}
where $V_h(s;\theta, \beta, \bSigma)=\max_{a\in\cA}\{\bphi(s,a)^{\top}\theta-\beta\sum_{i=1}^d\sqrt{\phi_i(s,a)\mathbf{1}_i^{\top}\bSigma^{-1}\phi_i(s,a)\mathbf{1}_i}\}_{[0, H-h+1]}$. For simplicity, we denote $f_{\alpha_i}(s):=[\widehat{V}_{h+1}^{\rho}(s)]_{\alpha_i} - [V_{h+1}^{\star,\rho}(s)]_{\alpha_i}$, then $f_{\alpha_i}\in \cF_{h+1}(\alpha_i)$, where
\begin{align*}
    \cF_{h+1}(\alpha):=\big\{[\widehat{V}_{h+1}^{\rho}(s)]_{\alpha} - [V_{h+1}^{\star,\rho}(s)]_{\alpha}:\widehat{V}_{h+1}^{\rho}(s)\in \cV_{h+1}(D_0,B_0,\lambda) \big\}.
\end{align*}
Note that for any fixed $\alpha$, the covering number of $\cF_{h+1}(\alpha)$ is the same as that of $\cV_h(D_0,B_0,\lambda)$.
By \Cref{lemma:weight bound},  we have $D_0=H\sqrt{Kd/\lambda}$. 
By the induction assumption \eqref{eq:Estimation Error Bound-estimation error bound of step h+1}, we have $B_0=4\sqrt{d}H\sqrt{\iota}$. 
Denote the $\epsilon$-covering of the interval $[0,H]$ with respect to the distance $\dist(\alpha_1, \alpha_2)=|\alpha_1 - \alpha_2|$ as $\cN_{[0,H]}(\epsilon)$, and its $\epsilon$-covering number as $|\cN_{[0,H]}(\epsilon)|$. For each $\alpha\in[0,H]$, we can find $\alpha_{\epsilon} \in \cN_{[0,H]}(\epsilon)$ such that $|\alpha-\alpha_\epsilon|\leq\epsilon$.
For any fixed $\alpha \in [0,H]$, we denote the $\epsilon$-covering of $\cF_{h+1}(\alpha)$ with respect to the distance $\dist(f_1, f_2)=\sup_x|f_1(x)-f_2(x)|$ as $\cN_{h+1}(\epsilon)$ (short for $\cN_{h+1}(\epsilon;D,B,\lambda)$) and its $\epsilon$-covering number as $|\cN_{h+1}(\epsilon)|$. For each $f_{\alpha}\in\cF_{h+1}(\alpha)$, we can find $f_{\alpha}^{\epsilon}\in\cN_{h+1}(\epsilon)$ such that $\sup_s|f_\alpha(s)-f_\alpha^\epsilon(s)|\leq \epsilon$. It follows that  
\begin{align*}
    &\Big\|\sum_{k=1}^K \bphi_h^{\tau}\eta_h^{\tau}(f_{\alpha_i}) \Big\|^2_{\bLambda_h^{-1}} \cdot \ind\big\{\|f_{\alpha_i}\|_{\infty}\leq R_{h+1} \big\}\\
    &\leq 2\Big\|\sum_{\tau=1}^K\bphi_h^{\tau}\eta_h^{\tau}(f_{{\alpha_i}_{\epsilon}})\Big\|^2_{\bLambda_h^{-1}}\cdot \ind\big\{\|f_{{\alpha_i}_{\epsilon}}\|_{\infty}\leq R_{h+1}+\epsilon \big\} + 2\Big\|\sum_{k=1}^{K}\bphi_h^{\tau}\big(\eta_h^{\tau}(f_{\alpha_i}) -\eta_h^{\tau}(f_{{\alpha_i}_{\epsilon}})\big) \Big\|^2_{\bLambda_h^{-1}}.
\end{align*}
Note that 
\begin{align*}
    2\Big\|\sum_{k=1}^{K}\bphi_h^{\tau}\big(\eta_h^{\tau}(f_{\alpha_i}) -\eta_h^{\tau}(f_{{\alpha_i}_{\epsilon}})\big) \Big\|^2_{\bLambda_h^{-1}}\leq 2\epsilon^2\sum_{\tau, \tau'=1}^K\big|\bphi_h^{\tau}\bLambda_h^{-1}\bphi_h^{\tau'}\big| \leq 2\epsilon^2K^2/\lambda.
\end{align*}
Then we have
\begin{align*}
     &\Big\|\sum_{k=1}^K \bphi_h^{\tau}\eta_h^{\tau}(f_{\alpha_i}) \Big\|^2_{\bLambda_h^{-1}} \cdot \ind\big\{\|f_{\alpha_i}\|_{\infty}\leq R_{h+1} \big\}\\
     &\leq 4\Big\|\sum_{\tau=1}^K\bphi_h^{\tau}\eta_h^{\tau}(f^{\epsilon}_{{\alpha_i}_{\epsilon}})\Big\|^2_{\bLambda_h^{-1}}\cdot \ind\big\{\|f^{\epsilon}_{{\alpha_i}_{\epsilon}}\|_{\infty}\leq R_{h+1}+2\epsilon \big\}\\
     &\quad + 4\Big\|\sum_{k=1}^{K}\bphi_h^{\tau}\big(\eta_h^{\tau}(f_{{\alpha_i}_{\epsilon}}) -\eta_h^{\tau}(f^{\epsilon}_{{\alpha_i}_{\epsilon}})\big) \Big\|^2_{\bLambda_h^{-1}}  + \frac{2\epsilon^2K^2}{\lambda}\\
     &\leq 4 \Big\|\sum_{\tau=1}^K\bphi_h^{\tau}\eta_h^{\tau}(f^{\epsilon}_{{\alpha_i}_{\epsilon}})\Big\|^2_{\bLambda_h^{-1}}\cdot \ind\big\{\|f^{\epsilon}_{{\alpha_i}_{\epsilon}}\|_{\infty}\leq R_{h+1}+2\epsilon \big\} + \frac{6\epsilon^2K^2}{\lambda},
\end{align*}
where the last inequality holds by the fact that 
\begin{align*}
    4\Big\|\sum_{k=1}^{K}\bphi_h^{\tau}\big(\eta_h^{\tau}(f_{{\alpha_i}_{\epsilon}}) -\eta_h^{\tau}(f_{{\alpha_i}_{\epsilon}}^{\epsilon})\big) \Big\|^2_{\bLambda_h^{-1}}\leq 4\epsilon^2\sum_{\tau, \tau'=1}^K\big|\bphi_h^{\tau}\bLambda_h^{-1}\bphi_h^{\tau'}\big| \leq 4\epsilon^2K^2/\lambda.
\end{align*}
With a union bound over $\cN_{h+1}(\epsilon)$ and $\cN_{[0,H]}(\epsilon)$ and by \Cref{lemma:self-normalize concentration for offline setting}, we have
\begin{align*}
    &\PP\Bigg \{\sup_{\substack{{\alpha_i}_{\epsilon}\in\cN_{[0,H]}(\epsilon)\\f_{{\alpha_i}_{\epsilon}}^{\epsilon}\in\cN_{h+1}(\epsilon)}}  \Big\|\sum_{\tau=1}^K\bphi_h^{\tau}\eta_h^{\tau}(f^{\epsilon}_{{\alpha_i}_{\epsilon}})\Big\|^2_{\bLambda_h^{-1}}\cdot \ind\big\{\|f^{\epsilon}_{{\alpha_i}_{\epsilon}}\|_{\infty}\leq R_{h+1}+2\epsilon\big\} \\
    &> (R_{h+1}+2\epsilon)^2\Big(2\log\frac{3dH^2|\cN_{h+1}(\epsilon)||\cN_{[0,H]}(\epsilon)|}{\delta} +d\log\Big(1+\frac{K}{\lambda}\Big)\Big)
    \Bigg\} \leq \frac{\delta}{3dH^2}.
\end{align*}
Then with probability at least $1-\delta/3dH^2$, for all $f_{\alpha_i}\in \cF_{h+1}(\alpha_i)$, we have
\begin{align*}
    &\Big\|\sum_{k=1}^K \bphi_h^{\tau}\eta_h^{\tau}(f_{\alpha_i}) \Big\|^2_{\bLambda_h^{-1}} \cdot \ind\big\{\|f_{\alpha_i}\|_{\infty}\leq R_{h+1} \big\}\\
    &\leq 4\inf_{\epsilon>0}\Big\{(R_{h+1}+2\epsilon)^2 \Big(2\log\Big(\frac{3dH^2|\cN_{h+1}(\epsilon)||\cN_{[0,H]}(\epsilon)|}{\delta}\Big)+d\log\Big(1+\frac{K}{\lambda} \Big) \Big)+\frac{6\epsilon^2K^2}{\lambda} \Big\}.
\end{align*}
By \Cref{lemma:covering number of the function class V} and \Cref{lemma:Covering number of an interval} together with $D_0=H\sqrt{Kd/\lambda}$ and $B_0=4\sqrt{d}H\sqrt{\iota}$, setting $\epsilon = d^{3/2}H^2/(K^{3/2}\sqrt{\kappa})$ and $K\geq \sqrt{d}H/(32\sqrt{\kappa}\iota)$, 
we have $\log|\cN_{h+1}(\epsilon)|\leq 2d^2\log(512K^3\iota/d^{3/2}H^2)$. Thus, we have
\begin{align*}
    \Big\|\sum_{k=1}^K \bphi_h^{\tau}\eta_h^{\tau}(f_{\alpha_i}) \Big\|^2_{\bLambda_h^{-1}} \cdot \ind\big\{\|f_{\alpha_i}\|_{\infty}\leq R_{h+1} \big\}&\leq \frac{512dH^4\iota}{K\kappa}\Big(2\log\frac{2dH^2}{\delta}+4d^2\log\frac{512K^3\iota}{d^{3/2}H^2}\Big)\\
    &\leq \frac{20480d^3H^4\iota^2}{K\kappa}. 
\end{align*}
Then, with a union bound over $i\in[d]$, we have
\begin{align*}        &\PP\bigg(\sup_{i\in[d]}\Big\|\sum_{\tau=1}^K\bphi_h^{\tau}\eta_h^{\tau}\big([\widehat{V}_{h+1}^{\rho}]_{\alpha_i} - [V_{h+1}^{\star, \rho}]_{\alpha_i}\big)\Big\|_{\bLambda_h^{-1}} > \frac{143d^{3/2}H^2\iota}{\sqrt{K\kappa}} \bigg)\\
    &\leq \PP \bigg( \sup_{i\in[d]}\Big\|\sum_{\tau=1}^K\bphi_h^{\tau}\eta_h^{\tau}\big([\widehat{V}_{h+1}^{\rho}]_{\alpha_i} - [V_{h+1}^{\star, \rho}]_{\alpha_i}\big)\Big\|_{\bLambda_h^{-1}} \ind\big\{\big\|[\widehat{V}_{h+1}^{\rho}]_{\alpha_i} - [V_{h+1}^{\star, \rho}]_{\alpha_i}\big\|_{\infty}\leq R_{h+1}\big\}\\
    &\quad >\frac{143d^{3/2}H^2\iota}{\sqrt{K\kappa}}\bigg) + \PP \big( \ind\big\{\big\|[\widehat{V}_{h+1}^{\rho}]_{\alpha_i} - [V_{h+1}^{\star, \rho}]_{\alpha_i}\big\|_{\infty}> R_{h+1}\big\}\big)\\
    &\leq \frac{\delta}{3H^2}+\delta_{h+1},
\end{align*}
which implies with probability at least $1-\delta/3H^2 - \delta_{h+1}$, the term iv at step $h$ can be bounded as
\begin{align}
\label{eq:Estimation Error Bound-bound iv of step h+1}
    \text{iv} \leq \frac{143d^{3/2}H^2\iota}{\sqrt{K\kappa}}\sum_{i=1}^d\|\phi_i(s,a)\mathbf{1}_i\|_{\bLambda_h^{-1}}.
\end{align}
Then by a union bound over \eqref{eq:Estimation Error Bound-bound i&ii&iii of step h+1} and \eqref{eq:Estimation Error Bound-bound iv of step h+1}, if $K>20449d^2H^2/\kappa$, then with probability at least $1-2\delta/3H^2-\delta_{h+1}$ we have
\begin{align}
    &\Big|\inf_{P_h(\cdot|s,a)\in\cU_h^{\rho}(s,a;\bmu_{h,i}^0)}[\PP_h\widehat{V}_{h+1}^{\rho}](s,a)-\widehat{\inf_{P_h(\cdot|s,a)\in\cU_h^{\rho}(s,a;\bmu_{h,i}^0)}}[\PP_h\widehat{V}_{h+1}^{\rho}](s,a)\Big|\notag\\
    &\leq 4\sqrt{d}H\sqrt{\iota}\sum_{i=1}^d\|\phi_i(s,a)\mathbf{1}_i\|_{\bLambda_h^{-1}}:=\Gamma_h(s,a).\label{eq:Estimation Error Bound-estimation error bound at step h}
\end{align}
Further, when $K>\max\{512\log(3H^2/\delta)/\kappa^2, 4/\kappa\}$, by \Cref{lemma:matrix-normalized concentration}, with probability at least $1-\delta/3H^2$, we have
\begin{align}
\label{eq:Estimation Error Bound-bound the error by K at step h}
    \Gamma_h(s,a)\leq 4\sqrt{d}H\sqrt{\iota}\sum_{i=1}^d\|\phi_i(s,a)\mathbf{1}_i\|_{\bLambda_h^{-1}}\leq \frac{8\sqrt{d}H\sqrt{\iota}}{\sqrt{K\kappa}}.
\end{align}
Then by a union bound over \eqref{eq:Estimation Error Bound-estimation error bound at step h} and \eqref{eq:Estimation Error Bound-bound the error by K at step h}, under the event $\cE_{h+1}$, with probability at least $1-\delta/H^2-\delta_{h+1}$, we have
\begin{align}
    &V_h^{\star, \rho}(s)-\widehat{V}_h^{\rho}(s)\notag \\
    &=Q^{\star, \rho}_h(s, \pi^{\star}(s)) - \widehat{Q}_h^{\rho}(s, \pi^{\star}(s))+\widehat{Q}_h^{\rho}(s, \pi^{\star}(s)) - \widehat{Q}^{\rho}_h(s, \hat{\pi}(s))\notag \\
    &\leq \inf_{P_h(\cdot|s,a)\in\cU_h^{\rho}(s,a;\bmu_{h,i}^0)}[\PP_hV_{h+1}^{\star,\rho}](s,\pi^{\star}(s))-\widehat{\inf_{P_h(\cdot|s,a)\in\cU_h^{\rho}(s,a;\bmu_{h,i}^0)}}[\PP_h\widehat{V}_{h+1}^{\rho}](s,a)+\Gamma_h(s,\pi^{\star}(s))\notag \\
    &= \inf_{P_h(\cdot|s,a)\in\cU_h^{\rho}(s,a;\bmu_{h,i}^0)}[\PP_hV_{h+1}^{\star,\rho}](s,\pi^{\star}(s)) -\inf_{P_h(\cdot|s,a)\in\cU_h^{\rho}(s,a;\bmu_{h,i}^0)}[\PP_h\widehat{V}_{h+1}^{\rho}](s,\pi^{\star}(s)) \notag \\
    &\quad  +\inf_{P_h(\cdot|s,a)\in\cU_h^{\rho}(s,a;\bmu_{h,i}^0)}[\PP_h\widehat{V}_{h+1}^{\rho}](s,\pi^{\star}(s))-\widehat{\inf_{P_h(\cdot|s,a)\in\cU_h^{\rho}(s,a;\bmu_{h,i}^0)}}[\PP_h\widehat{V}_{h+1}^{\rho}](s,a)+\Gamma_h(s,\pi^{\star}(s))\notag\\
    &\leq R_{h+1}+2\Gamma_h(s,\pi^{\star}(s))\label{eq:Estimation Error Bound-bound the difference of two inf by R_h+1}\\
    &\leq \frac{16\sqrt{d}H(H-h)\sqrt{\iota}}{\sqrt{K\kappa}}+\frac{16\sqrt{d}H\sqrt{\iota}}{\sqrt{K\kappa}}\notag\\
    &=\frac{16\sqrt{d}H(H-h+1)\sqrt{\iota}}{\sqrt{K\kappa}}:=R_h,\notag
\end{align}
where \eqref{eq:Estimation Error Bound-bound the difference of two inf by R_h+1} holds by the following argument
\begin{align}
   &\inf_{P_h(\cdot|s,a)\in\cU_h^{\rho}(s,a;\bmu_{h,i}^0)}[\PP_hV_{h+1}^{\star,\rho}](s,\pi^{\star}(s)) -\inf_{P_h(\cdot|s,a)\in\cU_h^{\rho}(s,a;\bmu_{h,i}^0)}[\PP_h\widehat{V}_{h+1}^{\rho}](s,\pi^{\star}(s))\notag \\
   &\leq [\hat{\PP}_hV_{h+1}^{\star,\rho}](s,\pi^{\star}(s)) - [\hat{\PP}_h\widehat{V}_{h+1}^{\rho}](s,\pi^{\star}(s))\notag \\
   &\leq \sup_s|V_{h+1}^{\star,\rho}(s) - \widehat{V}_{h+1}^{\rho}(s)|\notag \\
   &\leq R_{h+1},\label{eq:Estimation Error Bound-bound the difference of value function by R_h+1}
\end{align}
where $\hat{P}_h(\cdot|s,a)=\arginf_{P_h(\cdot|s,a)\in\cU_h^{\rho}(s,a;\bmu_{h,i}^0)}[\PP_h\widehat{V}_{h+1}^{\rho}](s,a), \forall (s,a)\in\cS\times\cA$, and \eqref{eq:Estimation Error Bound-bound the difference of value function by R_h+1} is due to the induction assumption \eqref{eq:Estimation Error Bound-event E_h+1}.
Finally, denote 
\begin{align*}
    \cE_h=\{0\leq V^{\star,\rho}_{h+1}(s) - \widehat{V}_{h+1}^{\rho}(s)\leq R_h, \forall s\in\cS\},
\end{align*}
then we have $P(\cE_h)\leq \delta_{h+1}+\delta/H^2:=\delta_h$.

\paragraph{Generalization.} 
By induction and a union bound over $h\in[H]$, setting 
\begin{align*}
    \Gamma_h(s,a)= 4\sqrt{d}H\sqrt{\iota}\sum_{i=1}^d\|\phi_i(s,a)\mathbf{1}_i\|_{\bLambda_h^{-1}},
\end{align*}
then with probability at least $1-(\delta/H^2 + 2\delta/H^2 +\cdots+H\delta/H^2)=1-dH(H+1)\delta/2H^2>1-\delta
$, for all $(s,a,h)\in\cS\times\cA\times[H]$, we have
\begin{align*}
     \Big|\inf_{P_h(\cdot|s,a)\in\cU_h^{\rho}(s,a;\bmu_{h,i}^0)}[\PP_h\widehat{V}_{h+1}^{\rho}](s,a)-\widehat{\inf_{P_h(\cdot|s,a)\in\cU_h^{\rho}(s,a;\bmu_{h,i}^0)}}[\PP_h\widehat{V}_{h+1}^{\rho}](s,a)\Big|\leq \Gamma_h(s,a).
\end{align*}
This concludes the proof.
\end{proof}

\subsection{Proof of \Cref{lemma:variance estimation}}
\begin{proof}
    Note that the conditional variance estimation does not involve any element of model uncertainty, and thus the proof follows from Lemma 5 of \cite{xiong2023nearly}. Recall that we estimate $[\VV_h[V_{h+1}^{\rho}]_{\alpha}](s,a)$ based on $\cD'$ as
    \begin{align*}
        \widehat{\sigma}_h^2(s,a;\alpha) = \max\Big\{1, \big[\bphi(s,a)^{\top}\tilde{\beta}_{h,2}(\alpha) \big]_{[0,H^2]}-\big[\bphi(s,a)^{\top}\tilde{\beta}_{h,1}(\alpha) \big]^2_{[0,H]} - \tilde{O}\Big(\frac{\sqrt{d}H^3}{\sqrt{K\kappa}}\Big)\Big\}.
    \end{align*}
    Note that 
    \begin{align*}
        &\Big|\big[\bphi(s,a)^{\top}\tilde{\beta}_{h,2}(\alpha) \big]_{[0,H^2]}-\big[\bphi(s,a)^{\top}\tilde{\beta}_{h,1}(\alpha) \big]^2_{[0,H]} - [\PP_h[\widehat{V}^{'\rho}_{h+1}]_{\alpha}^2](s,a) -([\PP_h[\widehat{V}^{'\rho}_{h+1}]_{\alpha}](s,a))^2  \Big|\\
        &\leq \Big|\big[\bphi(s,a)^{\top}\tilde{\beta}_{h,2}(\alpha) \big]_{[0,H^2]} - [\PP_h[\widehat{V}^{'\rho}_{h+1}]_{\alpha}^2](s,a) \Big| + \Big|\big[\bphi(s,a)^{\top}\tilde{\beta}_{h,1}(\alpha) \big]^2_{[0,H]} - ([\PP_h[\widehat{V}^{'\rho}_{h+1}]_{\alpha}](s,a))^2  \Big|\\
        &\leq \underbrace{\Big|\bphi(s,a)^{\top}\tilde{\beta}_{h,2}(\alpha) - [\PP_h[\widehat{V}^{'\rho}_{h+1}]_{\alpha}^2](s,a) \Big|}_\text{i} + 2H\underbrace{\Big|\bphi(s,a)^{\top}\tilde{\beta}_{h,1}(\alpha) - [\PP_h[\widehat{V}^{'\rho}_{h+1}]_{\alpha}](s,a) \Big|}_\text{ii}.
    \end{align*}
    Note that the estimation error i and ii both come from regular ridge regressions with targets $[\widehat{V}^{'\rho}_{h+1}(s)]_{\alpha}^2$ and $[\widehat{V}^{'\rho}_{h+1}(s)]_{\alpha}$, respectively. Thus, the analysis is standard and for simplicity we omit the details here and focus on the results: with probability at least $1-\delta/2$, we have
    \begin{align}
    \label{eq:variance estimation-variance estimation error}
        &\Big|\big[\bphi(s,a)^{\top}\tilde{\beta}_{h,2}(\alpha) \big]_{[0,H^2]}-\big[\bphi(s,a)^{\top}\tilde{\beta}_{h,1}(\alpha) \big]^2_{[0,H]} - [\PP_h[\widehat{V}^{'\rho}_{h+1}]_{\alpha}^2](s,a) -([\PP_h[\widehat{V}^{'\rho}_{h+1}]_{\alpha}](s,a))^2  \Big|\notag\\
        &\leq \tilde{O}\Big(\frac{dH^2}{\sqrt{K\kappa}} \Big).
    \end{align}
    Then by \Cref{th:DRPVI} and \Cref{lemma:matrix-normalized concentration}, for all $(s,a,h)\in\cS\times\cA\times[H]$, with probability at least $1-\delta/2$, we have
    \begin{align}
        &\big|[\Var_h[\widehat{V}_{h+1}^{'\rho}]_{\alpha}](s,a) - [\Var_h[V_{h+1}^{\star, \rho}]_{\alpha}](s,a) \big|\notag \\
        &\leq \big| [\PP_h[\widehat{V}_{h+1}^{'\rho}]_{\alpha}^2](s,a) -[\PP_h[V_{h+1}^{\star, \rho}]_{\alpha}^2](s,a) \big| + \big| \big([\PP_h[\widehat{V}_{h+1}^{'\rho}]_{\alpha}](s,a)\big)^2 -\big([\PP_h[V_{h+1}^{\star, \rho}]_{\alpha}](s,a)\big)^2 \big|\notag \\
        &\leq 2H\big|\big[\PP_h([\widehat{V}_{h+1}^{\rho}]_{\alpha} -  [V_{h+1}^{\star, \rho}]_{\alpha})\big] (s,a)\big| + 2H \big|\big[\PP_h\big([V_{h+1}^{\star, \rho}]_{\alpha} - [V_{h+1}^{\star, \rho}]_{\alpha} \big)\big](s,a) \big|\notag \\
        &\leq  \tilde{O}\Big(\frac{\sqrt{d}H^3}{\sqrt{K\kappa}} \Big).\label{eq:variance estimation-variance approximation gap}
    \end{align}
    By \eqref{eq:variance estimation-variance estimation error} and \eqref{eq:variance estimation-variance approximation gap} and a union bound, we know that with probability at least $1-\delta$, we have
    \begin{align*}
        &\Big|\big[\bphi(s,a)^{\top}\tilde{\beta}_{h,2}(\alpha) \big]_{[0,H^2]}-\big[\bphi(s,a)^{\top}\tilde{\beta}_{h,1}(\alpha) \big]^2_{[0,H]} - [\Var_h[V_{h+1}^{\star, \rho}]_{\alpha}](s,a)  \Big|\\
        &\leq \Big|\big[\bphi(s,a)^{\top}\tilde{\beta}_{h,2}(\alpha) \big]_{[0,H^2]}-\big[\bphi(s,a)^{\top}\tilde{\beta}_{h,1}(\alpha) \big]^2_{[0,H]}-[\Var_h[\widehat{V}_{h+1}^{'\rho}]_{\alpha}](s,a)  \Big| \\
        &\quad +\big|[\Var_h[\widehat{V}_{h+1}^{'\rho}]_{\alpha}](s,a) -[\Var_h[V_{h+1}^{\star, \rho}]_{\alpha}](s,a)  \big|  \\
        &\leq \tilde{O}\Big(\frac{dH^3}{\sqrt{K\kappa}} \Big),
    \end{align*}
    which implies that 
    \begin{align*}
        \big[\bphi(s,a)^{\top}\tilde{\beta}_{h,2}(\alpha) \big]_{[0,H^2]}-\big[\bphi(s,a)^{\top}\tilde{\beta}_{h,1}(\alpha) \big]^2_{[0,H]} - \tilde{O}\Big(\frac{dH^3}{\sqrt{K\kappa}} \Big)\leq [\Var_h[V_{h+1}^{\star, \rho}]_{\alpha}](s,a) .
    \end{align*}
    By the fact that the operator $\min\{1,\cdot\}$ is order preserving, thus we have
    \begin{align*}
        \widehat{\sigma}_h^2(s,a;\alpha)\leq  [\VV_h[V_{h+1}^{\star, \rho}]_{\alpha}](s,a).
    \end{align*}
    Further, by the fact that the operator $\min\{1,\cdot\}$ is a contraction map, \eqref{eq:variance estimation-variance estimation error} and \eqref{eq:variance estimation-variance approximation gap}, we have
    \begin{align*}
        &\big|\widehat{\sigma}_h^2(s,a;\alpha) -\big[\VV_h[V_{h+1}^{\star,\rho}]_{\alpha}\big](s,a)\big|\\
        & \leq \big|\widehat{\sigma}_h^2(s,a;\alpha) - \big[\VV_h[\widehat{V}_{h+1}^{'\rho}]_{\alpha}\big](s,a) \big| + \big|\big[\VV_h[\widehat{V}_{h+1}^{'\rho}]_{\alpha}\big](s,a) - \big[\VV_h[V_{h+1}^{\star, \rho}]_{\alpha}\big](s,a) \big|\\
        &\leq \Big|\big[\bphi(s,a)^{\top}\tilde{\beta}_{h,2}(\alpha) \big]_{[0,H^2]}-\big[\bphi(s,a)^{\top}\tilde{\beta}_{h,1}(\alpha) \big]^2_{[0,H]} - \tilde{O}\Big(\frac{dH^3}{\sqrt{K\kappa}}\Big)- [\Var_h[\widehat{V}_{h+1}^{'\rho}]_{\alpha}](s,a)\Big| \\
        &\quad +\big|[\Var_h[\widehat{V}_{h+1}^{'\rho}]_{\alpha}](s,a) - [\Var_h[V_{h+1}^{\star,\rho}]_{\alpha}](s,a) \big|\\
        &\leq \tilde{O}\Big(\frac{dH^2}{\sqrt{K\kappa}} \Big) + \tilde{O}\Big(\frac{dH^3}{\sqrt{K\kappa}} \Big)+ \tilde{O}\Big(\frac{\sqrt{d}H^3}{\sqrt{K\kappa}} \Big)\\
        &=\tilde{O}\Big(\frac{dH^3}{\sqrt{K\kappa}}\Big).
    \end{align*}
    This concludes the proof.
\end{proof}

\subsection{Proof of \Cref{lemma:Variance-Aware Reference-Advantage Decomposition}}
\begin{proof}
    Note that the reference-advantage decomposition is exactly the same as that in the proof of \Cref{lemma:Reference-Advantage Decomposition}, thus we have
    \begin{align*}
        &\inf_{P_h(\cdot|s,a)\in\cU_h^{\rho}(s,a;\bmu_{h,i}^0)}[\PP_h\widehat{V}_{h+1}^{\rho}](s,a)-\widehat{\inf_{P_h(\cdot|s,a)\in\cU_h^{\rho}(s,a;\bmu_{h,i}^0)}}[\PP_h\widehat{V}_{h+1}^{\rho}](s,a)\\
        &\leq \underbrace{\sum_{i=1}^d\phi_i(s,a)\mathbf{1}_i^{\top}\big(\EE^{\bmu_{h}^0}[V^{\star,\rho}_{h+1}(s)]_{\alpha_i} - \widehat{\EE}^{\bmu_{h}^0}[V^{\star,\rho}_{h+1}(s)]_{\alpha_i} \big)}_\text{reference uncertainty} \\
        &\quad + \underbrace{\sum_{i=1}^d\phi_i(s,a)\mathbf{1}_i^{\top}\big(\EE^{\bmu_{h}^0}\big[[\widehat{V}^{\rho}_{h+1}(s)]_{\alpha_i} - [V^{\star,\rho}_{h+1}(s)]_{\alpha_i} \big] - \widehat{\EE}^{\bmu_{h}^0}\big[[\widehat{V}^{\rho}_{h+1}(s)]_{\alpha_i} - [V^{\star,\rho}_{h+1}(s)]_{\alpha_i} \big] \big)}_\text{advantage uncertainty}.
    \end{align*}
    Next, we further decompose the reference uncertainty and the advantage uncertainty, respectively.
    \paragraph{The Reference Uncertainty.} Specifically, we have
    \begin{align*}
    &\sum_{i=1}^d\phi_i(s,a)\mathbf{1}_i^{\top}\big(\EE^{\bmu_{h}^0}[V^{\star,\rho}_{h+1}(s)]_{\alpha_i} - \widehat{\EE}^{\bmu_{h}^0}[V^{\star,\rho}_{h+1}(s)]_{\alpha_i} \big)\\
    &=\sum_{i=1}^d\phi_i(s,a)\mathbf{1}_i^{\top}\Big(\EE^{\bmu_{h}^0}[V^{\star,\rho}_{h+1}(s)]_{\alpha_i} - \bSigma_h^{-1}(\alpha_i)\sum_{\tau=1}^K\frac{\bphi_h^{\tau}\big[\PP_h^0[V_{h+1}^{\star, \rho}]_{\alpha_i} \big](s_h^{\tau}, a_h^{\tau})}{\widehat{\sigma}_h^2(s_h^{\tau},a_h^{\tau};\alpha_i)} \\
    & \quad + \bSigma_h^{-1}(\alpha_i)\sum_{\tau=1}^K\frac{\bphi_h^{\tau}\big[\PP_h^0[V_{h+1}^{\star, \rho}]_{\alpha_i} \big](s_h^{\tau}, a_h^{\tau})}{\widehat{\sigma}_h^2(s_h^{\tau},a_h^{\tau};\alpha_i)} - \bSigma_h^{-1}(\alpha_i)\sum_{\tau=1}^K\frac{\bphi_h^{\tau}[V^{\star,\rho}_{h+1}(s_{h+1}^{\tau})]_{\alpha_i}}{\widehat{\sigma}_h^2(s_h^{\tau},a_h^{\tau};\alpha_i)} \Big)\\
    & =\lambda\sum_{i=1}^d\phi_i(s,a)\mathbf{1}_i^{\top}\bSigma_h^{-1}(\alpha_i)\EE^{\bmu_{h}^0}[V^{\star,\rho}_{h+1}(s)]_{\alpha_i} + \sum_{i=1}^d\phi_i(s,a)\mathbf{1}_i^{\top}\bSigma_h^{-1}(\alpha_i)\sum_{\tau=1}^K\frac{\bphi_h^\tau\eta_h^\tau([V_{h+1}^{\star, \rho}]_{\alpha_i})}{\widehat{\sigma}_h^2(s_h^{\tau},a_h^{\tau};\alpha_i)}\\
    &\leq \underbrace{\lambda\sum_{i=1}^d\|\phi_i(s,a)\mathbf{1}_i\|_{\bSigma_h^{-1}(\alpha_i)}\|\EE^{\bmu_{h}^0}[V^{\star,\rho}_{h+1}(s)]_{\alpha_i}\|_{\bSigma_h^{-1}(\alpha_i)}}_\text{i}\\
    &\quad +\underbrace{\sum_{i=1}^d\|\phi_i(s,a)\mathbf{1}_i\|_{\bSigma_h^{-1}(\alpha_i)}\Big\|\sum_{\tau=1}^K\frac{\bphi_h^\tau\eta_h^\tau([V_{h+1}^{\star, \rho}]_{\alpha_i})}{\widehat{\sigma}_h^2(s_h^{\tau},a_h^{\tau};\alpha_i)}\Big\|_{\bSigma_h^{-1}(\alpha_i)}}_\text{ii}.
    \end{align*}
\paragraph{The Advantage Uncertainty.} 
Similar to the argument in decomposing the reference uncertainty, we have
\begin{align*}
&\sum_{i=1}^d\phi_i(s,a)\mathbf{1}_i^{\top}\big(\EE^{\bmu_{h}^0}\big[[\widehat{V}^{\rho}_{h+1}(s)]_{\alpha_i} - [V^{\star,\rho}_{h+1}(s)]_{\alpha_i} \big] - \widehat{\EE}^{\bmu_{h}^0}\big[[\widehat{V}^{\rho}_{h+1}(s)]_{\alpha_i} - [V^{\star,\rho}_{h+1}(s)]_{\alpha_i} \big] \big)\\
    &\leq \underbrace{\lambda\sum_{i=1}^d\|\phi_i(s,a)\mathbf{1}_i\|_{\bSigma_h^{-1}(\alpha_i)}\Big\|\EE^{\bmu_{h}^0}\big[[\widehat{V}^{\rho}_{h+1}(s)]_{\alpha_i} - [V^{\star,\rho}_{h+1}(s)]_{\alpha_i} \big]\Big\|_{\bSigma_h^{-1}(\alpha_i)}}_\text{iii}\\
&\quad + \underbrace{\sum_{i=1}^d\|\phi_i(s,a)\mathbf{1}_i\|_{\bSigma_h^{-1}(\alpha_i)}\Big\|\sum_{\tau=1}^K\frac{\bphi_h^{\tau}\eta_h^{\tau}([\widehat{V}^{\rho}_{h+1}(s)]_{\alpha_i} - [V^{\star,\rho}_{h+1}(s)]_{\alpha_i})}{\widehat{\sigma}_h^2(s_h^{\tau}, a_h^{\tau};\alpha_i)}\Big\|_{\bSigma_h^{-1}(\alpha_i)}}_\text{iv}.
    \end{align*}
Put terms i-iv together, we have
\begin{align*}
    &\inf_{P_h(\cdot|s,a)\in\cU_h^{\rho}(s,a;\bmu_{h,i}^0)}[\PP_h\widehat{V}_{h+1}^{\rho}](s,a)-\widehat{\inf_{P_h(\cdot|s,a)\in\cU_h^{\rho}(s,a;\bmu_{h,i}^0)}}[\PP_h\widehat{V}_{h+1}^{\rho}](s,a)\\
    &\leq \underbrace{\lambda\sum_{i=1}^d\|\phi_i(s,a)\mathbf{1}_i\|_{\bSigma_h^{-1}(\alpha_i)}\|\EE^{\bmu_{h}^0}[V^{\star,\rho}_{h+1}(s)]_{\alpha_i}\|_{\bSigma_h^{-1}(\alpha_i)}}_\text{i}\\
    &\quad +\underbrace{\sum_{i=1}^d\|\phi_i(s,a)\mathbf{1}_i\|_{\bSigma_h^{-1}(\alpha_i)}\Big\|\sum_{\tau=1}^K\frac{\bphi_h^\tau\eta_h^\tau([V_{h+1}^{\star, \rho}]_{\alpha_i})}{\widehat{\sigma}_h^2(s_h^{\tau},a_h^{\tau};\alpha_i)}\Big\|_{\bSigma_h^{-1}(\alpha_i)}}_\text{ii} \\
    &\quad  + \underbrace{\lambda\sum_{i=1}^d\|\phi_i(s,a)\mathbf{1}_i\|_{\bSigma_h^{-1}(\alpha_i)}\Big\|\EE^{\bmu_{h}^0}\big[[\widehat{V}^{\rho}_{h+1}(s)]_{\alpha_i} - [V^{\star,\rho}_{h+1}(s)]_{\alpha_i} \big]\Big\|_{\bSigma_h^{-1}(\alpha_i)}}_\text{iii}\\
    &\quad +\underbrace{\sum_{i=1}^d\|\phi_i(s,a)\mathbf{1}_i\|_{\bSigma_h^{-1}(\alpha_i)}\Big\|\sum_{\tau=1}^K\frac{\bphi_h^{\tau}\eta_h^{\tau}([\widehat{V}^{\rho}_{h+1}(s)]_{\alpha_i} - [V^{\star,\rho}_{h+1}(s)]_{\alpha_i})}{\widehat{\sigma}_h^2(s_h^{\tau}, a_h^{\tau};\alpha_i)}\Big\|_{\bSigma_h^{-1}(\alpha_i)}}_\text{iv}.
\end{align*}
By similar argument as \Cref{lemma:Reference-Advantage Decomposition}, we know there exist $\{\tilde{\alpha}_i\}_{i\in[d]}$ such that 
\begin{align*}
    &\Big|\inf_{P_h(\cdot|s,a)\in\cU_h^{\rho}(s,a;\bmu_{h,i}^0)}[\PP_h\widehat{V}_{h+1}^{\rho}](s,a)-\widehat{\inf_{P_h(\cdot|s,a)\in\cU_h^{\rho}(s,a;\bmu_{h,i}^0)}}[\PP_h\widehat{V}_{h+1}^{\rho}](s,a)\Big|\\
    &\leq \underbrace{\lambda\sum_{i=1}^d\|\phi_i(s,a)\mathbf{1}_i\|_{\bSigma_h^{-1}(\tilde{\alpha}_i)}\|\EE^{\bmu_{h}^0}[V^{\star,\rho}_{h+1}(s)]_{\alpha_i}\|_{\bSigma_h^{-1}(\tilde{\alpha}_i)}}_\text{i}\\
    &\quad +\underbrace{\sum_{i=1}^d\|\phi_i(s,a)\mathbf{1}_i\|_{\bSigma_h^{-1}(\tilde{\alpha}_i)}\Big\|\sum_{\tau=1}^K\frac{\bphi_h^\tau\eta_h^\tau([V_{h+1}^{\star, \rho}]_{\alpha_i})}{\widehat{\sigma}_h^2(s_h^{\tau},a_h^{\tau};\alpha_i)}\Big\|_{\bSigma_h^{-1}(\tilde{\alpha}_i)}}_\text{ii} \\
    &\quad +\underbrace{\lambda\sum_{i=1}^d\|\phi_i(s,a)\mathbf{1}_i\|_{\bSigma_h^{-1}(\alpha_i)}\Big\|\EE^{\bmu_{h}^0}\big[[\widehat{V}^{\rho}_{h+1}(s)]_{\tilde{\alpha}_i} - [V^{\star,\rho}_{h+1}(s)]_{\alpha_i} \big]\Big\|_{\bSigma_h^{-1}(\tilde{\alpha}_i)}}_\text{iii}\\
    &\quad + \underbrace{\sum_{i=1}^d\|\phi_i(s,a)\mathbf{1}_i\|_{\bSigma_h^{-1}(\tilde{\alpha}_i)}\Big\|\sum_{\tau=1}^K\frac{\bphi_h^{\tau}\eta_h^{\tau}([\widehat{V}^{\rho}_{h+1}(s)]_{\tilde{\alpha}_i} - [V^{\star,\rho}_{h+1}(s)]_{\alpha_i})}{\widehat{\sigma}_h^2(s_h^{\tau}, a_h^{\tau};\alpha_i)}\Big\|_{\bSigma_h^{-1}(\tilde{\alpha}_i)}}_\text{iv}.
\end{align*}
This concludes the proof.
\end{proof}

\subsection{Proof of \Cref{lemma:Range Shrinkage}}
\label{sec:proof of range shrinkage}
\begin{proof}
By the robust bellman equation \eqref{eq:robust bellman equation}, we know 
\begin{align}
\label{eq:Range Shrinkage-robust bellman equation}
    V_h^{\pi,\rho}(s) = \EE_{a\sim\pi(\cdot|s)}\Big[r(s,a)+\inf_{P_h(\cdot|s,a) \in \cU_{h}^{\rho}(s,a;\bmu_h^0)}[\PP_hV_{h+1}^{\pi, \rho}](s,a)\Big].
\end{align}
Then, we can trivially bound $\max_{s\in\cS}V_h^{\pi,\rho}(s)$ as
\begin{align}
\label{eq:Range Shrinkage-maxV}
    \max_{s\in\cS}V_h^{\pi,\rho}(s) \leq \max_{s,a}\Big(1+\inf_{P_h(\cdot|s,a) \in \cU_{h}^{\rho}(s,a;\bmu_h^0)}[\PP_hV_{h+1}^{\pi, \rho}](s,a)\Big).
\end{align}
Further, by the definition of the $d$-rectangular uncertainty set, we have
\begin{align}
\label{eq:Range Shrinkage-decomposition of inf}
   \inf_{P_h(\cdot|s,a) \in \cU_{h}^{\rho}(s,a;\bmu_h^0)}[\PP_hV_{h+1}^{\pi, \rho}](s,a) = \sum_{i=1}^d\phi_i(s,a)\inf_{{\mu}_{h,i}\in{\cU}_{h,i}^{\rho}({\mu}_{h,i}^0)} \EE_{s\sim\mu_{h,i}}[V_{h+1}^{\pi, \rho}(s)].
\end{align}
Denoting $s_{\max}=\argmax_{s\in\cS}V_{h+1}^{\pi, \rho}(s)$ and $s_{\min}=\argmin_{s\in\cS}V_{h+1}^{\pi, \rho}(s)$, and 
for all $i\in[d]$, we construct a distribution $\check{\mu}_{h,i}=(1-\rho)\mu_{h,i}+\rho\delta_{s_{\min}}$, where $\delta_x$ is the Dirac Delta distribution with mass on $x$. Note that $\check{\mu}_{h,i}\in{\cU}_{h,i}^{\rho}({\mu}_{h,i}^0)$, thus we have
\begin{align}
\label{eq:Range Shrinkage-upper bound of the worst case factor expectation}
    \inf_{{\mu}_{h,i}\in{\cU}_{h,i}^{\rho}({\mu}_{h,i}^0)} \EE_{s\sim\mu_{h,i}}[V_{h+1}^{\pi, \rho}(s)] \leq \EE_{s\sim\check{\mu}_{h,i}}[V_{h+1}^{\pi, \rho}(s)] \leq (1-
    \rho)\max_{s\in\cS}V_{h+1}^{\pi,\rho}(s) + \rho\min_sV_{h+1}^{\pi,\rho}(s).
\end{align}
Combining \eqref{eq:Range Shrinkage-maxV}, \eqref{eq:Range Shrinkage-decomposition of inf} and \eqref{eq:Range Shrinkage-upper bound of the worst case factor expectation}, we have
\begin{align}
\label{eq:Range Shrinkage-upper bound of maxV}
    \max_{s\in\cS}V_h^{\pi,\rho}(s) \leq (1-\rho)\max_{s\in\cS}V_{h+1}^{\pi, \rho}(s)+\rho\min_{s\in\cS}V_{h+1}^{\pi,\rho}(s) + 1.
\end{align}
On the other hand, by \eqref{eq:Range Shrinkage-robust bellman equation}, we can trivially bound $\min_sV_h^{\pi, \rho}(s)$ as
\begin{align}
\label{eq:Range Shrinkage-minV}
    \min_sV_h^{\pi, \rho}(s) \geq \min_{s,a}\inf_{P_h(\cdot|s,a) \in \cU_{h}^{\rho}(s,a;\bmu_h^0)}[\PP_hV_{h+1}^{\pi, \rho}](s,a).
\end{align}
By the fact that 
\begin{align}
\label{eq:Range Shrinkage-lower bound of the worst case factor expectation}
    \inf_{{\mu}_{h,i}\in{\cU}_{h,i}^{\rho}({\mu}_{h,i}^0)} \EE_{s\sim\mu_{h,i}}[V_{h+1}^{\pi, \rho}(s)] \geq \min_{s\in\cS}V_{h+1}^{\pi,\rho}(s),
\end{align}
combining \eqref{eq:Range Shrinkage-decomposition of inf}, \eqref{eq:Range Shrinkage-minV} and \eqref{eq:Range Shrinkage-lower bound of the worst case factor expectation}, we have
\begin{align}
    \label{eq:Range Shrinkage-lower bound of minV}
     \min_sV_h^{\pi, \rho}(s) \geq \min_{s\in\cS}V_{h+1}^{\pi,\rho}(s).
\end{align}
For any $h\in[H]$, by \eqref{eq:Range Shrinkage-upper bound of maxV} and \eqref{eq:Range Shrinkage-lower bound of minV}, we have
\begin{align}
    &\max_{s\in\cS}V_{h}^{\pi,\rho}(s)-\min_{s\in\cS}V_h^{\pi,\rho}(s)\notag\\
    &\leq 1+(1-\rho)\max_{s\in\cS}V_{h+1}^{\pi,\rho}(s)-\min_{s\in\cS}V_{h+1}^{\pi,\rho}(s) + \rho\min_{s\in\cS}V_{h+1}^{\pi,\rho}(s)\notag\\
    &= 1+(1-\rho)\big[\max_{s\in\cS}V_{h+1}^{\pi,\rho}(s)-\min_{s\in\cS}V_{h+1}^{\pi,\rho}(s)\big].\label{eq:Range Shrinkage-upper bound the range of current step by the range of former step}
\end{align}
For step $H$, by the definition of the value function, we have $0\leq V_H^{\pi,\rho}(s)\leq 1, \forall s\in\cS$. Applying \eqref{eq:Range Shrinkage-upper bound the range of current step by the range of former step} with $h=H-1$ leads to $\max_{s\in\cS}V_{H-1}^{\pi,\rho}(s)-\min_{s\in\cS}V_{H-1}^{\pi,\rho}(s)\leq 1+(1-\rho)\cdot 1$. We finish the proof by recursively applying \eqref{eq:Range Shrinkage-upper bound the range of current step by the range of former step}. 
\end{proof}

\subsection{Proof of \Cref{lemma:reduction to testing}}
\label{sec:proof of reduction to testing}
\begin{proof} The proof of \Cref{lemma:reduction to testing} consists of the following two steps:
\paragraph{Step 1: lower bound the suboptimality by Hamming distance.}
    For any $\bxi\in\{-1,1\}^{dH}$, denote $V_{\bxi}^{\star, \rho}(s)$ as the optimal robust value function for the MDP instance $M_{\bxi}$. For any function $\pi$, denote $V_{\bxi}^{\pi,\rho}$ as the robust value function corresponding to a policy $\pi$. Then by definition, we have
    \begin{align*}
        &V_{\bxi}^{\star,\rho}(x_1)=\max_{\pi}\inf_{P\in\cU^{\rho}(P^0)}\EE^{\pi, P}\big[r_1(s_1, a_1) + \cdots + r_H(s_H,a_H)|s_1=x_1\big],\\
        &V_{\bxi}^{\pi,\rho}(x_1)=\inf_{P\in\cU^{\rho}(P^0)}\EE^{\pi, P}\big[r_1(s_1, a_1) + \cdots + r_H(s_H,a_H)|s_1=x_1\big].
    \end{align*}
    For any given $\bxi$, the optimal action at step $h$ is 
    \begin{align*}
        a_h^{\star}=((1+\xi_{h1})/2, \cdots, (1+\xi_{hd})/2).
    \end{align*}
    The worst case transition at the first step is known as 
    \begin{align*}
        \PP_1(x_1|x_1,a)=(1-\rho), ~\PP_1(x_2|x_1,a)=\rho, ~\PP_1(x_2|x_2,a)=1,~\forall a\in\cA,
    \end{align*}
    and from the second step on, the state always stays at $s_2$. With these facts in mind, we have
    \begin{align*}
        &V_{\bxi}^{\star,\rho}(x_1) \\
        &= {\delta}\Big\{\Big[\frac{1}{2}+\sum_{i=1}^d\frac{1+\xi_{1i}}{4d}\Big] +(1-\rho)\Big[\frac{1}{2}+\sum_{i=1}^d\frac{1+\xi_{2i}}{4d}\Big] +\cdots+(1-\rho)\Big[\frac{1}{2}+\sum_{i=1}^d\frac{1+\xi_{Hi}}{4d}\Big]  \Big\}\\
        &=\frac{\delta}{2d}\Big\{\Big[d+\sum_{i=1}^d\frac{1+\xi_{1i}}{2}\Big] +(1-\rho)\Big[d+\sum_{i=1}^d\frac{1+\xi_{2i}}{2}\Big] +\cdots+(1-\rho)\Big[d+\sum_{i=1}^d\frac{1+\xi_{Hi}}{2}\Big] \Big\},
    \end{align*}
    and
    \begin{align*}
        &V_{\bxi}^{\pi,\rho}(x_1)\\
        &=\frac{\delta}{2d}\EE^{\pi}\Big\{\Big[d+\sum_{i=1}^d\xi_{1i}a_{1i} \Big]+ (1-\rho)\Big[d+\sum_{i=1}^d\xi_{2i}a_{2i} \Big]\cdots+(1-\rho)\Big[d+\sum_{i=1}^d\xi_{Hi}a_{Hi} \Big] \Big\}.
    \end{align*}
    Then we have
    \begin{align}
        &V_{\bxi}^{\star,\rho}(x_1) - V_{\bxi}^{\pi,\rho}(x_1)\notag\\
        & = \frac{\delta}{2d}\Big\{\Big[\sum_{i=1}^d\frac{1+\xi_{1i}}{2}-\xi_{1i}\EE^{\pi}a_{1i} \Big]+ (1-\rho)\sum_{h=2}^H\sum_{i=1}^d\Big(\frac{1+\xi_{hi}}{2}-\xi_{hi}\EE^{\pi}a_{hi}\Big) \Big\}\notag\\
        &\geq \frac{\delta}{2d}(1-\rho)\sum_{h=1}^H\sum_{i=1}^d\Big(\frac{1+\xi_{hi}}{2}-\xi_{hi}\EE^{\pi}a_{hi}\Big) \notag\\
        &=\frac{\delta}{2d}(1-\rho)\sum_{h=1}^H\sum_{i=1}^d\Big(\frac{1}{2}+\xi_{hi}\EE^{\pi}\Big(\frac{1}{2}-a_{hi}\Big)\Big)\notag\\
        &=\frac{\delta}{4d}(1-\rho)\sum_{h=1}^H\sum_{i=1}^d(1-\xi_{hi}\EE^{\pi}(2a_{hi}-1)).\label{eq:reduction to testing-difference of V}
    \end{align} 
Note that for any $(h,i)\in[H]\times[d]$, by design we have $1=\xi_{hi}^2$, thus
\begin{align}
    \frac{\delta}{4d}(1-\rho)\sum_{h=1}^H\sum_{i=1}^d(1-\xi_{hi}\EE^{\pi}(2a_{hi}-1))&=\frac{\delta}{4d}(1-\rho)\sum_{h=1}^H\sum_{i=1}^d(\xi_{hi}-\EE^{\pi}(2a_{hi}-1))\xi_{hi}\notag\\
    &=\frac{\delta}{4d}(1-\rho)\sum_{h=1}^H\sum_{i=1}^d|\xi_{hi}-\EE^{\pi}(2a_{hi}-1)|, \label{eq:reduction to testing-range of E(2a-1)}
\end{align}
where \eqref{eq:reduction to testing-range of E(2a-1)} holds due to the fact that $\EE^{\pi}(2a_{hi}-1)\in[-1,1]$. To continue, we have
\begin{align}
    &\frac{\delta}{4d}(1-\rho)\sum_{h=1}^H\sum_{i=1}^d|\xi_{hi}-\EE^{\pi}(2a_{hi}-1)|\notag\\
    &\geq  \frac{\delta}{4d}(1-\rho)\sum_{h=1}^H\sum_{i=1}^d|\xi_{hi}-\EE^{\pi}(2a_{hi}-1)|\ind\{\xi_{hi}\neq \sign(\EE^{\pi}(2a_{h,i}-1))\}\notag\\
    &\geq \frac{\delta}{4d}(1-\rho)\sum_{h=1}^H\sum_{i=1}^d\ind\{\xi_{hi}\neq \sign(\EE^{\pi}(2a_{h,i}-1))\}\notag \\
    &\geq \frac{\delta}{4d}(1-\rho)D_H(\bxi,\bxi^{\pi}),
        \label{eq:reduction to testing-lower bound by Hamming distance}
\end{align}
    where $D_H(\cdot,\cdot)$ is the Hamming distance, $\bxi^{\pi}=\{\bxi_h^{\pi}\}_{h\in[H]}$, and $\xi_{hi}^{\pi}:=\sign(\EE^{\pi}(2a_{hi}-1)), \forall i\in[d]$. Combining \eqref{eq:reduction to testing-difference of V}, \eqref{eq:reduction to testing-range of E(2a-1)}, \eqref{eq:reduction to testing-lower bound by Hamming distance} and the definition of the suboptimality gap, we have
    \begin{align}
    \label{eq:reduction to testing-lower bound the suboptimality by Hamming distance}
        \text{SupOpt}(M_{\bxi}, x_1,\pi, \rho) \geq \frac{\delta}{4d}(1-\rho)D_H(\bxi,\bxi^{\pi}).
    \end{align}
\paragraph{Step 2: lower bound the hamming distance by testing error.}
Applying Assouad's method \cite[Lemma 2.12]{tsybakov2009nonparametric}, we have
\begin{align}
\label{eq:reduction to testing-lower bound hamming distance by testing error}
    \inf_{\pi}\sup_{\bxi\in\Omega}\EE_{\bxi}\big[D_H(\bxi, \bxi') \big]\geq \frac{dH}{2}\min_{\substack{\bxi, \bxi'\in\Omega\\D_H(\bxi,\bxi')=1}}\inf_{\psi}\Big[\QQ_{\bxi}(\psi(\cD)\neq\bxi) + \QQ_{\bxi'}(\psi(\cD)\neq\bxi') \Big],
\end{align}
where $\inf_\psi$ denotes the infimum over all test functions taking values in $\{\bxi, \bxi'\}$. 
We conclude the proof by combining \eqref{eq:reduction to testing-lower bound the suboptimality by Hamming distance} and \eqref{eq:reduction to testing-lower bound hamming distance by testing error}.
\end{proof}

\subsection{Proof of \Cref{lemma:Lower bound on testing error}}
\label{sec:proof of lower bound on testing error}
\begin{proof}
    By the Theorem 2.12 in \cite{tsybakov2009nonparametric}, we lower bound the testing error as follows
    \begin{align*}
        &\min_{\bxi, \bxi': D_H(\bxi,\bxi')=1}\inf_{\psi}\Big[\QQ_{\bxi}(\psi(\cD)\neq\bxi) + \QQ_{\bxi'}(\psi(\cD)\neq\bxi') \Big]\\
        &\geq 1-\Big(\frac{1}{2}\max_{\bxi, \bxi': D_H(\bxi,\bxi')=1}D_{\text{KL}}\big(\QQ_{\bxi}||\QQ_{\bxi'}\big)\Big)^{1/2},
    \end{align*}
    where $D_{\text{KL}}(\cdot||\cdot)$ is the Kullback-Leibler divergence.
    Then it remains to bound $D_{\text{KL}}\big(\QQ_{\bxi}||\QQ_{\bxi'}\big)$. According to the definition of $\QQ_{\bxi}(\cD)$, we have
    \begin{align*}
        \QQ_{\bxi}(\cD)=\prod_{k=1}^K\prod_{h=1}^H\pi^b_h(a^k_h|s^k_h)P_h(s^k_{h+1}|s^k_h,a^k_h)R(s^k_h,a^k_h;r_h^k),
    \end{align*}
    where $R(s^k_h,a^k_h;r_h^k)$ is the density function of $\cN(r_h(s^k_h,a^k_h),1)$ at $r_h^k$.
    Note that the difference between the two distribution $\QQ_{\bxi}(\cD)$ and $\QQ_{\bxi'}(\cD)$ lies only in the reward distribution corresponding to the index where $\bxi$ and $\bxi'$ differ. 
    Then, by the chain rule of Kullback-Leibler divergence, we have
    \begin{align*}
        D_{\text{KL}}\big(\QQ_{\bxi}(\cD)||\QQ_{\bxi'}(\cD)\big) =  \sum_{k=1}^{\frac{K}{d+2}}D_{\text{KL}}\Big(\cN\Big(\frac{d+1}{2d}\delta,1\Big)\Big|\Big|\cN\Big(\frac{d-1}{2d}\delta,1\Big)\Big)=\frac{K}{d+2}\frac{\delta^2}{d^2}.
    \end{align*}
    Then by our choice of $\delta$, we have
    \begin{align*}
        \min_{\bxi, \bxi':D_H(\bxi,\bxi')=1}\inf_{\psi}\Big[\QQ_{\bxi}(\psi(\cD)\neq\bxi) + \QQ_{\bxi'}(\psi(\cD)\neq\bxi') \Big]
        &\geq 1-\Big(\frac{K\delta^2}{2(d+2)d^2}\Big)^{1/2}\\
        &\geq 1-\Big(\frac{K\delta^2}{2d^3}\Big)^{1/2}\\
        &=\frac{1}{2}.
    \end{align*}
    This completes the proof.
\end{proof}

\subsection{Proof of \Cref{lemma:upper bound on the summation}}
\label{sec:proof of upper bound on the summation}
\begin{proof}
    Recall that 
    \begin{align*}
        \bSigma^{\star-1}_h=\sum_{k=1}^K\frac{\bphi_h^\tau\bphi_h^{\tau\top}}{[\VV_hV_h^{\star,\rho}](s_h^\tau, a_h^\tau)}+\lambda I.
    \end{align*}
We first show that with sufficiently large $K$, the clipped conditional variances of the optimal robust value functions are always 1. Note that $V_h^{\star,\rho}(x_2)=0, \forall h\in[H]$, and 
\begin{align*}
    V_H^{\star,\rho}(x_1)&=\frac{\delta}{2d}\Big(\sum_{i=1}^d\frac{1+\xi_{Hi}}{2}+d\Big)\leq\delta,\\
    V_{H-1}^{\star,\rho}(x_1) &= \frac{\delta}{2d}\Big(\sum_{i=1}^d\frac{1+\xi_{H-1i}}{2}+d\Big) + V_H^{\star,\rho}(x_1)\leq 2\delta,\\
    &\cdots\\
    V_2^{\star,\rho}(x_1)&=\frac{\delta}{2d}\Big(\sum_{i=1}^d\frac{1+\xi_{2i}}{2}+d\Big) + V_3^{\star, \rho}(x_1)\leq (H-1)\cdot\delta.
\end{align*}
Then, when $K\geq \Omega(H^2d^3)$, we have 
\begin{align*}
    \big[\Var_1V_2^{\star, \rho}\big](x_1,a) = \big[\PP_1^0(V_2^{\star, \rho})^2\big](x_1,a) - \big(\big[\PP^0_1(V_2^{\star, \rho})^2\big](x_1,a)\big)^2\leq (1-\rho)\rho H^2\delta^2\leq 1,
\end{align*}
and by design we have,
\begin{align*}
    [\Var_1V_{2}^{\star,\rho}](x_2,a)=0~ \text{and} ~ [\Var_h V_{h+1}^{\star,\rho}](s,a)=0, \forall (s,a,h)\in\cS\times\cA\times[H]/\{1\}.
\end{align*}
Thus, we have $[\VV_hV_h^{\star,\rho}](s_h^\tau, a_h^\tau)= 1$, which implies 
\begin{align}
\label{eq:upper bound on the summation-two matrix equal}
    \bSigma_h^\star = \bLambda_h.
\end{align}
 Define 
\begin{align*}
\tilde{\bLambda}_h=\EE^{\pi^b,P^0}[\bphi(s_h,a_h)\bphi(s_h,a_h)^{\top}],
\end{align*}
then by definition we have
\begin{align*}
    \tilde{\bLambda}_h &= \frac{1}{d+2}\begin{bmatrix}
  \frac{1}{d^2}& 0 & \cdots &0 &\frac{1}{d}(1-\frac{1}{d}) & 0 \\
  0 &  0 & \cdots &0 & 0 & 0\\ 
  \vdots & \vdots & &\vdots &\vdots &\vdots \\
  0 & 0 & \cdots & 0 &0 &0\\
  \frac{1}{d}(1-\frac{1}{d}) & 0 & \cdots &0 & (1-\frac{1}{d})^2 & 0\\
  0 & 0 &\cdots & 0 & 0 & 0
\end{bmatrix}
+ \frac{1}{d+2}\begin{bmatrix}
 0& 0& \cdots  &0 & 0 & 0 \\
  0 &  \frac{1}{d^2} & \cdots & 0& \frac{1}{d}(1-\frac{1}{d})& 0   \\ 
  \vdots & \vdots & &\vdots &\vdots&\vdots \\
  0 & 0 & \cdots & 0 &0 & 0\\
  0 & \frac{1}{d}(1-\frac{1}{d}) & \cdots & 0& (1-\frac{1}{d})^2 & 0\\
  0 & 0 &\cdots & 0 & 0 & 0
\end{bmatrix}
\\
&\quad  + \cdots+\frac{1}{d+2}\begin{bmatrix}
 0& 0& \cdots & 0 & 0 &0\\
  0 &  0 & \cdots & 0 & 0  &0 \\ 
  \vdots & \vdots & &\vdots &\vdots & \vdots \\
  0 & 0 & \cdots  & \frac{1}{d^2} & \frac{1}{d}(1-\frac{1}{d}) & 0\\
  0 & 0 & \cdots &\frac{1}{d}(1-\frac{1}{d})& (1-\frac{1}{d})^2 & 0\\
  0 & 0 &\cdots & 0 & 0 & 0
\end{bmatrix}
+\frac{1}{d+2}\begin{bmatrix}
 0& 0& \cdots & 0 & 0 &0\\
  0 &  0 & \cdots & 0 & 0  &0 \\ 
  \vdots & \vdots & &\vdots &\vdots & \vdots \\
  0 & 0 & \cdots  & 0 & 0 & 0\\
  0 & 0 & \cdots &0& 1 & 0\\
  0 & 0 &\cdots & 0 & 0 & 0
\end{bmatrix}  \\
&\quad +\frac{1}{d+2}\begin{bmatrix}
 0& 0& \cdots & 0 & 0 &0\\
  0 &  0 & \cdots & 0 & 0  &0 \\ 
  \vdots & \vdots & &\vdots &\vdots & \vdots \\
  0 & 0 & \cdots  & 0 & 0 & 0\\
  0 & 0 & \cdots &0& 0 & 0\\
  0 & 0 &\cdots & 0 & 0 &1
\end{bmatrix} \\
&=\frac{d}{d+2}\begin{bmatrix}
 \frac{1}{d^3}& 0& \cdots & 0 & \frac{1}{d^2}(1-\frac{1}{d}) &0\\
  0 & \frac{1}{d^3} & \cdots & 0 & \frac{1}{d^2}(1-\frac{1}{d})  &0 \\ 
  \vdots & \vdots & &\vdots &\vdots & \vdots \\
  0 & 0 & \cdots  & \frac{1}{d^3} & \frac{1}{d^2}(1-\frac{1}{d}) & 0\\
  \frac{1}{d^2}(1-\frac{1}{d}) & \frac{1}{d^2}(1-\frac{1}{d}) & \cdots &\frac{1}{d^2}(1-\frac{1}{d})& (1-\frac{1}{d})^2+\frac{1}{d} & 0\\
  0 & 0 &\cdots & 0 & 0 & \frac{1}{d}
\end{bmatrix} .
\end{align*}
Denote
\begin{align*}
    D=\begin{bmatrix}
 \frac{1}{d^3}& 0& \cdots & 0 & \frac{1}{d^2}(1-\frac{1}{d})\\
  0 & \frac{1}{d^3} & \cdots & 0 & \frac{1}{d^2}(1-\frac{1}{d})   \\ 
  \vdots & \vdots & &\vdots &\vdots  \\
  0 & 0 & \cdots  & \frac{1}{d^3} & \frac{1}{d^2}(1-\frac{1}{d})\\
  \frac{1}{d^2}(1-\frac{1}{d}) & \frac{1}{d^2}(1-\frac{1}{d}) & \cdots &\frac{1}{d^2}(1-\frac{1}{d})& (1-\frac{1}{d})^2+\frac{1}{d}
\end{bmatrix} 
,
\end{align*}
then by Gaussian elimination, we have
\begin{align*}
    D^{-1}=\begin{bmatrix}
 2d^3-2d^2+d& d^3-2d^2+d& \cdots &  d^3-2d^2+d & d-d^2\\
  d^3-2d^2+d & 2d^3-2d^2+d & \cdots & d^3-2d^2+d & d-d^2   \\ 
  \vdots & \vdots & &\vdots &\vdots  \\
  d^3-2d^2+d & d^3-2d^2+d & \cdots  & 2d^3-2d^2+d & d-d^2\\
  d-d^2 & d-d^2 & \cdots &d-d^2& d
\end{bmatrix}.
\end{align*}
Note that
\begin{align*}
     \tilde{\Lambda}_h = \frac{d}{d+2}\begin{bmatrix}
D & 0\\
0 & \frac{1}{d}
\end{bmatrix},
\end{align*}
then we have
\begin{align*}
     \tilde{\bLambda}_h^{-1} = \frac{d+2}{d}\begin{bmatrix}
D^{-1} & 0\\
0 & d
\end{bmatrix}.
\end{align*}
Note that  
$\lambda_{\min}(D)=O(1/d^3)$, thus $\|\tilde{\bLambda}_h^{-1}\|=O(d^3)$. Then when $K>\tilde{O}(d^6)$, for any $(s,a,i, h)\in \cS\times\cA\times[d]\times[H]$, with probability at least $1-\delta$, we have
\begin{align}
    \label{eq:upper bound on the summation-matrix concentration} \|\phi_i(s,a)\mathbf{1}_i\|_{\bLambda_h^{-1}}\leq \frac{2}{\sqrt{K}}\|\phi_i(s,a)\mathbf{1}_i\|_{\tilde{\bLambda}_h^{-1}}.
\end{align}
With this in mind, we have
\begin{align}
    &\sup_{P\in\cU^{\rho}(P^0)}\sum_{h=1}^H\EE^{\pi^{\star}, P}\Big[\sum_{i=1}^d\|\phi_i(s,a)\mathbf{1}_i\|_{\bSigma_h^{\star-1}}|s_1=x_1\Big]\notag\\
    &= \sup_{P\in\cU^{\rho}(P^0)}\sum_{h=1}^H\EE^{\pi^{\star}, P}\Big[\sum_{i=1}^d\|\phi_i(s_h,a_h)\mathbf{1}_i\|_{\bLambda_h^{-1}}|s_1=x_1\Big] \notag \\
    &\leq \sup_{P\in\cU^{\rho}(P^0)}\sum_{h=1}^H\EE^{\pi^{\star}, P}\Big[\frac{2}{\sqrt{K}}\sum_{i=1}^d\|\phi_i(s_h,a_h)\mathbf{1}_i\|_{\tilde{\bLambda}_h^{-1}}|s_1=x_1\Big]\label{eq:upper bound on the summation-use matrix concentration}\\
    & = \sup_{P\in\cU^{\rho}(P^0)}\sum_{h=1}^H\EE^{\pi^{\star}, P}\Big[\frac{2}{\sqrt{K}}\sum_{i=1}^d\phi_i(s_h,a_h)\big(\tilde{\bLambda}_h^{-1}\big)_{ii}^{1/2}|s_1=x_1 \Big]\notag \\
    &\leq \frac{4Hd^{3/2}}{\sqrt{K}},\notag
\end{align}
where $\eqref{eq:upper bound on the summation-use matrix concentration}$ is due to \eqref{eq:upper bound on the summation-matrix concentration}.
This concludes the proof.
\end{proof}

\section{Proof of Supporting Lemmas}
\subsection{Proof of \Cref{lemma:Reference-Advantage Decomposition}}
To prove \Cref{lemma:Reference-Advantage Decomposition}, we need the following proposition on the dual formulation under the TV uncertainty set.

\begin{proposition}
\label{prop:strong duality for TV}
    (Strong duality for TV \citep[Lemma 4]{shi2024curious}). Given any probability measure $\mu^0$ over $\cS$, a fixed uncertainty level $\rho$, the uncertainty set $ \cU^{\rho}(\mu^0) =\{\mu: \mu\in \Delta(\cS), D_{TV}(\mu||\mu^0)\leq \rho\}$, and any function $V:\cS \rightarrow [0,H]$, we obtain 
    \begin{align}\label{eq:duality}
        & \inf_{\mu\in\cU^{\rho}(\mu^0)}\EE_{s\sim\mu}V(s) = \max_{\alpha \in [V_{\min}, V_{\max}]}\big\{\EE_{s\sim \mu^0}[V(s)]_{\alpha}-\rho\big(\alpha - \min_{s'}[V(s')]_{\alpha}\big) \big\},
    \end{align}
    where $[V(s)]_{\alpha}=\min\{V(s), \alpha\}$, $V_{\min}=\min_{s}V(s)$ and $V_{\max}=\max_{s}V(s)$. Notably, the range of $\alpha$ can be relaxed to $[0,H]$ without impacting the optimization. 
\end{proposition}

\begin{proof}[Proof of \Cref{lemma:Reference-Advantage Decomposition}]
By \Cref{assumption:linear MDP} and \Cref{prop:strong duality for TV}, we have
    \begin{align*}
        &\inf_{P_h(\cdot|s,a)\in\cU_h^{\rho}(s,a;\bmu_{h,i}^0)}[\PP_h\widehat{V}_{h+1}^{\rho}](s,a)-\widehat{\inf_{P_h(\cdot|s,a)\in\cU_h^{\rho}(s,a;\bmu_{h,i}^0)}}[\PP_h\widehat{V}_{h+1}^{\rho}](s,a)\\
        &=\sum_{i=1}^d\phi_i(s,a)\Big[\max_{\alpha\in[0,H]}\{\EE^{\mu_{h,i}^0}[\widehat{V}^{\rho}_{h+1}(s)]_{\alpha}-\rho(\alpha-\min_{s'}[\widehat{V}_{h+1}^{\rho}(s')]_{\alpha})\}  \\
        &\quad - \max_{\alpha\in[0,H]}\{\widehat{\EE}^{\mu_{h,i}^0}[\widehat{V}^{\rho}_{h+1}(s)]_{\alpha} - \rho(\alpha-\min_{s'}[\widehat{V}_{h+1}^{\rho}(s')]_{\alpha})\} \Big].
    \end{align*}
    Denote $\alpha_i=\argmax_{\alpha\in[0,H]}\{\EE^{\mu_{h,i}^0}[\widehat{V}^{\rho}_{h+1}(s)]_{\alpha}-\rho(\alpha-\min_{s'}[\widehat{V}_{h+1}^{\rho}(s')]_{\alpha})\}$, then we have
    \begin{align*}
        &\inf_{P_h(\cdot|s,a)\in\cU_h^{\rho}(s,a;\bmu_{h,i}^0)}[\PP_h\widehat{V}_{h+1}^{\rho}](s,a)-\widehat{\inf_{P_h(\cdot|s,a)\in\cU_h^{\rho}(s,a;\bmu_{h,i}^0)}}[\PP_h\widehat{V}_{h+1}^{\rho}](s,a)\\
        &\leq  \sum_{i=1}^d\phi_i(s,a) \big(\EE^{\mu_{h,i}^0}[\widehat{V}^{\rho}_{h+1}(s)]_{\alpha_i} - \widehat{\EE}^{\mu_{h,i}^0}[\widehat{V}^{\rho}_{h+1}(s)]_{\alpha_i} \big)\\
        &=  \sum_{i=1}^d\phi_i(s,a) \big[\mathbf{1}_i^\top\EE^{\bmu_{h}^0}[\widehat{V}^{\rho}_{h+1}(s)]_{\alpha_i} - \mathbf{1}_i^\top\widehat{\EE}^{\bmu_{h}^0}[\widehat{V}^{\rho}_{h+1}(s)]_{\alpha_i}\big].
    \end{align*}
Here we do reference-advantage decomposition by using the optimal robust value function as the reference function. Specifically, we have
\begin{align}
    &\inf_{P_h(\cdot|s,a)\in\cU_h^{\rho}(s,a;\bmu_{h,i}^0)}[\PP_h\widehat{V}_{h+1}^{\rho}](s,a)-\widehat{\inf_{P_h(\cdot|s,a)\in\cU_h^{\rho}(s,a;\bmu_{h,i}^0)}}[\PP_h\widehat{V}_{h+1}^{\rho}](s,a)\notag\\
        &\leq\sum_{i=1}^d\phi_i(s,a) \big[\mathbf{1}_i^\top\big(\EE^{\bmu_{h}^0}\big[[\widehat{V}^{\rho}_{h+1}(s)]_{\alpha_i} - [V^{\star,\rho}_{h+1}(s)]_{\alpha_i} + [V^{\star,\rho}_{h+1}(s)]_{\alpha_i}\big] \big)  \notag \\
        &\quad -\mathbf{1}_i^\top\big(\widehat{\EE}^{\bmu_{h}^0}\big[[\widehat{V}^{\rho}_{h+1}(s)]_{\alpha_i} - [V^{\star,\rho}_{h+1}(s)]_{\alpha_i} + [V^{\star,\rho}_{h+1}(s)]_{\alpha_i} \big]  \big) \big] \notag\\
        &=\underbrace{\sum_{i=1}^d\phi_i(s,a)\mathbf{1}_i^{\top}\big(\EE^{\bmu_{h}^0}[V^{\star,\rho}_{h+1}(s)]_{\alpha_i} - \widehat{\EE}^{\bmu_{h}^0}[V^{\star,\rho}_{h+1}(s)]_{\alpha_i} \big)}_\text{reference uncertainty} \notag \\
        &\quad +\underbrace{\sum_{i=1}^d\phi_i(s,a)\mathbf{1}_i^{\top}\big(\EE^{\bmu_{h}^0}\big[[\widehat{V}^{\rho}_{h+1}(s)]_{\alpha_i} - [V^{\star,\rho}_{h+1}(s)]_{\alpha_i} \big] - \widehat{\EE}^{\bmu_{h}^0}\big[[\widehat{V}^{\rho}_{h+1}(s)]_{\alpha_i} - [V^{\star,\rho}_{h+1}(s)]_{\alpha_i} \big] \big)}_\text{advantage uncertainty}.\label{eq:reference advantage decomposition-reference advantage decomposition}
    \end{align}
    \paragraph{The Reference Uncertainty.} First, we bound the reference uncertainty. Specifically, we have
    \begin{align*}
        &\sum_{i=1}^d\phi_i(s,a)\mathbf{1}_i^{\top}\big(\EE^{\bmu_{h}^0}[V^{\star,\rho}_{h+1}(s)]_{\alpha_i} - \widehat{\EE}^{\bmu_{h}^0}[V^{\star,\rho}_{h+1}(s)]_{\alpha_i} \big)\\
        &=\sum_{i=1}^d\phi_i(s,a)\mathbf{1}_i^{\top}\Big(\EE^{\bmu_{h}^0}[V^{\star,\rho}_{h+1}(s)]_{\alpha_i} - \bLambda_h^{-1}\sum_{\tau=1}^K\bphi_h^{\tau}\big[\PP_h^0[V_{h+1}^{\star, \rho}]_{\alpha_i} \big](s_h^{\tau}, a_h^{\tau}) \\
        & \quad  +\bLambda_h^{-1}\sum_{\tau=1}^K\bphi_h^{\tau}\big[\PP_h^0[V_{h+1}^{\star, \rho}]_{\alpha_i} \big](s_h^{\tau}, a_h^{\tau}) - \bLambda_h^{-1}\sum_{\tau=1}^K\bphi_h^{\tau}[V^{\star,\rho}_{h+1}(s_{h+1}^{\tau})]_{\alpha_i} \Big)\\
        &=\sum_{i=1}^d\phi_i(s,a)\mathbf{1}_i^{\top}\Big(\EE^{\bmu_{h}^0}[V^{\star,\rho}_{h+1}(s)]_{\alpha_i} - \bLambda_h^{-1}\sum_{\tau=1}^K\bphi_h^{\tau}\bphi_h^{\tau\top}\EE^{\bmu_{h}^0}[V^{\star,\rho}_{h+1}(s)]_{\alpha_i} \\ &\quad +\bLambda_h^{-1}\sum_{\tau=1}^K\bphi_h^{\tau}\big(\big[\PP_h^0[V_{h+1}^{\star, \rho}]_{\alpha_i} \big](s_h^{\tau}, a_h^{\tau}) - [V^{\star,\rho}_{h+1}(s_{h+1}^{\tau})]_{\alpha_i} \big) \Big).\\
    \end{align*}
    For any function $f:\cS\rightarrow [0,H-1]$, we define $\eta_h^{\tau}([f]_{\alpha_i}) = \big(\big[\PP_h^0[f]_{\alpha_i}\big](s_h^{\tau},a_h^{\tau})-[f(s_{h+1}^{\tau})]_{\alpha_i} \big)$. Then, we have
    \begin{align}       &\sum_{i=1}^d\phi_i(s,a)\mathbf{1}_i^{\top}\big(\EE^{\bmu_{h}^0}[V^{\star,\rho}_{h+1}(s)]_{\alpha_i} - \widehat{\EE}^{\bmu_{h}^0}[V^{\star,\rho}_{h+1}(s)]_{\alpha_i} \big)\notag \\         &=\lambda\sum_{i=1}^d\phi_i(s,a)\mathbf{1}_i^{\top}\bLambda_h^{-1}\EE^{\bmu_{h}^0}[V^{\star,\rho}_{h+1}(s)]_{\alpha_i} + \sum_{i=1}^d\phi_i(s,a)\mathbf{1}_i^{\top}\bLambda_h^{-1}\sum_{\tau=1}^K\bphi_h^\tau\eta_h^\tau([V_{h+1}^{\star, \rho}]_{\alpha_i})\notag \\ 
        &\leq \underbrace{\lambda\sum_{i=1}^d\|\phi_i(s,a)\mathbf{1}_i\|_{\bLambda_h^{-1}}\|\EE^{\bmu_{h}^0}[V^{\star,\rho}_{h+1}(s)]_{\alpha_i}\|_{\bLambda_h^{-1}}}_\text{i}+\underbrace{\sum_{i=1}^d\|\phi_i(s,a)\mathbf{1}_i\|_{\bLambda_h^{-1}}\Big\|\sum_{\tau=1}^K\bphi_h^{\tau}\eta_h^{\tau}([V_{h+1}^{\star, \rho}]_{\alpha_i})\Big\|_{\bLambda_h^{-1}}}_\text{ii}. \label{eq:reference advantage decomposition-bound of the reference uncertainty}
    \end{align}
    \paragraph{The Advantage Uncertainty.} Next, we bound the advantage uncertainty. By similar argument in bounding the reference uncertainty, we have
    \begin{align}
    &\sum_{i=1}^d\phi_i(s,a)\mathbf{1}_i^{\top}\big(\EE^{\bmu_{h}^0}\big[[\widehat{V}^{\rho}_{h+1}(s)]_{\alpha_i} - [V^{\star,\rho}_{h+1}(s)]_{\alpha_i} \big] - \widehat{\EE}^{\bmu_{h}^0}\big[[\widehat{V}^{\rho}_{h+1}(s)]_{\alpha_i} - [V^{\star,\rho}_{h+1}(s)]_{\alpha_i} \big] \big)\notag\\
    &\leq \underbrace{\lambda\sum_{i=1}^d\|\phi_i(s,a)\mathbf{1}_i\|_{\bLambda_h^{-1}}\Big\|\EE^{\bmu_{h}^0}\big[[\widehat{V}^{\rho}_{h+1}(s)]_{\alpha_i} - [V^{\star,\rho}_{h+1}(s)]_{\alpha_i} \big]\Big\|_{\bLambda_h^{-1}}}_\text{iii}\notag\\
    &\quad+\underbrace{\sum_{i=1}^d\|\phi_i(s,a)\mathbf{1}_i\|_{\bLambda_h^{-1}}\Big\|\sum_{\tau=1}^K\bphi_h^{\tau}\eta_h^{\tau}([\widehat{V}^{\rho}_{h+1}(s)]_{\alpha_i} - [V^{\star,\rho}_{h+1}(s)]_{\alpha_i})\Big\|_{\bLambda_h^{-1}}}_\text{iv}.\label{eq:reference advantage decomposition-bound of the advantage uncertainty}
    \end{align}
Combining \eqref{eq:reference advantage decomposition-reference advantage decomposition}, \eqref{eq:reference advantage decomposition-bound of the reference uncertainty} and \eqref{eq:reference advantage decomposition-bound of the advantage uncertainty}, we have
\begin{align*}
    &\inf_{P_h(\cdot|s,a)\in\cU_h^{\rho}(s,a;\bmu_{h,i}^0)}[\PP_h\widehat{V}_{h+1}^{\rho}](s,a)-\widehat{\inf_{P_h(\cdot|s,a)\in\cU_h^{\rho}(s,a;\bmu_{h,i}^0)}}[\PP_h\widehat{V}_{h+1}^{\rho}](s,a)\\
        &\leq \underbrace{\lambda\sum_{i=1}^d\|\phi_i(s,a)\mathbf{1}_i\|_{\bLambda_h^{-1}}\|\EE^{\bmu_{h}^0}[V^{\star,\rho}_{h+1}(s)]_{\alpha_i}\|_{\bLambda_h^{-1}}}_\text{i}+\underbrace{\sum_{i=1}^d\|\phi_i(s,a)\mathbf{1}_i\|_{\bLambda_h^{-1}}\Big\|\sum_{\tau=1}^K\bphi_h^{\tau}\eta_h^{\tau}([V_{h+1}^{\star, \rho}]_{\alpha_i})\Big\|_{\bLambda_h^{-1}}}_\text{ii}  \\
        &\quad +\underbrace{\lambda\sum_{i=1}^d\|\phi_i(s,a)\mathbf{1}_i\|_{\bLambda_h^{-1}}\Big\|\EE^{\bmu_{h}^0}\big[[\widehat{V}^{\rho}_{h+1}(s)]_{\alpha_i} - [V^{\star,\rho}_{h+1}(s)]_{\alpha_i} \big]\Big\|_{\bLambda_h^{-1}}}_\text{iii} \\
&\quad + \underbrace{\sum_{i=1}^d\|\phi_i(s,a)\mathbf{1}_i\|_{\bLambda_h^{-1}}\Big\|\sum_{\tau=1}^K\bphi_h^{\tau}\eta_h^{\tau}([\widehat{V}^{\rho}_{h+1}(s)]_{\alpha_i} - [V^{\star,\rho}_{h+1}(s)]_{\alpha_i})\Big\|_{\bLambda_h^{-1}}}_\text{iv}.
\end{align*}
On the other hand, we can similarly deduce
\begin{align*}
    &\widehat{\inf_{P_h(\cdot|s,a)\in\cU_h^{\rho}(s,a;\bmu_{h,i}^0)}}[\PP_h\widehat{V}_{h+1}^{\rho}](s,a) - \inf_{P_h(\cdot|s,a)\in\cU_h^{\rho}(s,a;\bmu_{h,i}^0)}[\PP_h\widehat{V}_{h+1}^{\rho}](s,a)\\
        &\leq \underbrace{\lambda\sum_{i=1}^d\|\phi_i(s,a)\mathbf{1}_i\|_{\bLambda_h^{-1}}\|\EE^{\bmu_{h}^0}[V^{\star,\rho}_{h+1}(s)]_{\alpha_i'}\|_{\bLambda_h^{-1}}}_\text{i}+\underbrace{\sum_{i=1}^d\|\phi_i(s,a)\mathbf{1}_i\|_{\bLambda_h^{-1}}\Big\|\sum_{\tau=1}^K\bphi_h^{\tau}\eta_h^{\tau}([V_{h+1}^{\star, \rho}]_{\alpha_i'})\Big\|_{\bLambda_h^{-1}}}_\text{ii} \\
        &\quad +\underbrace{\lambda\sum_{i=1}^d\|\phi_i(s,a)\mathbf{1}_i\|_{\bLambda_h^{-1}}\Big\|\EE^{\bmu_{h}^0}\big[[\widehat{V}^{\rho}_{h+1}(s)]_{\alpha_i'} - [V^{\star,\rho}_{h+1}(s)]_{\alpha_i'} \big]\Big\|_{\bLambda_h^{-1}}}_\text{iii} \\
&\quad+\underbrace{\sum_{i=1}^d\|\phi_i(s,a)\mathbf{1}_i\|_{\bLambda_h^{-1}}\Big\|\sum_{\tau=1}^K\bphi_h^{\tau}\eta_h^{\tau}([\widehat{V}^{\rho}_{h+1}(s)]_{\alpha_i'} - [V^{\star,\rho}_{h+1}(s)]_{\alpha_i'})\Big\|_{\bLambda_h^{-1}}}_\text{iv},
\end{align*}
where $\alpha_i'=\argmax_{\alpha\in[0,H]}\{\widehat{\EE}^{\mu_{h,i}^0}[\widehat{V}^{\rho}_{h+1}(s)]_{\alpha}-\rho(\alpha-\min_{s'}[\widehat{V}_{h+1}^{\rho}(s')]_{\alpha})\}$.
Then for all $i\in[d]$, there exist $\tilde{\alpha}_i\in\{\alpha_i, \alpha_i'\}$, such that 
\begin{align*}
    &\Big|\inf_{P_h(\cdot|s,a)\in\cU_h^{\rho}(s,a;\bmu_{h,i}^0)}[\PP_h\widehat{V}_{h+1}^{\rho}](s,a)-\widehat{\inf_{P_h(\cdot|s,a)\in\cU_h^{\rho}(s,a;\bmu_{h,i}^0)}}[\PP_h\widehat{V}_{h+1}^{\rho}](s,a)\Big|\\
        &\leq \underbrace{\lambda\sum_{i=1}^d\|\phi_i(s,a)\mathbf{1}_i\|_{\bLambda_h^{-1}}\|\EE^{\bmu_{h}^0}[V^{\star,\rho}_{h+1}(s)]_{\tilde{\alpha}_i}\|_{\bLambda_h^{-1}}}_\text{i}+\underbrace{\sum_{i=1}^d\|\phi_i(s,a)\mathbf{1}_i\|_{\bLambda_h^{-1}}\Big\|\sum_{\tau=1}^K\bphi_h^{\tau}\eta_h^{\tau}([V_{h+1}^{\star, \rho}]_{\tilde{\alpha}_i})\Big\|_{\bLambda_h^{-1}}}_\text{ii} \\
        &\quad +\underbrace{\lambda\sum_{i=1}^d\|\phi_i(s,a)\mathbf{1}_i\|_{\bLambda_h^{-1}}\Big\|\EE^{\bmu_{h}^0}\big[[\widehat{V}^{\rho}_{h+1}(s)]_{\tilde{\alpha}_i} - [V^{\star,\rho}_{h+1}(s)]_{\tilde{\alpha}_i} \big]\Big\|_{\bLambda_h^{-1}}}_\text{iii}\\
&\quad+\underbrace{\sum_{i=1}^d\|\phi_i(s,a)\mathbf{1}_i\|_{\bLambda_h^{-1}}\Big\|\sum_{\tau=1}^K\bphi_h^{\tau}\eta_h^{\tau}([\widehat{V}^{\rho}_{h+1}(s)]_{\tilde{\alpha}_i} - [V^{\star,\rho}_{h+1}(s)]_{\tilde{\alpha}_i})\Big\|_{\bLambda_h^{-1}}}_\text{iv},
\end{align*}
This concludes the proof.
\end{proof}

\subsection{Proof of \Cref{lemma:weight bound}}
The proof of \Cref{lemma:weight bound} will use the following fact.
\begin{lemma}\cite[Lemma D.1]{jin2020provably}
\label{lemma:self-normalize}
    Let $\bLambda_t=\lambda \mathbf{I} + \sum_{i=1}^t\bphi_i\bphi_i^{\top}$, where $\bphi_i\in\RR^d$ and $\lambda > 0$. Then:
    \begin{align*}
        \sum_{i=1}^t\bphi_i^{\top}(\bLambda_t)^{-1}\bphi_i \leq d.
    \end{align*}
\end{lemma}
\begin{proof}[Proof of \Cref{lemma:weight bound}]
The proof of \Cref{lemma:weight bound} is similar to that of Lemma E.1 in \cite{liu2024distributionally}.
Denote $\alpha_i = \argmax_{\alpha\in[0,H]} \{\hat{z}_{h,i}(\alpha)-\rho(\alpha-\min_{s'}[\widehat{V}_{h+1}^{\rho}(s')]_{\alpha})\}, i\in[d]$. 
    For any vector $\bv \in \RR^d$, we have 
    \begin{align}
        \big|\bv^{\top}\bw_h^{\rho}\big| &= \Big|\bv^{\top}\btheta_h + \bv^{\top} \Big[\max_{\alpha\in[0,H]}\{\hat{z}_{h,i}(\alpha)-\rho(\alpha-\min_{s'}[\widehat{V}_{h+1}^{\rho}(s')]_{\alpha})\} \Big]_{i\in [d]} \Big| \notag \\
        &\leq \big|\bv^{\top}\btheta_h\big| + \Big|\bv^{\top} \Big[\max_{\alpha\in[0,H]}\{\hat{z}_{h,i}(\alpha)-\rho(\alpha-\min_{s'}[\widehat{V}_{h+1}^{\rho}(s')]_{\alpha})\} \Big]_{i\in [d]} \Big| \notag \\
        &\leq \sqrt{d}\|\bv\|_2 + H\Vert\bv\Vert_1 + \bigg|\bv^{\top}\bigg[\mathbf{1}_i^\top\bigg(\bLambda_h^{-1}\sum_{\tau=1}^{K}\bphi_h^{\tau}[\max_a \widehat{Q}_{h+1}^{ \rho}(s_{h+1}^{\tau},a)]_{\alpha_i}\bigg)\bigg]_{i \in [d]}
        \bigg| \label{eq:weight bound-expand_z_h_i}\\
        &\leq \sqrt{d}\|\bv\|_2 + H\sqrt{d}\Vert \bv\Vert_2 + \sqrt{\bigg[ \sum_{\tau=1}^{K}\bv^{\top}\bLambda_h^{-1}\bv\bigg]\bigg[\sum_{\tau=1}^{K}(\bphi_h^{\tau})^{\top}\bLambda_h^{-1}\bphi_h^{\tau}\bigg]}\cdot H \label{eq:weight bound-use_C_S}\\
        &\leq 2H\Vert\bv\Vert_2\sqrt{dK/\lambda}\label{eq:weight bound-weight_bound}.
    \end{align}
    We note that the term $[(\bLambda_h^{-1}\sum_{\tau=1}^{K}\bphi_h^{\tau}[\max_a \widehat{Q}_{h+1}^{\rho}(s_{h+1}^{\tau},a)]_{\alpha_i})_{i}]_{i \in [d]}$ in \eqref{eq:weight bound-expand_z_h_i} is constructed by first taking out the $i$-th coordinate of the ridge solution vector, $\bLambda_h^{-1}\sum_{\tau=1}^{K}\bphi_h^{\tau}[\max_a \widehat{Q}_{h+1}^{\rho}(s_{h+1}^{\tau},a)]_{\alpha_i}\in\RR^d,~\forall i\in[d]$, and then concatenating all $d$ values into a vector.    
    Inequality \eqref{eq:weight bound-expand_z_h_i} is due to the fact that $\rho \leq 1$, \eqref{eq:weight bound-use_C_S} is due to the fact that $\widehat{Q}_{h+1}^{\rho} \leq H$, and   
    \eqref{eq:weight bound-weight_bound} is due to \Cref{lemma:self-normalize} with $t=K$ and the fact that the minimum eigenvalue of $\bLambda_h$ is lower bounded by $\lambda$. The remainder of the proof follows from the fact that $\Vert \bw_h^{\rho} \Vert_2 = \max_{\bv:\Vert\bv\Vert_2=1}|\bv^{\top}\bw_h^{\rho}| $.

\end{proof}

\subsection{Proof of \Cref{lemma:covering number of the function class V}}
The proof of \Cref{lemma:covering number of the function class V} will use the following fact.
\begin{lemma}
\label{lemma:Covering Number of Euclidean Ball}
    \cite[Covering Number of Euclidean Ball]{jin2020provably} For any $\epsilon>0$, the $\epsilon$-covering number of the Euclidean ball in $\RR^d$ with radius $R > 0$ is upper bounded by $(1 + 2R/\epsilon)^d$. 
\end{lemma}
\begin{proof}[Proof of \Cref{lemma:covering number of the function class V}]
    The proof is similar to the proof of Lemma E.3 in \cite{liu2024distributionally}. Denote $\bA=\beta^2\bSigma_h^{-1}$, so we have
    \begin{align}
    \label{eq:function class V_h}
        \cV_h(\cdot)=\max_{a\in\cA}\Big\{\bphi(s,a)^{\top}\btheta-\sum_{i=1}^d\sqrt{\phi_i(s,a)\mathbf{1}_i^{\top}\bA\phi_i(s,a)\mathbf{1}_i}\Big\}_{[0, H-h+1]},
    \end{align}
    for $\Vert \btheta\Vert \leq L$, $\|\bA\|\leq B^2\lambda^{-1}$. For any two functions $V_1, V_2 \in \cV$, let them take the form in \eqref{eq:function class V_h} with parameters $(\btheta_1, \bA_1)$ and $(\btheta_2,\bA_2)$, respectively. Then since both $\{\cdot\}_{[0,H-h+1]}$ and $\max_a$ are contraction maps, we have
    \begin{align}
        \dist(V_1, V_2)&\leq \sup_{x,a}\bigg|\bigg[\btheta_1^{\top}\bphi(x,a)-\sum_{i=1}^d\sqrt{\phi_i(x,a)\mathbf{1}_i^{\top}\bA_1\phi_i(x,a)\mathbf{1}_i} \bigg] \notag\\
        &\qquad - \bigg[\btheta_2^{\top}\bphi(x,a)-\sum_{i=1}^d\sqrt{\phi_i(x,a)\mathbf{1}_i^{\top}\bA_2\phi_i(x,a)\mathbf{1}_i} \bigg]\bigg|\notag\\
        &\leq \sup_{\bphi:\Vert\bphi\Vert\leq 1}\bigg|\bigg[\btheta_1^{\top}\bphi-\sum_{i=1}^d\sqrt{\phi_i\mathbf{1}_i^{\top}\bA_1\phi_i\mathbf{1}_i} \bigg] - \bigg[\btheta_2^{\top}\bphi-\sum_{i=1}^d\sqrt{\phi_i\mathbf{1}_i^{\top}\bA_2\phi_i\mathbf{1}_i} \bigg] \bigg|\notag\\
        &\leq \sup_{\bphi:\Vert\bphi\Vert\leq 1}\big|(\btheta_1-\btheta_2)^{\top}\bphi\big|  + \sup_{\bphi:\Vert\bphi\Vert\leq 1}\sum_{i=1}^d\sqrt{\phi_i\mathbf{1}_i^{\top}(\bA_1-\bA_2)\phi_i\mathbf{1}_i}\label{eq:bound_by_triangular}\\
        &\leq \Vert \btheta_1-\btheta_2\Vert + \sqrt{\Vert \bA_1-\bA_2\Vert}\sup_{\bphi:\Vert\bphi\Vert\leq 1}\sum_{i=1}^d\Vert \phi_i\mathbf{1}_i \Vert\notag\\
        & \leq \Vert \btheta_1-\btheta_2\Vert + \sqrt{\Vert \bA_1-\bA_2\Vert_F}, \label{eq:dist(V1,V2)}
    \end{align}
    where \eqref{eq:bound_by_triangular} follows from triangular inequlaity and the fact that $|\sqrt{x} - \sqrt{y}| \leq \sqrt{|x-y|},~ \forall x, y
    \geq 0$. For matrices, $\Vert\cdot\Vert$ and $\Vert\cdot\Vert_F$ denote the matrix operator norm and Frobenius norm respectively.
    
    Let $\cC_{\btheta}$ be an $\epsilon/2$-cover of $\{\btheta\in\RR^d|\Vert\btheta\Vert_2\leq L\}$ with respect to the 2-norm, and $\cC_{A}$ be an $\epsilon^2/4$-cover of $\{A\in\RR^{d\times d}|\Vert A\Vert_F\leq d^{1/2}B^2\lambda^{-1}\}$ with respect to the Frobenius norm. By \Cref{lemma:Covering Number of Euclidean Ball}, we know:
    \begin{align*}
        \big|\cC_{\btheta}\big|\leq \big(1+4L/\epsilon\big)^d, \quad \big|\cC_A\big|\leq \big[1+8d^{1/2}B^2/(\lambda\epsilon^2)\big]^{d^2}.
    \end{align*}
    By \eqref{eq:dist(V1,V2)}, for any $V_1\in \cV$, there exists $\btheta_2\in \cC_{\btheta}$ and $A_2\in\cC_A$ such that $V_2$ parametrized by $(\btheta_2, A_2)$ satisfies $\dist(V_1, V_2)\leq \epsilon$. Hence, it holds that $\cN_{\epsilon}\leq |\cC_{\btheta}|\cdot|\cC_{A}|$, which leads to
    \begin{align*}
        \log\cN_{\epsilon}\leq \log|\cC_{\wb}|+\log|\cC_A| \leq d\log(1+4L/\epsilon) + d^2\log\big[1+8d^{1/2}B^2/(\lambda\epsilon^2)\big].
    \end{align*}
    This concludes the proof.
\end{proof}

\subsection{Proof of \Cref{th:modified VA-DRPVI}}
In this section, we give the proof of \Cref{th:modified VA-DRPVI}, which largely follows the proof of \Cref{th:VA-DRPVI}, only with minor modifications of the argument of the variance estimation.

The following lemma bounds the estimation error by reference-advantage decomposition.
\begin{lemma}[Modified Variance-Aware Reference-Advantage Decomposition]
\label{lemma:Modified Variance-Aware Reference-Advantage Decomposition}
    There exist $\{\alpha_i\}_{i\in[d]}$, where $\alpha_i\in[0,H], \forall i\in [d]$, such that 
    \begin{align*}
        &\Big|\inf_{P_h(\cdot|s,a)\in\cU_h^{\rho}(s,a;\bmu_{h,i}^0)}[\PP_h\widehat{V}_{h+1}^{\rho}](s,a)-\widehat{\inf_{P_h(\cdot|s,a)\in\cU_h^{\rho}(s,a;\bmu_{h,i}^0)}}[\PP_h\widehat{V}_{h+1}^{\rho}](s,a)\Big|\\
        &\leq \underbrace{\lambda\sum_{i=1}^d\|\phi_i(s,a)\mathbf{1}_i\|_{\bSigma_h^{-1}}\|\EE^{\bmu_{h}^0}[V^{\star,\rho}_{h+1}(s)]_{\alpha_i}\|_{\bSigma_h^{-1}}}_\text{i}+\underbrace{\sum_{i=1}^d\|\phi_i(s,a)\mathbf{1}_i\|_{\bSigma_h^{-1}}\Big\|\sum_{\tau=1}^K\frac{\bphi_h^{\tau}\eta_h^{\tau}([V_{h+1}^{\star, \rho}]_{\alpha_i})}{\widehat{\sigma}^2_h(s_h^{\tau}, a_h^{\tau})}\Big\|_{\bSigma_h^{-1}}}_\text{ii}  \\
        &\quad +\underbrace{\lambda\sum_{i=1}^d\|\phi_i(s,a)\mathbf{1}_i\|_{\bSigma_h^{-1}}\Big\|\EE^{\bmu_{h}^0}\big[[\widehat{V}^{\rho}_{h+1}(s)]_{\alpha_i} - [V^{\star,\rho}_{h+1}(s)]_{\alpha_i} \big]\Big\|_{\bSigma_h^{-1}}}_\text{iii} \\
&\quad+\underbrace{\sum_{i=1}^d\|\phi_i(s,a)\mathbf{1}_i\|_{\bSigma_h^{-1}}\Big\|\sum_{\tau=1}^K\frac{\bphi_h^{\tau}\eta_h^{\tau}([\widehat{V}^{\rho}_{h+1}(s)]_{\alpha_i} - [V^{\star,\rho}_{h+1}(s)]_{\alpha_i})}{\widehat{\sigma}^2_h(s_h^{\tau}, a_h^{\tau})}
\Big\|_{\bSigma_h^{-1}}}_\text{iv},
    \end{align*}
    where $\eta_h^{\tau}([f]_{\alpha_i}) = \big(\big[\PP_h^0[f]_{\alpha_i}\big](s_h^{\tau},a_h^{\tau})-[f(s_{h+1}^{\tau})]_{\alpha_i} \big)$, for any function $f:\cS\rightarrow [0,H-1]$.
\end{lemma}

\begin{proof}[Proof of \Cref{th:modified VA-DRPVI}]
To prove this theorem, we bound the estimation error by $\Gamma_h(s,a)$, then invoke \Cref{lemma:Regret Decomposition for DRMDP}
to get the results. First, we bound terms i-iv in \Cref{lemma:Modified Variance-Aware Reference-Advantage Decomposition} at each step $h\in[H]$ respectively to deduce $\Gamma_h(s,a)$.

\paragraph{Bound i and iii:} We set $\lambda = 1/H^2$ to ensure that for all $(s,a,h)\in\cS\times\cA\times[H]$, we have
\begin{align}
\label{eq:Modified Variance-Aware Reference-Advantage Decomposition-bound i&iii}
    \text{i} + \text{iii} \leq \sqrt{\lambda}\sqrt{d}H\sum_{i=1}^d\|\phi_i(s,a)\mathbf{1}_i\|_{\bSigma_h^{-1}}= \sqrt{d}\sum_{i=1}^d\|\phi_i(s,a)\mathbf{1}_i\|_{\bSigma_h^{-1}}.
\end{align}

\paragraph{Bound ii:}
For all $(s,a,\alpha)\in\cS\times\cA\times[0,H]$, by definition we have $\widehat{\sigma}_h(s,a)\geq 1$. Thus, for all $(h,\tau,i)\in[H]\times[K]\times[d]$, we have 
$\eta_h^{\tau}([V_{h+1}^{\star, \rho}]_{\alpha_i})/\widehat{\sigma}_h(s_h^{\tau},a_h^{\tau})\leq H$. 
Note that $V_{H+1}^{\star,\rho}$ is independent of $\cD$, we can directly apply Bernstein-type self-normalized concentration inequality \Cref{lemma:Bernstein Concentration of Self-Normalized Processes} and a union bound to obtain the upper bound. In concrete, we define the filtration $\cF_{\tau-1, h} = \sigma(\{(s_h^{j},a_h^{j})\}_{j=1}^{
\tau}\cup\{s_{h+1}^{j}\}_{j=1}^{\tau-1})$. 
Since $V_{h+1}^{\star, \rho}$ and $\widehat{\sigma}_h(s,a)$ are independent of $\cD$, thus $\eta_h^{\tau}([V_{h+1}^{\star, \rho}]_{\alpha_i})/\widehat{\sigma}_h(s_h^{\tau},a_h^{\tau})$ is mean-zero conditioned on the filtration $\cF_{\tau-1, h}$. By \Cref{lemma:variance estimation} with $\alpha=H$, we have 
\begin{align}
\label{eq:modified variance estimation}
        \big[\VV_hV_{h+1}^{\star, \rho}\big](s,a)-\tilde{O}\Big(\frac{dH^3}{\sqrt{K\kappa}}\Big)\leq \widehat{\sigma}^2_h(s,a)\leq \big[\VV_hV_{h+1}^{\star, \rho}\big](s,a),
\end{align}
thus, for any $\alpha_i\in[0,H]$, we have
\begin{align}
\label{eq:the left inequality of modified variance estimation}
    \big[\VV_h[V_{h+1}^{\star, \rho}]_{\alpha_i}\big](s,a)-\tilde{O}\Big(\frac{dH^3}{\sqrt{K\kappa}}\Big)\leq \big[\VV_hV_{h+1}^{\star, \rho}\big](s,a)-\tilde{O}\Big(\frac{dH^3}{\sqrt{K\kappa}}\Big)\leq \widehat{\sigma}^2_h(s,a).
\end{align}
Further, we have
\begin{align}
    \EE\Big[\Big(\frac{\eta_h^{\tau}([V_{h+1}^{\star, \rho}]_{\alpha_i})}{\widehat{\sigma}_h(s_h^{\tau},a_h^{\tau})}\Big)^2 \Big|\cF_{\tau-1, h}\Big]&=\frac{[\Var[V_{h+1}^{\star,\rho}]_{\alpha_i}](s_h^{\tau}, a_h^{\tau})}{\widehat{\sigma}^2_h(s_h^{\tau}, a_h^{\tau})}\label{eq:Modified Variance-Aware Reference-Advantage Decomposition-sigma_hat is independent of D_without alpha}\\
    & \leq \frac{[\VV[V_{h+1}^{\star,\rho}]_{\alpha_i}](s_h^{\tau}, a_h^{\tau})}{\widehat{\sigma}_h^2(s_h^{\tau}, a_h^{\tau})}\notag\\
    & =\frac{[\VV[V_{h+1}^{\star,\rho}]_{\alpha_i}](s_h^{\tau}, a_h^{\tau}) - \tilde{O}(dH^3/\sqrt{K\kappa}) }{\widehat{\sigma}_h^2(s_h^{\tau}, a_h^{\tau})} + \frac{\tilde{O}(dH^3/\sqrt{K\kappa})}{\widehat{\sigma}_h^2(s_h^{\tau}, a_h^{\tau})}\notag\\
    &\leq 1+\frac{\tilde{O}(dH^3/\sqrt{K\kappa})}{\widehat{\sigma}_h^2(s_h^{\tau}, a_h^{\tau}) - \tilde{O}(dH^3/\sqrt{K\kappa})}\label{eq:Modified Variance-Aware Reference-Advantage Decomposition-invoke variance estimation lemma_without alpha}\\
    &\leq 1+2\tilde{O}\Big(\frac{dH^3}{\sqrt{K\kappa}} \Big)\label{eq:Modified Variance-Aware Reference-Advantage Decomposition-set K large enough_without alpha},
\end{align}
where \eqref{eq:Modified Variance-Aware Reference-Advantage Decomposition-sigma_hat is independent of D_without alpha} holds by the fact that $\widehat{\sigma}_h^2(\cdot,\cdot)$ is independent of $\cD$ and $(s_h^\tau,a_h^\tau)$ is $\cF_{\tau-1, h}$ measurable. \eqref{eq:Modified Variance-Aware Reference-Advantage Decomposition-invoke variance estimation lemma_without alpha} holds by \eqref{eq:the left inequality of modified variance estimation}, and \eqref{eq:Modified Variance-Aware Reference-Advantage Decomposition-set K large enough_without alpha} holds by setting $K\geq \tilde{\Omega}(d^2H^6/\kappa)$ such that $\widehat{\sigma}_h^2(s_h^{\tau}, a_h^{\tau}) - \tilde{O}(dH^3/\sqrt{K\kappa})\geq 1-\tilde{O}(dH^3/\sqrt{K\kappa})\geq 1/2$.
Further, by \eqref{eq:Modified Variance-Aware Reference-Advantage Decomposition-set K large enough_without alpha}, our choice of $K$ also ensures that $\EE\big[\big(\eta_h^{\tau}([V_{h+1}^{\star, \rho}]_{\alpha_i})\big)^2 |\cF_{\tau-1, h}\big]=O(1)$. Then by \Cref{lemma:Bernstein Concentration of Self-Normalized Processes}, we have
\begin{align*}
    \Big\|\sum_{\tau=1}^K\frac{\bphi_h^{\tau}\eta_h^{\tau}([V_{h+1}^{\star, \rho}]_{\alpha_i})}{\widehat{\sigma}^2_h(s_h^{\tau}, a_h^{\tau})}\Big\|_{\bSigma_h^{-1}}\leq \tilde{O}(\sqrt{d}).
\end{align*}
This implies 
\begin{align}
\label{eq:Modified Variance-Aware Reference-Advantage Decomposition-bound ii}
    \text{ii}\leq \tilde{O}(\sqrt{d})\sum_{i=1}^d\|\phi_i(s,a)\mathbf{1}_i\|_{\bSigma_h^{-1}}.
\end{align}

\paragraph{Bound iv:} Following the same induction analysis procedure, we know that $\|[\widehat{V}_{h+1}^{\rho}]_{\alpha_i}-[V_{h+1}^{\star,\rho}]_{\alpha_i}\|\leq \tilde{O}(\sqrt{d}H^2/\sqrt{K\kappa})$. Using standard $\epsilon$-covering number argument and \Cref{lemma:Hoeffding Concentration of Self-Normalized Processes}, we have
\begin{align}
\label{eq:Modified Variance-Aware Reference-Advantage Decomposition-bound iv}
    \text{iv} \leq \tilde{O}\Big(\frac{d^{3/2}H^2}{\sqrt{K\kappa}}\Big)\sum_{i=1}^d\|\phi_i(s,a)\mathbf{1}_i\|_{\bSigma_h^{-1}}.
\end{align}
To make it non-dominant, we require $K\geq \tilde{\Omega}(d^2H^4/\kappa)$.
By \eqref{eq:modified variance estimation}, we have $\widehat{\sigma}_h^2(s_h^{\tau}, a_h^{\tau})\leq[\VV_hV_{h+1}^{\star}](s_h^{\tau},a_h^{\tau})$, which implies that 
\begin{align*}
    \Big(\sum_{\tau=1}^K\frac{\bphi_h^{\tau}\bphi_h^{\tau\top}}{\widehat{\sigma}^2_h(s_h^{\tau},a_h^{\tau})}+\lambda I  \Big)^{-1}\preceq \Big(\sum_{\tau=1}^K\frac{\bphi_h^{\tau}\bphi_h^{\tau\top}}{[\VV_hV_{h+1}^{\star}](s_h^{\tau},a_h^{\tau})}+\lambda I  \Big)^{-1}.
\end{align*}
Combining \eqref{eq:Modified Variance-Aware Reference-Advantage Decomposition-bound i&iii}, \eqref{eq:Modified Variance-Aware Reference-Advantage Decomposition-bound ii} and \eqref{eq:Modified Variance-Aware Reference-Advantage Decomposition-bound iv}, we have 
\begin{align*}
    &\Big|\inf_{P_h(\cdot|s,a)\in\cU_h^{\rho}(s,a;\bmu_{h,i}^0)}[\PP_h\widehat{V}_{h+1}^{\rho}](s,a)-\widehat{\inf_{P_h(\cdot|s,a)\in\cU_h^{\rho}(s,a;\bmu_{h,i}^0)}}[\PP_h\widehat{V}_{h+1}^{\rho}](s,a)\Big|\\
    &\leq \tilde{O}(\sqrt{d})\sum_{i=1}^d\|\phi_i(s,a)\mathbf{1}_i\|_{\bSigma_h^{\star-1}}.
\end{align*}
Define $\Gamma_h(s,a)=\tilde{O}(\sqrt{d})\sum_{i=1}^d\|\phi_i(s,a)\mathbf{1}_i\|_{\bSigma_h^{\star-1}}$, we concludes the proof by invoking \Cref{lemma:Regret Decomposition for DRMDP}.
\end{proof}

\section{Auxiliary Lemmas}
\begin{lemma}[Concentration of Self-Normalized Processes]
\cite[Theorem 1]{abbasi2011improved}\label{lemma:Hoeffding Concentration of Self-Normalized Processes}
    Let $\{\epsilon_t\}_{t=1}^{\infty}$ be a real-valued stochastic process with corresponding filtration $\{\mathcal{F}_t\}_{t=0}^{\infty}$. Let $\epsilon_t|\mathcal{F}_{t-1}$ be mean-zero and $\sigma$-subGaussian; i.e. $\mathbb{E}[\epsilon_t|\mathcal{F}_{t-1}]=0$, and 
    \begin{equation*}
        \forall \lambda \in \mathbb{R}, ~~~~\mathbb{E}[e^{\lambda \epsilon_t}|\mathcal{F}_{t-1}] \leq e^{\lambda^2\sigma^2/2}.
    \end{equation*}
    Let $\{\bm{\phi}_t\}_{t=1}^{\infty}$ be an $\mathbb{R}^d$-valued stochastic process where $\bphi_t$ is $\mathcal{F}_{t-1}$ measurable. Assume $\bLambda_0$ is a $d\times d$ positive definite matrix, and let $\bLambda_t=\bLambda_0+\sum_{s=1}^t\bm{\phi}_s\bm{\phi}_s^\top$. Then for any $\delta > 0$, with probability at least $1-\delta$, we have for all $t \geq 0$:
    \begin{equation*}
        \bigg\Vert \sum_{s=1}^t \bm{\phi}_s\epsilon_s \bigg\Vert^2_{\bLambda_t^{-1}} 
        \leq 2\sigma^2 \log \bigg[ \frac{\det(\bLambda_t)^{1/2}\det(\bLambda_0)^{-1/2}}{\delta}\bigg].
    \end{equation*}
\end{lemma}
\begin{lemma}[Bernstein inequality for self-normalized martingales]\cite[Theorem 2]{zhou2021nearly} \label{lemma:Bernstein Concentration of Self-Normalized Processes}
Let $\{\eta_t\}_{t=1}^{\infty}$ be a real-valued stochastic process. Let $\{\mathcal{F}_t\}_{t=0}^{\infty}$ be a filtration, such that $\eta_t$ is $\mathcal{F}_t$-measurable. Assume $\eta_t$ also satisfies
\begin{align*}
    |\eta_t|\leq R, \mathbb{E}[\eta_t|\mathcal{F}_{t-1}]=0, \mathbb{E}[\eta^2_t|\mathcal{F}_{t-1}]\leq \sigma^2.
\end{align*}
Let $\{\bx_t\}_{t=1}^{\infty}$ be an $\mathbb{R}^d$-valued stochastic process where $\bx_t$ is $\mathcal{F}_{t-1}$ measurable and $\|\bx_t\|\leq L$. Let $\bLambda_t=\lambda \mathbf{I}_d+\sum_{s=1}^t\bx_s\bx_s^{\top}$. Then for any $\delta>0$, with probability at least $1-\delta$, for all $t>0$,
\begin{align*}
    \bigg\|\sum_{s=1}^t \bx_s\eta_s\bigg\|_{\bLambda_t^{-1}}\leq 8\sigma\sqrt{d\log\Big(1+\frac{tL^2}{\lambda d}\Big)\cdot\log\Big(\frac{4t^2}{\delta} \Big)}+4R\log\Big(\frac{4t^2}{\delta} \Big).
\end{align*}
\end{lemma}

\begin{lemma}\cite[Lemma H.5]{min2021variance} 
\label{lemma:matrix-normalized concentration}
Let $\bphi: \mathcal{S} \times \mathcal{A} \rightarrow \mathbb{R}^d$ satisfying $\|\bphi(x, a)\| \leq C$ for all $(x, a) \in \mathcal{S} \times \mathcal{A}$. For any $K>0$ and $\lambda>0$, define $\overline{\mathbb{G}}_K=\sum_{k=1}^K \bphi(x_k, a_k) \bphi(x_k, a_k)^{\top}+\lambda \mathbf{I}_d$ where $(x_k, a_k)$ 's are i.i.d. samples from some distribution $\nu$ over $\mathcal{S} \times \mathcal{A}$. Let $\mathbb{G}=\mathbb{E}_v[\bphi(x, a) \bphi(x, a)^{\top}]$. Then, for any $\delta \in(0,1)$, if $K$ satisfies that
\begin{align*}
    K \geq \max \Big\{512 C^4\big\|\mathbb{G}^{-1}\big\|^2 \log \Big(\frac{2 d}{\delta}\Big), 4 \lambda\big\|\mathbb{G}^{-1}\big\|\Big\},
\end{align*}
then with probability at least $1-\delta$, it holds simultaneously for all $u \in \mathbb{R}^d$ that
\begin{align*}
    \|u\|_{\overline{\mathbb{G}}_K^{-1}} \leq \frac{2}{\sqrt{K}}\|u\|_{\mathbb{G}^{-1}}.
\end{align*}
\end{lemma}

\end{document}